\documentclass{article}
\usepackage{cdcma_arxiv,times}

\usepackage[utf8]{inputenc}
\usepackage[T1]{fontenc}
\usepackage{microtype}
\usepackage{graphicx}
\usepackage{subcaption}
\usepackage{booktabs}
\usepackage{hyperref}
\usepackage{url}
\usepackage{amsfonts}
\usepackage{nicefrac}
\usepackage{xcolor}
\usepackage{amsmath}
\usepackage{amssymb}
\usepackage{mathtools}
\usepackage{amsthm}
\usepackage{algorithm}
\usepackage{algorithmic}
\usepackage[capitalize,noabbrev]{cleveref}
\usepackage{multirow}
\usepackage{capt-of}
\usepackage{wrapfig}

\theoremstyle{plain}
\newtheorem{theorem}{Theorem}[section]
\newtheorem{proposition}[theorem]{Proposition}

\newtheorem{corollary}[theorem]{Corollary}
\theoremstyle{definition}
\newtheorem{definition}[theorem]{Definition}

\theoremstyle{remark}

\title{Communication Gain and Delay Cost \\ under Cross-Timestep Delays in Cooperative \\ Multi-Agent Reinforcement Learning}

\author{Zihong Gao \quad Hongjian Liang \quad Lei Hao \quad Liangjun Ke\thanks{Corresponding author: Liangjun Ke (\href{mailto:keljxjtu@xjtu.edu.cn}{\texttt{keljxjtu@xjtu.edu.cn}}).} \\
{\small The State Key Laboratory for Manufacturing Systems Engineering} \\
{\small School of Automation Science and Engineering, Xi'an Jiaotong University} \\
{\small \href{mailto:g1246830895@stu.xjtu.edu.cn}{\texttt{g1246830895@stu.xjtu.edu.cn}} \quad
\href{mailto:lianghj@stu.xjtu.edu.cn}{\texttt{lianghj@stu.xjtu.edu.cn}}} \\
{\small \href{mailto:lei_hao@stu.xjtu.edu.cn}{\texttt{lei\_hao@stu.xjtu.edu.cn}} \quad
\href{mailto:keljxjtu@xjtu.edu.cn}{\texttt{keljxjtu@xjtu.edu.cn}}}}

\hypersetup{
  pdftitle={Communication Gain and Delay Cost under Cross-Timestep Delays in Cooperative Multi-Agent Reinforcement Learning},
  pdfauthor={Zihong Gao, Hongjian Liang, Lei Hao, Liangjun Ke},
  pdfsubject={Cooperative multi-agent reinforcement learning with cross-timestep communication delays},
  pdfkeywords={Multi-agent reinforcement learning, delayed communication, CGDC, CDCMA},
  hidelinks
}

\begin{document}

\maketitle

\begin{abstract}
	Communication is crucial for cooperative multi-agent reinforcement learning (MARL) under partial observability, but bounded cross-timestep communication delays make received messages temporally misaligned and potentially stale.
	We formalize this setting as a delayed-communication partially observable Markov game (DeComm-POMG) and introduce Communication Gain and Delay Cost (CGDC), a gain--cost metric that evaluates delayed-message utility through no-message and timely-reference counterfactuals.
	We further derive a delay-induced return-loss bound that relates the return gap between policies using timely references and those using delayed messages to the accumulated information gap.
	Guided by CGDC, we propose CDCMA, an actor--critic framework for selecting communication partners, constructing future-aware outgoing messages, and aggregating received delayed messages.
	Experiments on MPE tasks without teammate vision and on SMAC maps show that CDCMA achieves higher mean performance than the communication baselines across the evaluated bounded stochastic delay regimes; on two representative tasks, it also yields smaller decision-distribution gaps.
	CDCMA remains competitive in the evaluated zero-delay settings and exhibits favorable zero-shot cross-difficulty generalization on Cooperative Navigation.
\end{abstract}

\section{Introduction}

\begin{wrapfigure}{r}{0.45\textwidth}
	\vspace{-0.8em}
	\centering
	\includegraphics[width=\linewidth]{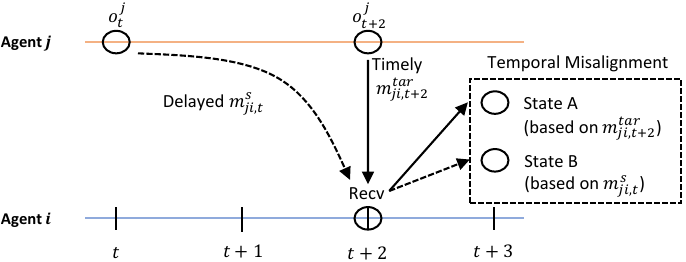}
	\caption{Temporal misalignment under cross-timestep communication delays.}
	\label{fig:info-gap}
	\vspace{-1.2em}
\end{wrapfigure}

Communication is central to coordination in cooperative multi-agent reinforcement learning (MARL) under partial observability.
Most communication methods assume that messages are available within the current decision step.
In practical systems, sensing, computation, and transmission latencies can induce cross-timestep communication delays, where a message generated at timestep $t$ reaches the receiver only at $t{+}d$.
As illustrated in Fig.~\ref{fig:info-gap}, conditioning on a delayed message rather than its timely reference can shift the receiver's decision context and change the resulting transition.
The key question is therefore whether a received delayed message still provides positive utility at consumption time.

Prior MARL communication methods mainly study delay-free settings, including centralized protocols and selective schemes.
Delayed-communication methods such as DACOM~\citep{yuan2023dacom} and CoDe~\citep{song2025code} explicitly consider latency, but they do not directly score delayed-message utility through both no-message and timely-reference counterfactuals.
This leaves open how a receiver should decide whether a delayed message remains useful under its current local context.

We address this gap with a gain--cost view of delayed-message utility.
We first formalize cooperative MARL with cross-timestep communication delays as a delayed-communication partially observable Markov game (DeComm-POMG), where delayed delivery is part of the decision process.
We then derive a delay-induced return-loss bound that relates the return gap between policies using timely references and those using delayed messages to an accumulated information gap.
This motivates two complementary counterfactual terms: communication gain compares using a delayed message against receiving no message, whereas delay cost compares it against the timely reference through an information-gap surrogate for delay-induced return loss.
Together, they define Communication Gain and Delay Cost (CGDC), a gain--cost metric for evaluating delayed-message utility under centralized training with decentralized execution (CTDE).

Building on CGDC, we propose CDCMA, an actor--critic framework for communication under cross-timestep delays.
It realizes the gain--cost view through three components.
The Delay Communication Objective Selector (DCOS) learns critic-based CGDC targets and predicts per-pair scores for partner selection.
The Observation Trajectory Generator (OTG) constructs future-aware outgoing messages to reduce consumption-time misalignment.
The CGDC-guided Attentional Message Aggregator (CAMA) aggregates received delayed messages using CGDC-guided attention.
We evaluate CDCMA on partially observable variants of Cooperative Navigation and Predator Prey without teammate vision, and on benchmark maps from the StarCraft Multi-Agent Challenge (SMAC)~\citep{samvelyan2019starcraft}, under bounded stochastic delay regimes.
Results show that CDCMA achieves higher mean performance than the communication baselines across the bounded delayed settings studied here and retains favorable zero-shot performance across unseen delay difficulties on Cooperative Navigation.

Our contributions are summarized as follows:
\begin{itemize}
	\item We formalize cooperative MARL with cross-timestep communication delays as \textbf{DeComm-POMG}, which models delayed message arrivals and admits the augmented-state Markov representation used in our analysis.
	\item We introduce \textbf{CGDC}, a gain--cost metric that evaluates delayed-message utility through no-message and timely-reference counterfactuals, and derive a delay-induced return-loss bound based on the information gap between action distributions under timely-reference and delayed-message inputs.
	\item We propose \textbf{CDCMA}, which realizes CGDC through partner selection, future-aware message construction, and CGDC-guided delayed-message aggregation, and validate it across bounded stochastic delay regimes.
\end{itemize}
\vspace{-10pt}

\section{Related Work}
Most MARL communication methods are developed for delay-free settings and can be broadly grouped into centralized and selective communication.
Centralized protocols broadcast or aggregate messages system-wide, often through differentiable message passing, as in CommNet~\citep{sukhbaatar2016learning}, DIAL~\citep{foerster2016learning}, and BiCNet~\citep{peng2017multiagent}; later work improves expressiveness or robustness through recurrent aggregation, targeted communication, message filtering, and auxiliary objectives~\citep{wang2022fcmnet,das2019tarmac,kim2019message,guan2022efficient}.
Selective methods improve scalability by restricting communication to selected agents, entities, or timesteps, as in ATOC~\citep{jiang2018learning}, I2C~\citep{ding2020learning}, SMS~\citep{xue2022efficient}, G2ANet~\citep{liu2020multi}, T2MAC~\citep{sun2024t2mac}, TGCNet~\citep{Zhang_He_Cheng_Li_2025}, and GRDC~\citep{gao2026grdc}.
Although these methods improve communication efficiency, they generally assume that exchanged messages are available within the current decision step.

A smaller line of work studies communication with delayed message arrivals.
DACOM~\citep{yuan2023dacom} is delay-aware and uses TimeNet to adjust how long an agent waits for messages within a decision step, while CoDe~\citep{song2025code} infers long-term intents via future-action prediction and fuses asynchronous messages through dual alignment.
Complementary work studies observation delays: RDC~\citep{fu2026rainbow} compensates delayed observation components under stochastic individual delays.
Our work instead focuses on delayed-message utility under unobserved cross-timestep communication delays, evaluating received delayed messages through no-message and timely-reference counterfactuals. Further discussion appears in Appendix~\ref{app:extended-related-work}.

\vspace{-5pt}
\section{Problem Formulation and Gain--Cost View}
\label{sec:pro_formu}

\subsection{DeComm-POMG}

We build on the standard partially observable Markov game (POMG) formulation~\citep{littman1994markov} and extend it to cross-timestep communication delays.
Unlike standard communication-augmented POMGs, a message generated at timestep $t$ may become available only after several timesteps, so delayed delivery must be treated as part of the decision process.

\begin{definition}[DeComm-POMG]\label{def:decomm-pomg}
	A \textbf{Delayed-Communication Partially Observable Markov Game} is the tuple $\boldsymbol{(\mathcal{N},\mathcal{S},\mathcal{A},\mathcal{O},\mathcal{R},\mathcal{P},\mathcal{M},\mathcal{D},\rho,\gamma,\varphi)}$, where the first seven components coincide with a standard communication-augmented POMG, $\mathcal{D}=\{\mathcal{D}_{ji}\subset\mathbb{Z}_{\ge 0}\mid j\neq i\}$ specifies per-pair delay domains, $\rho$ is the initial-state distribution, $\gamma\in(0,1)$ is the discount factor, and $\varphi=\{\varphi_{ji}\mid j\neq i\}$ governs the delay process.
	At each timestep $t$, a delay $d_{ji,t}\in\mathcal{D}_{ji}$ is sampled for each link $(j\!\to\! i)$ according to $\varphi_{ji}$.
	We assume that the delay process is exogenous and Markov, with any latent delay-process state included in $s_t$, and that its realizations are not observable to agents.
\end{definition}

At each timestep, agents choose an environment action $\boldsymbol{a}_t$ and outgoing messages $\boldsymbol{m}^{s}_t$.
For the Markov representation below, we regard their combination $\boldsymbol{u}_t:=(\boldsymbol{a}_t,\boldsymbol{m}^{s}_t)$ as the joint environment--communication action.
Let $m^{s}_{ji,t}$ denote the message sent from agent $j$ to agent $i$ at timestep $t$, and let $m^{r}_{ij,t}$ denote the message from $j$ that is available to agent $i$ at receiver timestep $t$.
If the sampled delay is $d_{ji,t}$, then $m^{s}_{ji,t}$ arrives at timestep $t{+}d_{ji,t}$.
If multiple messages from the same sender arrive simultaneously, only the most recent one is retained: $\kappa_{ji}(t)\in\arg\max_{k\le t}\{\,k\mid t=k+d_{ji,k}\,\}$ and $m^{r}_{ij,t}=m^{s}_{ji,\kappa_{ji}(t)}$.

\subsection{Delayed-Message Utility: Communication Gain and Delay Cost}
\label{subsec:decomm-info-gap}

A received delayed message should be evaluated under two counterfactual comparisons.
Compared with receiving no message, it may still improve decision making.
Compared with the timely reference message, however, it may incur delay-related return loss because of temporal misalignment.
This motivates a gain--cost view of delayed-message utility.

\paragraph{Markov representation.}
To analyze delayed communication with value-function tools, we augment the environment state with the message buffer $\mathcal{B}_t$ and write $\widetilde{s}_t=(s_t,\mathcal{B}_t)$, where $\mathcal{B}_t\in\mathcal{B}$ stores the in-flight and delivered messages required by the delayed-delivery process.

\begin{proposition}[Markov representation of DeComm-POMG]\label{prop:markov-decomm}
	With augmented state $\widetilde{s}_t=(s_t,\mathcal{B}_t)$ and joint environment--communication action $\boldsymbol{u}_t=(\boldsymbol{a}_t,\boldsymbol{m}^{s}_t)$, DeComm-POMG admits a controlled Markov representation on $\widetilde{\mathcal{S}}=\mathcal{S}\times\mathcal{B}$; specifically,
	$\Pr(\widetilde{s}_{t+1}\mid \widetilde{s}_{0:t},\boldsymbol{u}_{0:t})
	=
	\Pr(\widetilde{s}_{t+1}\mid \widetilde{s}_t,\boldsymbol{u}_t)$.
	The corresponding augmented POMG preserves the local observation and message histories of the original DeComm-POMG.
\end{proposition}

Proposition~\ref{prop:markov-decomm} shows that delayed communication can be analyzed through an augmented Markov state.
The proof is given in Appendix~\ref{app:proof:markov-decomm}.
For the delay-loss analysis below, we hold the message-generation protocol fixed and compare the induced environment-action policies.

\paragraph{Communication gain.}
For agent $i$, the discounted return in DeComm-POMG is $J_i(\boldsymbol{\pi}) := \mathbb{E}_{\boldsymbol{\pi}}[\sum_{t\ge 0}\gamma^t r_i(\widetilde{s}_t,\boldsymbol{a}_t)]$.
Let $\boldsymbol{\pi}^{(+j)}$ denote the joint policy in which agent $i$ can use messages from sender $j$ while other incoming messages are masked, and let $\boldsymbol{\pi}^{(0)}$ denote the joint policy in which all incoming messages to agent $i$ are masked.
Then $J_i(\boldsymbol{\pi}^{(+j)})-J_i(\boldsymbol{\pi}^{(0)})$ captures the return-level benefit attributable to sender $j$.

To express this comparison at a given state--action pair, define $Q_i^{\boldsymbol{\pi}}(\widetilde{s}_t,\boldsymbol{a}_t) := \mathbb{E}_{\boldsymbol{\pi}}[\sum_{k\ge 0}\gamma^k r_i(\widetilde{s}_{t+k},\boldsymbol{a}_{t+k}) \mid \widetilde{s}_t,\boldsymbol{a}_t]$.
We define the communication gain as
\begin{equation}
	V^g_{ij,t}(\widetilde{s}_t,\boldsymbol{a}_t)
	:=
	Q_i^{(+j)}(\widetilde{s}_t,\boldsymbol{a}_t)
	-
	Q_i^{(0)}(\widetilde{s}_t,\boldsymbol{a}_t).
	\label{eq:q-gain}
\end{equation}
This compares continuation values under the two message-access policies, conditional on the same current state and joint action.
With bounded rewards and $\gamma\in(0,1)$, $V^g_{ij,t}$ is uniformly bounded; the bound is stated in Appendix~\ref{app:gain-bound}.

\paragraph{Delay-induced return loss and information gap.}
Communication gain alone is insufficient, because a delayed message may be useful relative to no message yet inferior to its timely reference.
Let $m^{\mathrm{tar}}_{ji,t}$ denote the timely reference for sender $j$ at receiver timestep $t$, representing the sender-side information aligned with the message-consumption time $t$.
Its concrete training-time instantiation is given in Sec.~\ref{subsec:cgdc-surrogate}.
For this policy comparison, when no message from sender $j$ is available at timestep $t$, both the received and timely-reference conditions are set to $m^\emptyset$.

Let $\boldsymbol{\pi}^{\mathrm{delayed}}$ denote the induced environment-action policy in which agent $i$ conditions on the received delayed message $m^r_{ij,t}$, and let $\boldsymbol{\pi}^{\mathrm{timely}}$ be identical except that agent $i$ conditions on the timely reference $m^{\mathrm{tar}}_{ji,t}$.
Both policies share the same message-generation protocol and differ only in this action-conditioning input.
We define the potential delay-induced return loss as
\begin{equation}
	\mathcal{L}^{\mathrm{delay}}_i
	:=
	\left[J_i(\boldsymbol{\pi}^{\mathrm{timely}})-J_i(\boldsymbol{\pi}^{\mathrm{delayed}})\right]_+,
	\label{eq:delay-loss}
\end{equation}
where $[x]_+:=\max\{x,0\}$ accounts for the fact that the timely condition need not yield a larger return on every policy instance.

To relate this return-level loss to a local discrepancy, fix a receiver timestep $t$ and hold fixed the information available before consuming incoming messages.
Let $x_{i,t}$ denote this message-free local input of agent $i$.
We define the information gap as
\begin{equation}\label{eq:info-gap}
	\Delta_{\mathrm{info}}(i,j,t)
	=
	D_{\mathrm{KL}}\!\left(
	\pi_i(\cdot \mid x_{i,t}, m^{\mathrm{tar}}_{ji,t})
	\ \big\|\ 
	\pi_i(\cdot \mid x_{i,t}, m^{r}_{ij,t})
	\right).
\end{equation}
This quantity measures the action-distribution gap induced by replacing the timely reference condition with the received delayed message under the same local input.

\begin{theorem}[Delay-induced return-loss bound]\label{thm:value-kl}
	Assume $|r_i(\widetilde{s},\boldsymbol{a})|\le R_{\max}$, $\gamma\in(0,1)$, and that the decentralized joint environment-action policy factorizes across agents. Then
	\begin{equation}
		\mathcal{L}^{\mathrm{delay}}_i
		\;\le\;
		\frac{2R_{\max}}{1-\gamma}\,
		\mathbb{E}_{\tau\sim \boldsymbol{\pi}^{\mathrm{timely}}}\!\left[
		\sum_{t\ge 0}\gamma^t \sqrt{\frac{1}{2}\Delta_{\mathrm{info}}(i,j,t)}
		\right],
		\label{eq:value-kl-bound}
	\end{equation}
	where the expectation is taken over trajectories $\tau$ under $\boldsymbol{\pi}^{\mathrm{timely}}$ on augmented DeComm-POMG.
\end{theorem}

The bound relates potential delay-induced return loss to the accumulated information gap between action distributions under timely-reference and delayed-message inputs.
Together with Eq.~\eqref{eq:q-gain}, it motivates a gain--cost view of delayed-message utility: useful delayed communication should have high utility relative to no message and low loss relative to the timely reference.

\subsection{Tractable Surrogates and CGDC}
\label{subsec:cgdc-surrogate}

Under CTDE, we instantiate this gain--cost view with two critics: $Q_i^{\mathrm{full}}$ provides the actor--critic learning signal, while the auxiliary $Q_i^{\mathrm{gain}}$ constructs the critic-induced tempered policy and the gain surrogate used for CGDC targets; decentralized execution uses only local inputs.

\paragraph{Critic-induced tempered policy.}
The information gap in Eq.~\eqref{eq:info-gap} is a theoretical quantity defined by comparing actor policies under timely-reference and delayed-message conditions.
To construct a tractable training-time delay-cost surrogate, we introduce a critic-induced tempered policy as a discrepancy probe under a common value reference.
This probe is used only during CTDE training; execution does not access the timely reference.
Let $\boldsymbol{o}_{-j,t}$ and $\boldsymbol{a}_{-j,t}$ denote the observations and actions of all agents except agent $j$ at timestep $t$, and let $\boldsymbol{a}_{-ij,t}$ denote the actions of all agents except agents $i$ and $j$.
Given a message condition $m$ for agent $i$ at $t$, we define
\begin{equation}
	\widetilde{\pi}_i(m)
	=
	\mathrm{softmax}_{a_{i,t}\in\mathcal{A}_i}
	\Big[
	\eta\, Q_i^{\mathrm{gain}}\big(a_{i,t},\boldsymbol{o}_{-j,t},\boldsymbol{a}_{-ij,t},m\big)
	\Big],
	\label{eq:temp-policy}
\end{equation}
where $\eta>0$ controls the sharpness of the softmax.

\paragraph{Delay cost surrogate.}
For pairwise CGDC targets, centralized training constructs a delayed-message condition for every $j\ne i$, including pairs without actual arrivals (Appendix~\ref{app:pairwise-targets}).
Within these target evaluations, $m^r_{ij,t}$ denotes the constructed reception condition, with timely reference $m^{\mathrm{tar}}_{ji,t}:=e_{j,t}$, the sender's current trajectory embedding computed without future prediction.
While OTG compensates for temporal evolution, $e_{j,t}$ provides the consumption-time reference; neither the constructed conditions nor the reference are available to the execution policy.
We define the delay cost surrogate as
\begin{equation}\label{eq:vc}
	V^c_{ij,t}
	=
	D_{\mathrm{KL}}\!\Big(
	\widetilde{\pi}_{i}(m^{\mathrm{tar}}_{ji,t})
	\ \Big\|\ 
	\widetilde{\pi}_{i}(m^{r}_{ij,t})
	\Big)\ \ge 0.
\end{equation}
For the following corollary, evaluate this discrepancy under the message conditions of Theorem~\ref{thm:value-kl}, including its empty-message convention. Write
$\pi_i^{\mathrm{tar}}:=\pi_i(\cdot\mid x_{i,t},m^{\mathrm{tar}}_{ji,t})$,
$\pi_i^{\mathrm{del}}:=\pi_i(\cdot\mid x_{i,t},m^r_{ij,t})$,
$\widetilde\pi_i^{\mathrm{tar}}:=\widetilde\pi_i(m^{\mathrm{tar}}_{ji,t})$, and
$\widetilde\pi_i^{\mathrm{del}}:=\widetilde\pi_i(m^r_{ij,t})$.
Define the actor--probe discrepancies
\begin{equation}
	\delta^{\mathrm{tar}}_{ij,t}:=\|\pi_i^{\mathrm{tar}}-\widetilde\pi_i^{\mathrm{tar}}\|_{\mathrm{TV}},
	\qquad
	\delta^{\mathrm{del}}_{ij,t}:=\|\pi_i^{\mathrm{del}}-\widetilde\pi_i^{\mathrm{del}}\|_{\mathrm{TV}}.
	\label{eq:actor-probe-gap}
\end{equation}

\begin{corollary}[Return-loss bound with critic-tempered delay cost]\label{cor:surrogate-bound}
	Under the conditions of Theorem~\ref{thm:value-kl},
	\begin{equation}
		\mathcal{L}^{\mathrm{delay}}_i
		\le
		\frac{2R_{\max}}{1-\gamma}
		\mathbb{E}_{\tau\sim\boldsymbol{\pi}^{\mathrm{timely}}}\!\left[
		\sum_{t\ge0}\gamma^t
		\left(
		\delta^{\mathrm{tar}}_{ij,t}
		+\sqrt{\frac{1}{2}V^c_{ij,t}}
		+\delta^{\mathrm{del}}_{ij,t}
		\right)
		\right].
		\label{eq:surrogate-return-bound}
	\end{equation}
\end{corollary}
The corollary separates the implemented critic-tempered discrepancy from its actor--probe approximation terms.
The proof and further discussion are given in Appendix~\ref{app:surrogate-justification}.

\paragraph{Gain surrogate.}
For the pairwise gain target, we evaluate the auxiliary gain critic with and without the message from sender $j$.
Let $m^\emptyset$ denote a null message representing message absence.
We define the gain surrogate as
\begin{equation}\label{eq:vg-critic}
	\widehat{V}^{g}_{ij,t}
	=
	Q_i^{\mathrm{gain}}\big(a_{i,t},\boldsymbol{o}_{-j,t},\boldsymbol{a}_{-ij,t},m^{r}_{ij,t}\big)
	-
	Q_i^{\mathrm{gain}}\big(a_{i,t},\boldsymbol{o}_{-j,t},\boldsymbol{a}_{-ij,t},m^\emptyset\big),
\end{equation}
Motivated by Eq.~\eqref{eq:q-gain}, this surrogate measures the message-conditioned difference in learned value predictions.
The two critic branches share continuation actions during training; their targets are specified in Appendix~\ref{app:training-scheme}.

\paragraph{CGDC.}
We combine communication gain and delay cost into a single gain--cost score.

\begin{definition}[CGDC]\label{def:cgdc}
	Let $V^g_{ij,t}$ and $V^c_{ij,t}$ be defined in Eqs.~\eqref{eq:q-gain} and~\eqref{eq:vc}.
	The \textbf{Communication Gain and Delay Cost} is
	\begin{equation}\label{eq:cgdc}
		c_{ij,t}
		=
		V^g_{ij,t}
		-
		\lambda\,V^c_{ij,t},
	\end{equation}
	where $\lambda\ge 0$ balances the two terms and controls the conservativeness of the communication gate.
\end{definition}

For training, we instantiate Eq.~\eqref{eq:cgdc} with $\widehat V^g_{ij,t}$ in place of $V^g_{ij,t}$.
Both terms evaluate the same pairwise delayed-message condition, relative to message absence and the timely reference, respectively.

\section{CDCMA: A Realization of the Gain--Cost View}\label{sec:method}

\subsection{Overview}

\begin{figure}[t]
	\vspace{-8pt}
	\centering
	\includegraphics[width=\linewidth]{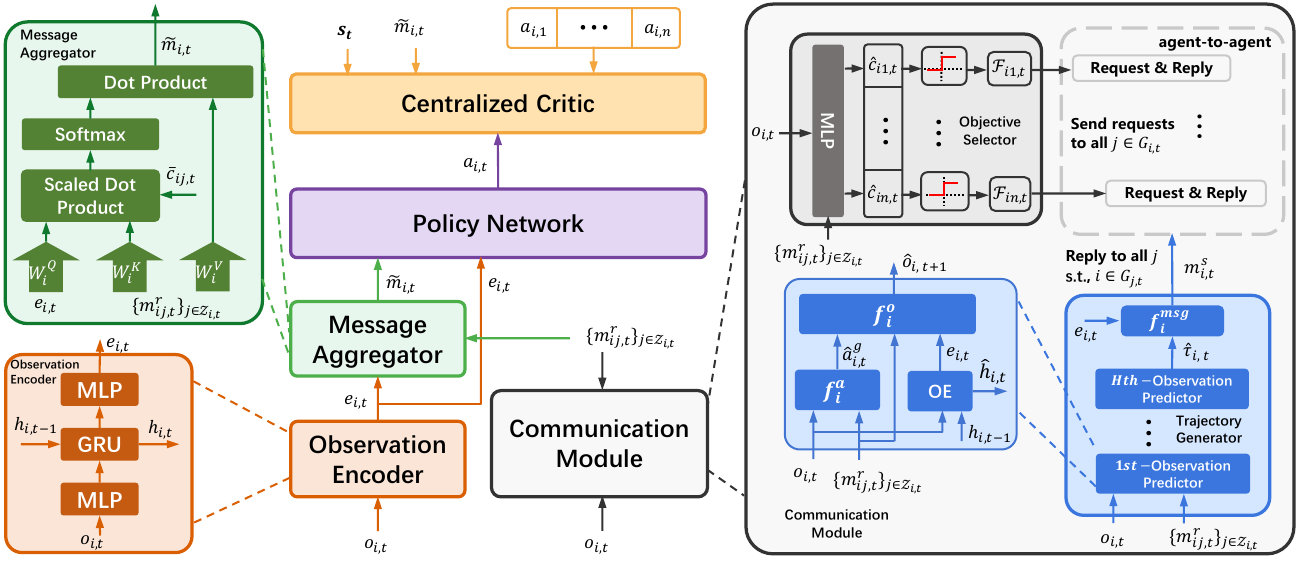}
	\caption{Overall architecture of CDCMA.}
	\label{fig:framework}
	\vspace{-5pt}
\end{figure}

The gain--cost view in Sec.~\ref{sec:pro_formu} suggests three requirements for communication under cross-timestep delays: messages should be requested when their predicted CGDC scores are positive, outgoing messages should remain informative at future consumption time, and received delayed messages should be aggregated using content relevance and the request-time predicted CGDC scores stored with them.
CDCMA realizes these requirements through a \emph{request--construct--aggregate} pipeline.

At timestep $t$, agent $i$ first encodes its local observation into a trajectory embedding and predicts which teammates to query.
Each sender constructs one outgoing message by combining historical trajectory information with predicted future observations, and the same message is used for all receivers that request it.
Finally, because the messages available at receiver timestep $t$ generally differ from those queried at $t$, agent $i$ aggregates received delayed messages and selects its action from the aggregated message and local embedding.

This pipeline is illustrated in Fig.~\ref{fig:framework}.
CDCMA consists of an observation encoder, the \emph{Delay Communication Objective Selector} (DCOS), the \emph{Observation Trajectory Generator} (OTG), the \emph{CGDC-guided Attentional Message Aggregator} (CAMA), a decentralized policy network, and two critics used during training.
DCOS, OTG, and CAMA implement partner selection, future-aware message construction, and CGDC-guided delayed-message aggregation, respectively.
Pseudocode and complexity are deferred to Algorithm~\ref{app:algo} and Appendix~\ref{app:time-complexity}.
Training details are provided in Appendix~\ref{app:training-scheme}.

\subsection{Delay Communication Objective Selector}
Under cross-timestep communication delays, a receiver must decide whom to query before future arrivals are known.
DCOS addresses this problem by predicting a per-pair CGDC score from local context available at execution and requesting a message only when the predicted score is positive.

Let $\mathcal{Z}_{i,t}\triangleq \{\, j\neq i \mid m^r_{ij,t}\neq m^\emptyset \,\}$ denote the available-sender set of agent $i$ at timestep $t$, where $m^\emptyset$ is a null message.
Given the current local observation and the received-message context $\{m^r_{ij,t}\}_{j\in \mathcal{Z}_{i,t}}$, agent $i$ predicts the scores $\{\hat c_{ij,t}\}_{j\ne i}$ by
\begin{equation}
	\{\hat c_{ij,t}\}_{j\ne i}
	=
	\phi_i\!\left(o_{i,t}, \{m^r_{ij,t}\}_{j\in \mathcal{Z}_{i,t}}; \theta_{\phi_i}\right),
	\qquad
	\hat c_{ii,t}=0.
	\label{eq:dcos}
\end{equation}
Here, $\hat c_{ij,t}$ is a learned estimator of the CGDC score defined in Sec.~\ref{subsec:cgdc-surrogate}.
Thresholding the prediction yields the request mask $F_{ij,t}=\mathbb{I}[\hat c_{ij,t}>0]$ and the queried partner set $\mathcal{G}_{i,t}=\{\,j\neq i \mid F_{ij,t}=1\,\}$.

The distinction between $\mathcal{G}_{i,t}$ and $\mathcal{Z}_{i,t}$ is essential.
$\mathcal{G}_{i,t}$ is the set of teammates queried at timestep $t$, whereas $\mathcal{Z}_{i,t}$ is the set of senders whose delayed messages are available at timestep $t$.
Because delivery is delayed and unobserved, these two sets generally differ.
This mismatch motivates the remaining components: OTG constructs future-aware outgoing messages, while CAMA aggregates received delayed messages.

\subsection{Observation Trajectory Generator}
\label{subsec:otg}
A message generated at timestep $t$ may be consumed several timesteps later, so a sender representation based only on the current observation may be misaligned when consumed.
OTG mitigates this temporal misalignment by constructing future-aware outgoing messages from predicted observations.

Let $H$ denote the prediction horizon.
At $t$, OTG outputs a predicted observation trajectory $\hat{\tau}_{i,t}=(\hat{o}_{i,t+1},\ldots,\hat{o}_{i,t+H})$.
The outgoing message is then formed by encoding the combined trajectory $(\tau_{i,t},\hat{\tau}_{i,t})$, where $\tau_{i,t}$ denotes the historical trajectory summary of agent $i$ up to $t$.

Because future observations depend on both teammates' actions and local dynamics, OTG factorizes prediction into two parts.
First, a teammate-action predictor estimates the actions of the agents whose messages are currently available:
\begin{equation}
	\hat{a}^{g}_{i,t}
	:=
	f^a_i\!\left(o_{i,t},\{m^r_{ij,t}\}_{j\in \mathcal{Z}_{i,t}}\right)
	=
	\{\hat{a}_{j,t}\}_{j\in \mathcal{Z}_{i,t}}.
	\label{eq:otg-action}
\end{equation}
Second, an observation-dynamics predictor estimates the next-step observation residual:
\begin{equation}\label{eq:predictobs}
	\hat{o}_{i,t+1}
	=
	o_{i,t}
	+
	f^o_i\!\left(e_{i,t},\{m^r_{ij,t}\}_{j\in \mathcal{Z}_{i,t}},\hat{a}^{g}_{i,t}\right),
\end{equation}
where $e_{i,t}$ denotes the local trajectory embedding.
Recursively applying Eqs.~\eqref{eq:otg-action}--\eqref{eq:predictobs} yields the $H$-step predicted trajectory $\hat{\tau}_{i,t}$.

OTG therefore constructs a future-aware outgoing message rather than simply encoding the sender's current state, keeping transmitted information better aligned with the receiver's decision context at consumption time.

\subsection{CGDC-guided Attentional Message Aggregator}
At receiver timestep $t$, agent $i$ may have multiple delayed messages available for decision making.
CAMA combines content-based attention with the request-time CGDC scores stored with the received messages in $\mathcal{Z}_{i,t}$.

For each $j\in\mathcal{Z}_{i,t}$, let $\bar c_{ij,t}>0$ denote the request-time predicted CGDC score stored with message $m^r_{ij,t}$.
Given the local trajectory embedding $e_{i,t}$, CAMA computes $q_{i,t}=W_i^Qe_{i,t}$, $k_{ij,t}=W_i^Km^r_{ij,t}$, and $v_{ij,t}=W_i^Vm^r_{ij,t}$ for $j\in\mathcal{Z}_{i,t}$, where $W_i^Q$, $W_i^K$, and $W_i^V$ are learnable projections.
It then modulates standard content-based attention by the stored CGDC score:
\begin{equation}\label{att_weights_cgdc}
	\alpha^{\mathrm{cgdc}}_{ij,t}
	=
	\frac{
		\bar c_{ij,t}\exp\!\left(\beta\,q^\top_{i,t}k_{ij,t}\right)
	}{
		\sum_{l\in\mathcal{Z}_{i,t}}
		\bar c_{il,t}\exp\!\left(\beta\,q^\top_{i,t}k_{il,t}\right)
	},
	\qquad j\in\mathcal{Z}_{i,t},
\end{equation}
where $\beta>0$ controls attention sharpness.
The stored CGDC score serves as a utility prior rather than a fixed aggregation weight, while the final attention also depends on the current receiver query and message content.
Appendix~\ref{app:cama-prior} shows that the weights in Eq.~\eqref{att_weights_cgdc} uniquely maximize a KL-regularized attention objective.

The aggregated delayed-message representation is $\tilde m_{i,t}=\sum_{j\in\mathcal{Z}_{i,t}}\alpha^{\mathrm{cgdc}}_{ij,t}v_{ij,t}$; if $\mathcal{Z}_{i,t}=\varnothing$, we define $\tilde m_{i,t}=\boldsymbol{0}$.
Thus, CAMA preserves content discrimination while using CGDC as a utility prior.
\vspace{-5pt}
\section{Experiments}
\label{sec:exp}

We evaluate CDCMA under four bounded stochastic delay regimes and a delay-free setting.
The experiments examine three aspects: performance against communication baselines, the contribution of each component and key hyperparameter, and behavior under delay-free and zero-shot cross-difficulty settings.
Implementation details and additional results, including MMM2, appear in Appendix~\ref{app:exp}.
\vspace{-5pt}
\subsection{Experimental Setup}
\label{sec:exp-setup}

\paragraph{Benchmarks and metrics.}
We evaluate CDCMA on variants of Cooperative Navigation (CN) and Predator Prey (PP) without teammate vision in Multi-Agent Particle Environment (MPE), and on two communication-sensitive SMAC maps: \texttt{1o\_2r\_vs\_4r} and \texttt{1o\_10b\_vs\_1r}.
We report mean episode reward on MPE and win rate on SMAC.
Performance tables report mean and standard deviation across four independent training seeds; learning-curve shading denotes one standard deviation. Evaluation and curve-smoothing details are given in Appendix~\ref{app:exp-details}.
\vspace{-5pt}
\paragraph{Baselines.}
We compare CDCMA with CoDe~\citep{song2025code}, DACOM~\citep{yuan2023dacom}, TGCNet~\citep{Zhang_He_Cheng_Li_2025}, T2MAC~\citep{sun2024t2mac}, SMS~\citep{xue2022efficient}, G2ANet~\mbox{\citep{liu2020multi}}, and ATOC~\citep{jiang2018learning}.
All methods share the same delay distributions and message-delivery rules.
The baselines use official code and hyperparameters, except CoDe, which we reimplement.
CoDe uses framework-provided message ages; CDCMA does not.
Throughout the experiments, DACOM denotes the variant without TimeNet.
Implementation and training details appear in Appendix~\ref{app:exp-details}.
\vspace{-5pt}
\paragraph{Delay settings.}
For each ordered pair $(j,i)$ and send timestep $k$, the communication delay is an integer $d_{ji,k}\in\{1,\dots,d_{\max}\}$ with $d_{\max}=5$, drawn from a difficulty-indexed probability mass function $p_\ell(d)$, where $\ell\in\{\texttt{easy},\texttt{medium},\texttt{hard},\texttt{super\_hard}\}$.
We also report a delay-free setting with $d_{ji,k}=0$.
The exact delay PMFs are given in Appendix~\ref{app:exp-delay-model}, and the evaluation protocol, implementation details, and hyperparameters are given in Appendix~\ref{app:exp-details}.
\vspace{-5pt}
\subsection{Main Results}
\label{sec:exp-main}

Table~\ref{results_table} reports performance across tasks and delay settings.
CDCMA achieves the highest mean performance on all four tasks under all four bounded stochastic delay regimes, outperforming CoDe and the TimeNet-free DACOM variant.
Its advantage persists across delay difficulties, although the margin varies by task and baseline.
For example, its advantage over CoDe on CN grows from $1.41$ under \texttt{easy} to $2.53$ under \texttt{super\_hard}, while that on \texttt{1o\_10b\_vs\_1r} remains above $48$ percentage points throughout.
CDCMA also remains ahead of strong baselines originally designed for delay-free communication, particularly TGCNet on SMAC and CN and T2MAC on PP.
Against TGCNet under \texttt{super\_hard}, CDCMA achieves $73.44\%$ versus $44.53\%$ on \texttt{1o\_10b\_vs\_1r} and $63.28\%$ versus $49.22\%$ on \texttt{1o\_2r\_vs\_4r}.
MPE and SMAC learning curves appear in Appendix~\ref{app:exp-curves}.

\begin{table*}[!t]
	\vspace{-5pt}
	\scriptsize
	\centering
\caption{Mean episode reward (win rate) across tasks under four bounded stochastic delay regimes. The highest mean is in bold.}
	\label{results_table}
	\resizebox{\textwidth}{!}{%
		\begin{tabular}{@{}llllllllll@{}}\toprule
			\textbf{Task} & \textbf{Difficulty} & \textbf{CDCMA} & \textbf{CoDe} & \textbf{DACOM} & \textbf{TGCNet} & \textbf{T2MAC} & \textbf{SMS} & \textbf{G2ANet} & \textbf{ATOC} \\
			\midrule
			\multirow{4}{*}{Cooperative Navigation}
			& \texttt{easy}        & \textbf{-1.70}(0.07) & -3.11(0.14) & -3.16(0.43) & -2.08(0.32) & -2.23(0.04) & -2.57(0.11) & -2.24(0.09) & -2.51(0.62) \\
			& \texttt{medium}      & \textbf{-1.80}(0.36) & -3.09(0.05) & -3.14(0.13) & -2.21(0.35) & -2.27(0.08) & -2.67(1.01) & -2.48(0.07) & -2.43(0.14) \\
			& \texttt{hard}        & \textbf{-1.78}(0.22) & -4.36(2.55) & -3.52(1.02) & -2.35(0.34) & -2.51(0.44) & -2.81(0.47) & -2.71(0.07) & -2.75(0.22) \\
			& \texttt{super\_hard} & \textbf{-1.89}(0.13) & -4.42(2.75) & -3.00(0.21) & -2.56(0.36) & -2.53(0.17) & -3.04(0.64) & -2.94(0.06) & -2.79(0.45) \\
			\midrule
			\multirow{4}{*}{Predator Prey}
			& \texttt{easy}        & \textbf{-0.93}(0.03) & -1.82(0.03) & -1.97(0.23) & -1.53(0.14) & -1.10(0.15) & -1.46(0.27) & -1.17(0.01) & -2.38(0.38) \\
			& \texttt{medium}      & \textbf{-0.91}(0.04) & -1.83(0.03) & -2.19(0.42) & -1.60(0.14) & -1.27(0.22) & -1.36(0.17) & -1.40(0.05) & -2.75(0.54) \\
			& \texttt{hard}        & \textbf{-1.02}(0.06) & -1.86(0.08) & -1.93(0.01) & -1.72(0.15) & -1.47(0.39) & -1.51(0.06) & -1.71(0.08) & -2.98(0.19) \\
			& \texttt{super\_hard} & \textbf{-0.97}(0.10) & -1.85(0.07) & -2.01(0.04) & -1.75(0.11) & -1.66(0.33) & -1.78(0.71) & -1.85(0.04) & -3.08(0.86) \\
			\midrule
			\multirow{4}{*}{\texttt{1o\_2r\_vs\_4r}(\%)}
			& \texttt{easy} & \textbf{87.50}(9.88) & 58.59(14.29) & 39.84(3.93) & 63.28(14.29) & 37.50(27.83) & 48.44(17.02) & 45.31(12.88) & 14.06(9.72) \\
			& \texttt{medium} & \textbf{75.00}(8.84) & 40.63(7.65) & 35.94(6.51) & 56.25(7.65) & 38.28(28.91) & 50.00(10.83) & 21.88(5.10) & 24.22(17.75) \\
			& \texttt{hard} & \textbf{68.75}(13.74) & 28.91(9.33) & 39.84(3.93) & 50.78(9.33) & 33.59(31.18) & 38.28(2.99) & 28.13(16.73) & 11.72(17.75) \\
			& \texttt{super\_hard} & \textbf{63.28}(11.23) & 24.22(6.93) & 38.28(6.93) & 49.22(6.93) & 20.31(22.46) & 40.63(14.88) & 33.59(10.94) & 17.19(12.10) \\
			\midrule
			\multirow{4}{*}{\texttt{1o\_10b\_vs\_1r}(\%)}
			& \texttt{easy} & \textbf{92.19}(4.03) & 32.81(1.80) & 11.72(4.69) & 66.41(6.44) & 33.59(38.98) & 13.28(17.75) & 44.53(25.94) & 0.00(0.00) \\
			& \texttt{medium} & \textbf{86.72}(6.93) & 35.16(4.69) & 10.16(8.22) & 42.97(2.99) & 28.91(34.55) & 10.16(20.31) & 23.44(10.97) & 0.00(0.00) \\
			& \texttt{hard} & \textbf{83.59}(6.44) & 24.22(22.88) & 10.94(6.51) & 54.69(14.09) & 31.25(28.30) & 0.00(0.00) & 8.59(8.98) & 0.00(0.00) \\
			& \texttt{super\_hard} & \textbf{73.44}(5.41) & 25.00(15.31) & 22.66(8.61) & 44.53(10.33) & 25.00(29.87) & 0.00(0.00) & 27.34(6.93) & 0.00(0.00) \\
			\bottomrule
		\end{tabular}%
	}
	\vspace{-5pt}
\end{table*}

These patterns are consistent with CDCMA's design: DCOS filters communication using predicted CGDC, OTG reduces consumption-time misalignment through future-aware messages, and CAMA uses CGDC as a utility prior for delayed-message aggregation.
\vspace{-5pt}
\subsection{Ablations and Sensitivity}
\label{sec:exp-ablation}

We evaluate module ablations on CN, horizon sensitivity on PP, and gain--cost weight sensitivity on \texttt{1o\_2r\_vs\_4r}.
Unless otherwise noted, all ablations are conducted under the \texttt{super\_hard} delay regime.
\vspace{-5pt}

\begin{wrapfigure}{r}{0.6\textwidth}
	\centering
	\subfloat[Module ablations\label{fig:abla_modules}]{
		\includegraphics[width=0.31\linewidth]{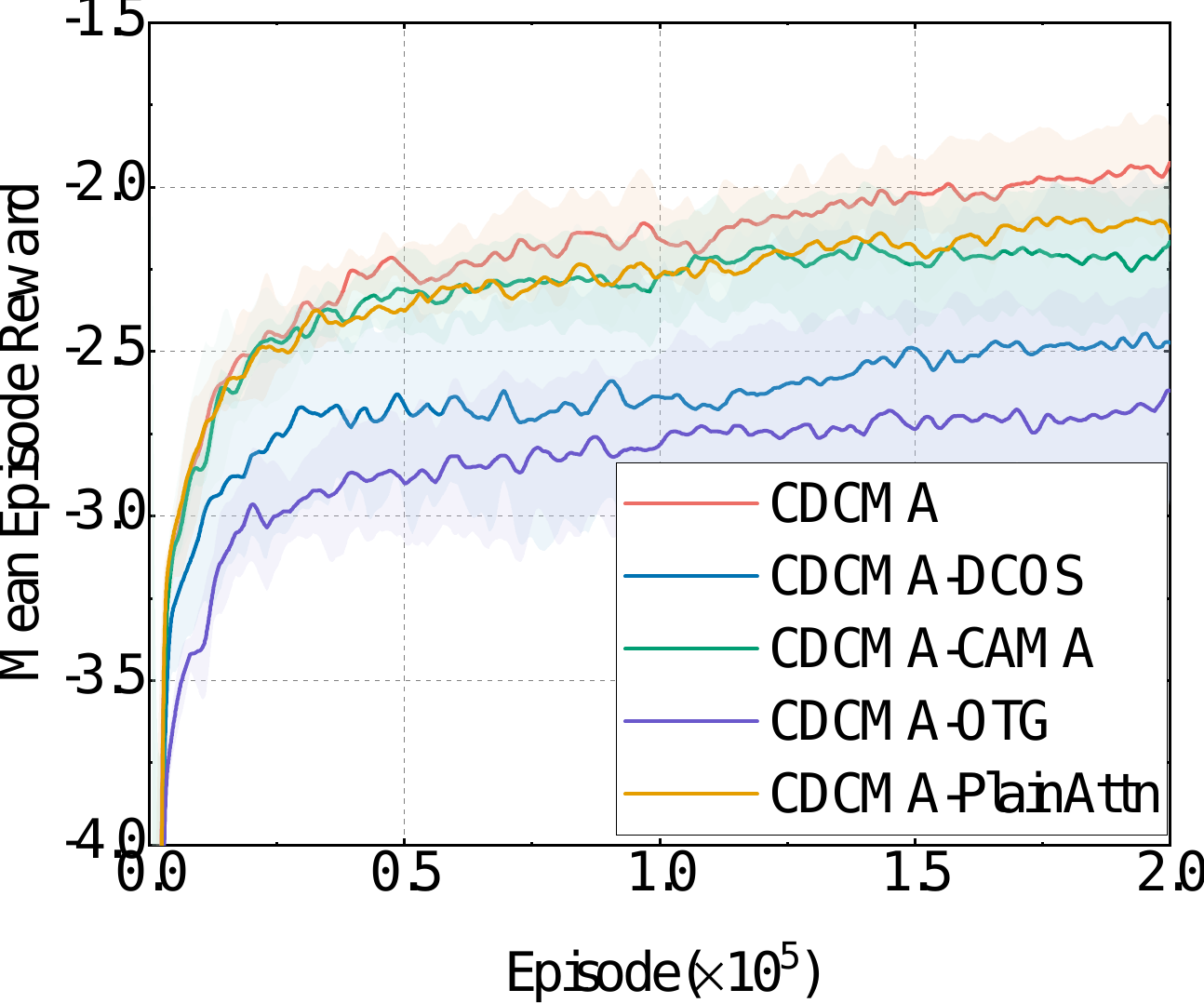}
	}\hfill
	\subfloat[Horizon $H$\label{fig:abla_H}]{
		\includegraphics[width=0.31\linewidth]{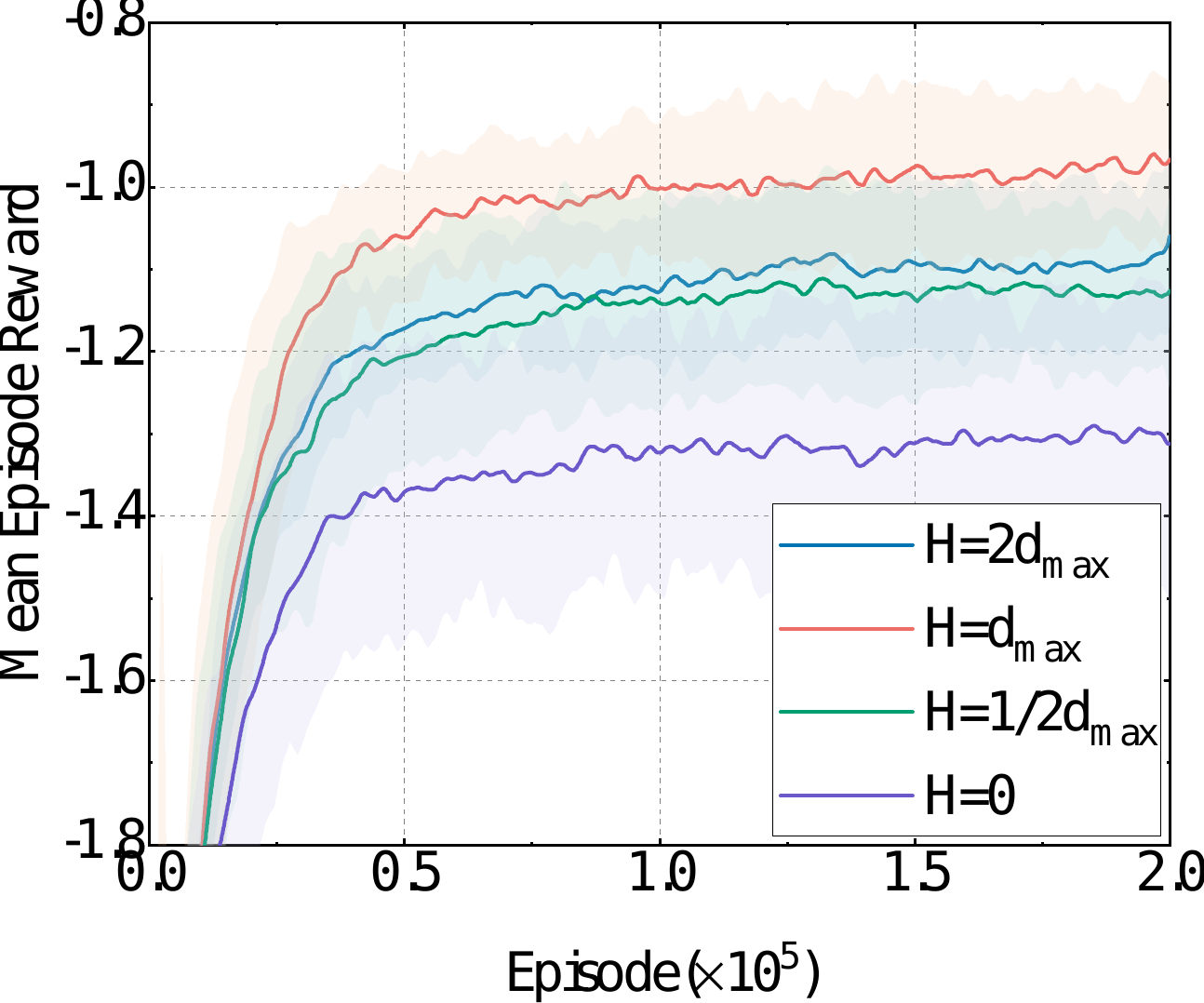}
	}\hfill
	\subfloat[CGDC weight $\lambda$\label{fig:abla_lambda}]{
		\includegraphics[width=0.31\linewidth]{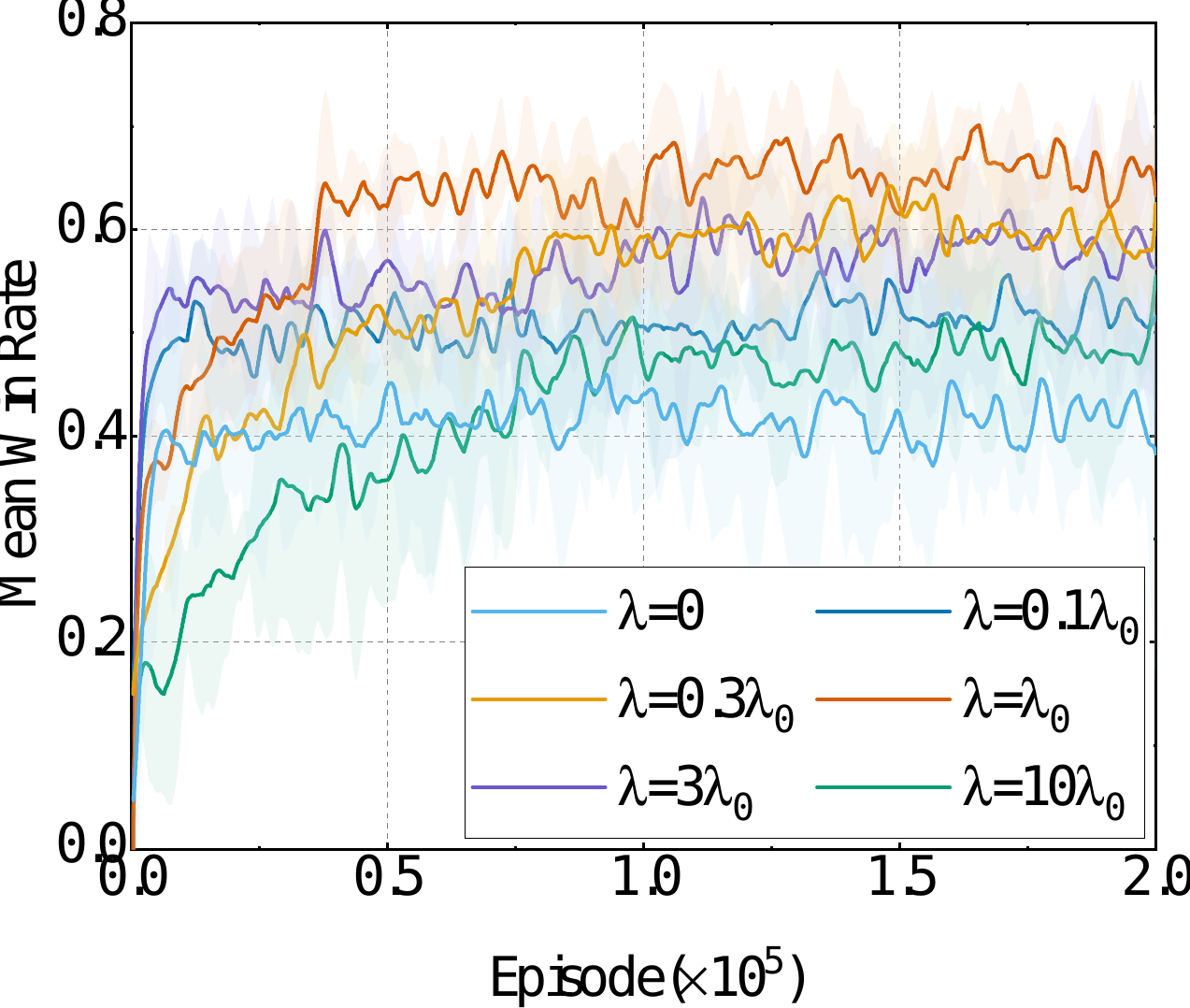}
	}
	\caption{Ablations and sensitivity analyses.}
	\label{fig:ablation_all}
	\vspace{-1.0em}
\end{wrapfigure}
\paragraph{Module ablations.}
We ablate the three main components of CDCMA on CN: \textit{CDCMA-OTG} (without OTG), \textit{CDCMA-CAMA} (without CAMA), \textit{CDCMA-DCOS} (without DCOS), and \textit{CDCMA-PlainAttn} (replacing CAMA with standard dot-product attention).
As shown in Fig.~\ref{fig:abla_modules}, removing any component degrades performance.
The largest drop in this setting occurs without OTG, indicating that future-aware outgoing messages provide the largest contribution among the evaluated module ablations.
\textit{CDCMA-PlainAttn} also underperforms CDCMA, showing that CAMA benefits from combining content attention with CGDC-based utility priors.
\vspace{-5pt}
\paragraph{Sensitivity to the prediction horizon $H$.}
We vary $H\in\{0,\lceil \tfrac{1}{2}d_{\max}\rceil,d_{\max},2d_{\max}\}$ on PP under \texttt{super\_hard} delays.
Compared with $H=0$ (no future prediction), performance improves up to $H=d_{\max}$ but deteriorates beyond this value.
This pattern is consistent with a coverage--rollout-error trade-off: increasing $H$ helps outgoing messages cover typical delays, whereas longer rollouts accumulate larger multi-step prediction errors.
Qualitative OTG diagnostics are provided in Appendix~\ref{app:exp-otg}.
\vspace{-5pt}
\paragraph{Sensitivity to the CGDC weight $\lambda$.}
On \texttt{1o\_2r\_vs\_4r} under \texttt{super\_hard}, we sweep $\lambda\in\{0,0.1\lambda_{0},0.3\lambda_{0},\lambda_{0},3\lambda_{0},10\lambda_{0}\}$, where $\lambda_0$ matches the gain and cost scales within each training minibatch; its definition and numerical safeguards are given in Appendix~\ref{app:weight-scaling}.
Unless otherwise stated, we use $\lambda=\lambda_0$.
Here, $\lambda=0$ is the gain-only variant without the delay-cost penalty.
Fig.~\ref{fig:abla_lambda} shows that both this variant and overly large $\lambda$ are suboptimal, whereas a moderate range around $\lambda_0$ performs best.
This pattern is consistent with the intended gain--cost balance: small $\lambda$ underweights delay cost, whereas large $\lambda$ can suppress otherwise useful communication.
\vspace{-5pt}
\subsection{Information-Gap Analysis and Robustness}
\label{sec:exp-robust}

We finally examine whether CDCMA reduces the decision-distribution gap and whether its performance persists under delay-free and zero-shot cross-difficulty settings.

\begin{wrapfigure}{r}{0.53\textwidth}
	\centering
	\subfloat[Predator Prey\label{pp_info}]{
		\includegraphics[width=0.47\linewidth]{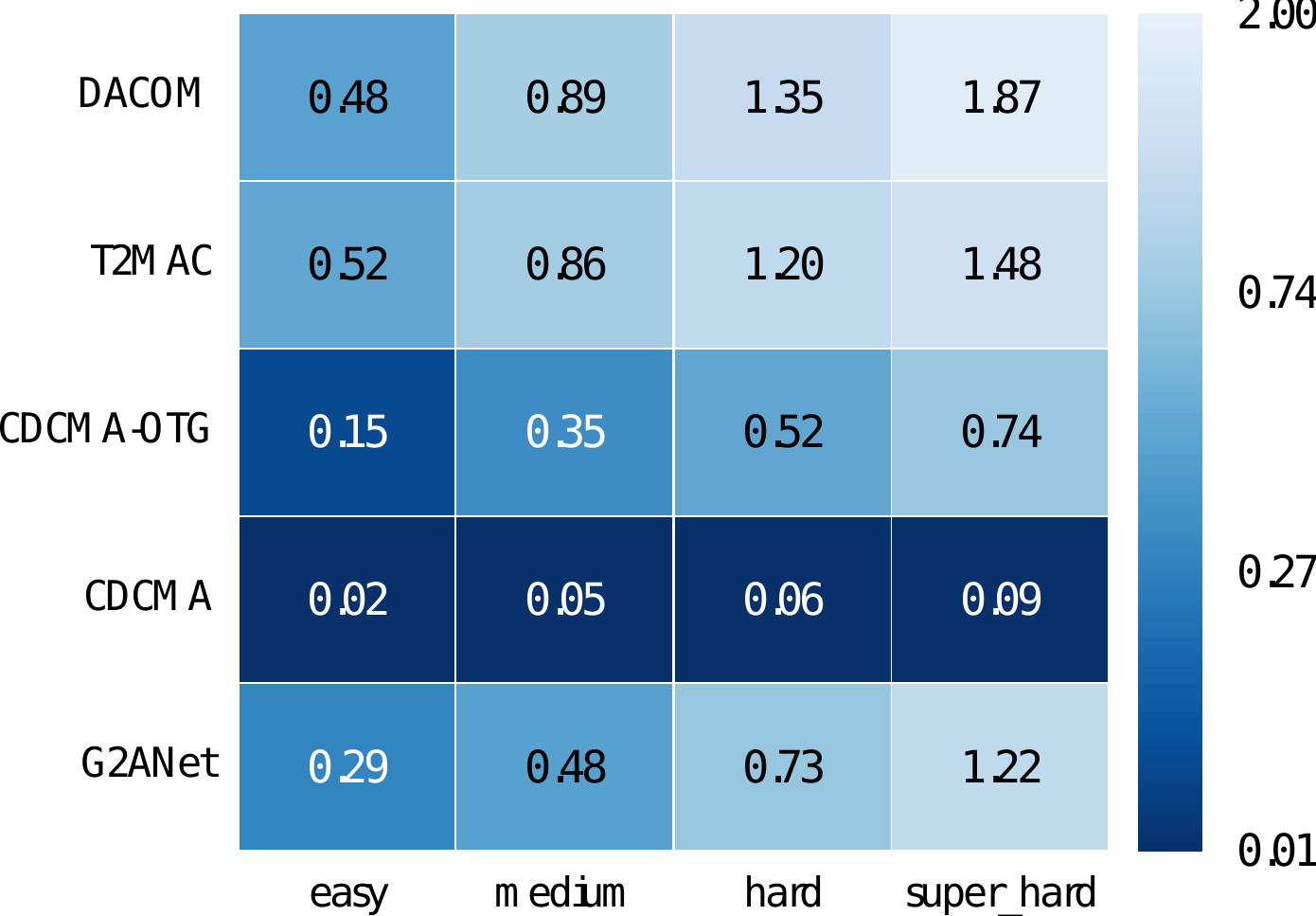}
	}\hfill
	\subfloat[\texttt{1o\_10b\_vs\_1r}\label{1o10b_info}]{
		\includegraphics[width=0.47\linewidth]{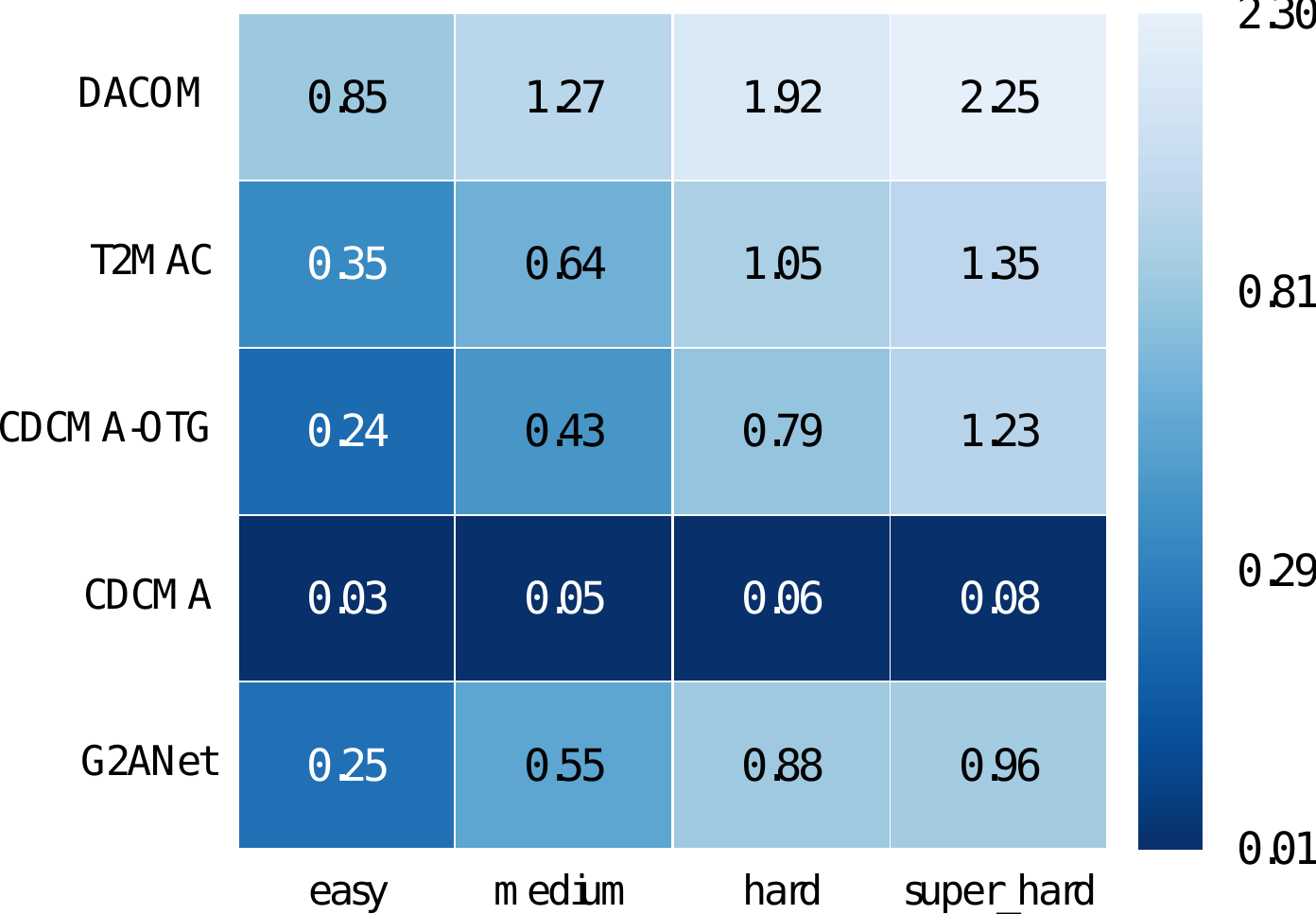}
	}
	\caption{Decision-distribution information gaps.}
	\label{information_fig}
	\vspace{-1.0em}
\end{wrapfigure}

\vspace{-5pt}
\paragraph{Decision-distribution information gaps.}
Motivated by Sec.~\ref{subsec:decomm-info-gap}, we compare decision distributions under timely-reference and delayed-message conditions, using actor outputs or softmax-normalized action values as appropriate for each method.
The timely branch replaces all received messages with the corresponding senders' current embeddings, keeping the network and all other inputs fixed.
Actions use delayed messages; the timely branch is used only for diagnosis.
Appendix~\ref{app:decision-gap} defines the metric and averaging.

We report gaps averaged over all agents and timesteps on PP and \texttt{1o\_10b\_vs\_1r} under \texttt{easy}--\texttt{super\_hard}.
Figure~\ref{information_fig} shows smaller decision-distribution gaps for CDCMA than the compared methods across all regimes, while its gap increases moderately with delay difficulty.
Removing OTG yields larger gaps, consistent with its role in mitigating temporal misalignment.
Together with Table~\ref{results_table} and Fig.~\ref{fig:abla_modules}, these results associate smaller decision-distribution gaps with stronger performance under the evaluated delays.
\vspace{-5pt}
\paragraph{Delay-free evaluation.}
To examine behavior in the zero-delay limit, we evaluate \texttt{1o\_10b\_vs\_1r} with $d_{ji,k}=0$ and set $H=0$ accordingly.
CDCMA remains stable and competitive across seeds in this zero-delay setting.
The corresponding learning curve is provided in Appendix~\ref{app:delay-free}.
\vspace{-5pt}
\newcommand{\msd}[2]{\ensuremath{#1(#2)}}

\noindent
\begin{minipage}[t]{0.45\textwidth}
	\vspace{0pt}
	\paragraph{Zero-shot cross-difficulty.}
	To test robustness under delay-distribution shift, we train on CN under \texttt{easy} and test zero-shot under \texttt{medium}, \texttt{hard}, and \texttt{super\_hard}.
	Table~\ref{tab:cn_zero_shot_compact} shows that CDCMA achieves the best mean episode reward under all three unseen delay difficulties and exhibits relatively modest degradation as the delay distribution becomes harder.
	From \texttt{easy} to \texttt{super\_hard}, its mean reward decreases by $0.23$, compared with $0.91$ for CoDe, $0.57$ for TGCNet, and $0.51$ for T2MAC.
\end{minipage}
\hfill
\begin{minipage}[t]{0.52\textwidth}
	\vspace{0pt}
	\centering
	\scriptsize
	\setlength{\tabcolsep}{2.5pt}
	\renewcommand{\arraystretch}{0.95}
		\captionsetup{skip=3pt}
		\captionof{table}{CN zero-shot cross-difficulty generalization. Train: \texttt{easy}; entries are mean(std) episode reward.}
		\label{tab:cn_zero_shot_compact}
		\resizebox{\linewidth}{!}{%
		\begin{tabular}{@{}lcccc@{}}
		\toprule
		\multirow{2}{*}{Method} &
		\multicolumn{1}{c}{Train} &
		\multicolumn{3}{c}{Test} \\
		\cmidrule(lr){2-2}\cmidrule(lr){3-5}
		& \texttt{easy} &
		\texttt{medium} &
		\texttt{hard} &
		\texttt{super\_hard} \\
		\midrule
		CDCMA  & \msd{\textbf{-1.70}}{\textbf{0.07}} & \msd{\textbf{-1.82}}{\textbf{0.15}} & \msd{\textbf{-1.84}}{\textbf{0.19}} & \msd{\textbf{-1.93}}{\textbf{0.09}} \\
		CoDe   & \msd{-3.11}{0.14} & \msd{-3.42}{0.19} & \msd{-3.59}{0.71} & \msd{-4.02}{1.88} \\
		DACOM  & \msd{-3.16}{0.43} & \msd{-3.21}{0.27} & \msd{-3.37}{0.39} & \msd{-3.29}{0.45} \\
		TGCNet & \msd{-2.08}{0.32} & \msd{-2.28}{0.33} & \msd{-2.41}{0.41} & \msd{-2.65}{0.47} \\
		T2MAC  & \msd{-2.23}{0.04} & \msd{-2.31}{0.12} & \msd{-2.62}{0.55} & \msd{-2.74}{0.59} \\
		SMS    & \msd{-2.57}{0.11} & \msd{-2.71}{0.94} & \msd{-2.89}{0.97} & \msd{-3.11}{0.77} \\
		G2ANet & \msd{-2.24}{0.09} & \msd{-2.57}{0.11} & \msd{-2.85}{0.09} & \msd{-3.01}{0.13} \\
		ATOC   & \msd{-2.51}{0.62} & \msd{-2.55}{0.26} & \msd{-2.87}{0.19} & \msd{-2.93}{0.37} \\
		\bottomrule
		\end{tabular}%
		}
	\end{minipage}

\section{Conclusion}

We studied cooperative MARL under cross-timestep communication delays, where messages may be informative yet temporally misaligned when consumed.
We formalized it as DeComm-POMG, introducing CGDC to assess delayed-message utility using no-message and timely-reference counterfactuals, deriving a delay-induced return-loss bound in terms of accumulated information gap.

We instantiated this view in CDCMA through CGDC-based partner selection, future-aware message construction, and CGDC-guided delayed-message aggregation.
On MPE tasks without teammate vision and on SMAC maps, CDCMA achieves higher mean performance than communication baselines under the evaluated bounded stochastic delays.
On the two representative tasks it yields smaller decision-distribution gaps.
CDCMA remains competitive in the evaluated zero-delay settings and retains favorable zero-shot performance across unseen delay difficulties on Cooperative Navigation.

\paragraph*{Limitations and future work}
Our study assumes bounded, discrete, reliably delivered, exogenous message delays, instantaneous requests, and a fixed OTG prediction horizon.
Extending the framework to nonstationary latency, packet loss, bandwidth constraints, and larger, noisier, or more stochastic environments remains important future work.

\clearpage
\bibliography{example_paper}
\bibliographystyle{iclr2027_conference}

\newpage
\appendix
\onecolumn

\section{Notation Summary}
\label{app:notation}

For convenience, Table~\ref{tab:notation_summary} summarizes the main paper-specific symbols used in the formulation and method.

\begin{table*}[h]
	\centering
	\scriptsize
	\caption{Summary of the main paper-specific symbols.}
	\label{tab:notation_summary}
	\resizebox{\textwidth}{!}{%
		\begin{tabular}{@{}lll@{}}
			\toprule
			\textbf{Category} & \textbf{Symbol} & \textbf{Meaning} \\
			\midrule
			
			\multirow{9}{*}{Delayed communication semantics}
			& $\boldsymbol{u}_t$ & Joint environment--communication action $(\boldsymbol{a}_t,\boldsymbol{m}^{s}_t)$ \\
			& $m^s_{ji,t}$ & Message sent from agent $j$ to agent $i$ at timestep $t$ \\
			& $m^r_{ij,t}$ & Message from agent $j$ that is available to agent $i$ at receiver timestep $t$ \\
			& $m^{\mathrm{tar}}_{ji,t}$ & Timely reference condition for sender $j$ at receiver timestep $t$ \\
			& $m^\emptyset$ & Null / masked message input \\
			& $d_{ji,t}$ & Sampled communication delay for link $(j\!\to\! i)$ at timestep $t$ \\
			& $\mathcal{B}_t$ & Message buffer at timestep $t$ \\
			& $\widetilde{s}_t$ & Augmented state $(s_t,\mathcal{B}_t)$ \\
			& $x_{i,t}$ & Message-free local input of agent $i$ at timestep $t$ \\
			\midrule
			
			\multirow{10}{*}{Gain, delay cost, and CGDC}
			& $V^g_{ij,t}$ & Communication gain from sender $j$ to receiver $i$ at timestep $t$ \\
			& $\widehat{V}^{g}_{ij,t}$ & Critic-based surrogate of communication gain \\
			& $\Delta_{\mathrm{info}}(i,j,t)$ & Sender-specific actor-policy information gap \\
				& $\Delta_{\mathrm{dec}}(i,t)$ & Decision-distribution gap for the complete received-message collection \\
			& $\mathcal{L}^{\mathrm{delay}}_i$ & Potential delay-induced return loss \\
			& $V^c_{ij,t}$ & Critic-tempered delay-cost surrogate \\
			& $\delta^{\mathrm{tar}}_{ij,t}$ & Actor--probe discrepancy under the timely reference \\
			& $\delta^{\mathrm{del}}_{ij,t}$ & Actor--probe discrepancy under the delayed message \\
			& $c_{ij,t}$ & Communication Gain and Delay Cost (CGDC) score \\
			& $\lambda$ & Weight balancing communication gain and delay cost \\
			\midrule
			
			\multirow{10}{*}{CDCMA modules}
			& $\phi_i$ & DCOS predictor for agent $i$ \\
			& $\hat c_{ij,t}$ & Predicted CGDC score for communication from $j$ to $i$ \\
			& $F_{ij,t}$ & Binary request mask indicating whether agent $i$ queries agent $j$ \\
			& $\mathcal{G}_{i,t}$ & Set of senders queried by agent $i$ at timestep $t$ \\
			& $\mathcal{Z}_{i,t}$ & Set of senders whose messages are available to agent $i$ at timestep $t$ \\
			& $e_{i,t}$ & Local trajectory embedding of agent $i$ at timestep $t$ \\
			& $\tau_{i,t}$ & Historical trajectory summary of agent $i$ up to timestep $t$ \\
			& $\hat{\tau}_{i,t}$ & Predicted future observation trajectory of agent $i$ \\
			& $H$ & OTG prediction horizon \\
			& $\tilde m_{i,t}$ & Aggregated delayed-message representation used by agent $i$ \\
			\midrule
			
			\multirow{5}{*}{Critics and attention}
			& $Q_i^{\mathrm{full}}$ & Centralized full critic used for actor--critic optimization \\
			& $Q_i^{\mathrm{gain}}$ & Auxiliary gain critic used to construct gain/cost surrogates \\
			& $\widetilde{\pi}_i(m)$ & Critic-induced tempered policy under message condition $m$ \\
			& $\bar c_{ij,t}$ & Request-time predicted CGDC score stored with a delayed message \\
			& $\alpha^{\mathrm{cgdc}}_{ij,t}$ & CAMA attention weight on message from sender $j$ to receiver $i$ \\
			\bottomrule
	\end{tabular}}
\end{table*}

\section{Extended Related Work}
\label{app:extended-related-work}

Communication learning is a longstanding theme in MARL. Most existing approaches study how agents communicate in delay-free settings, but differ in whether they primarily emphasize broad information sharing, selective communication, or delayed-message handling. Below, we summarize these lines of work and clarify the positioning of CDCMA.

\subsection{Centralized Communication}

Centralized communication protocols enable broad information sharing, commonly through differentiable message passing. Canonical examples include CommNet~\citep{sukhbaatar2016learning}, DIAL~\citep{foerster2016learning}, and BiCNet~\citep{peng2017multiagent}. Follow-up work improves message expressiveness through recurrent aggregation and richer communication memory, e.g., FCMNet~\citep{wang2022fcmnet}, and enhances discreteness, interpretability, or targeted addressing, e.g., Diff Discrete~\citep{freed2020communication}, TarMAC~\citep{das2019tarmac}, and IS~\citep{kim2020communication}. Other methods reduce redundancy or improve robustness by filtering or dropping messages, e.g., DCC-MD~\citep{kim2019message} and $\Re$\!-MACRL~\citep{xue2022mis}, or by introducing auxiliary structures such as centralized proxies and self-supervised objectives, e.g., IMAC~\citep{wang2020learning} and MASIA~\citep{guan2022efficient}. Symbolic or grounded communication has also been explored, e.g., GCL~\citep{mordatch2018emergence} and AE-Comm~\citep{lin2021learning}. While these methods substantially enrich communication modeling in MARL, they are mainly developed for settings in which exchanged messages are consumed within the current decision step and do not explicitly address cross-timestep staleness.

\subsection{Selective Communication}

Selective communication improves scalability and efficiency by restricting exchange to a subset of agents, entities, or timesteps. Early work studies learned communication structures and interaction graphs, e.g., NeurComm~\citep{chu2020multi} and Agent-Entity Graph~\citep{agarwal2020learning}, as well as bandwidth- or event-constrained communication, e.g., SchedNet~\citep{kim2018learning} and ETCNet~\citep{hu2023event}. Other methods introduce utility-driven selection objectives, e.g., I2C~\citep{ding2020learning}, VBC~\citep{zhang2019efficient}, NDQ~\citep{wang2019learning}, and TMC~\citep{zhang2020succinct}. More recent work estimates message importance using game-theoretic or counterfactual signals, e.g., SMS~\citep{xue2022efficient} and CMVC~\citep{11151221}, and personalizes communication to specific teammates, e.g., MAIC~\citep{yuan2022multi}. A broad family of attention- or graph-based selectors and aggregators has also been developed, including ATOC~\citep{jiang2018learning}, Gated-ACML~\citep{mao2020learning}, G2ANet~\citep{liu2020multi}, MAGIC~\citep{niu2021multi}, AERL~\citep{pu2022attention}, InforMARL~\citep{nayak2023scalable}, FlowComm~\citep{du2021learning}, TGCNet~\citep{Zhang_He_Cheng_Li_2025}, and GRDC~\citep{gao2026grdc}. Recent extensions further consider adversarial robustness, hierarchical dependency modeling, and improved aggregation or coordination strategies, e.g., AME~\citep{sun2023certifiably}, DHCG~\citep{liu2023deep}, T2MAC~\citep{sun2024t2mac}, and SeqComm~\citep{ding2024multi}. These methods significantly improve communication efficiency and selectivity, but are mostly designed for delay-free communication or do not explicitly formulate the receiver-side utility of stale messages under cross-timestep delays.

\subsection{Communication with Delays}

A smaller but closely related line of work explicitly studies MARL communication under delayed message arrivals. DACOM~\citep{yuan2023dacom} mitigates latency by adjusting waiting times through a delay-aware communication mechanism and is primarily geared toward short-horizon delay handling. CoDe~\citep{song2025code} infers long-term intents through future-action prediction and fuses asynchronous messages via dual alignment of intent and timeliness. These methods highlight the importance of delayed communication, but they do not explicitly cast the problem as evaluating the receiver-side utility of a stale message at the time it is consumed.

Beyond communication delays, RDC~\citep{fu2026rainbow} formulates stochastic individual observation delays and compensates delayed observation components. Its delay acts on locally received observations, whereas our setting concerns learned messages transmitted between agents across decision timesteps.

In contrast, CDCMA is built around an explicit gain--cost view of delayed-message utility under cross-timestep delays. We formalize this setting as DeComm-POMG, where delayed delivery is part of the decision process, and evaluate delayed communication by balancing communication gain against delay cost. DCOS, OTG, and CAMA instantiate this view through partner selection, future-aware outgoing messages, and gain--cost-guided aggregation, respectively.

\section{Additional Theory for Section~\ref{sec:pro_formu}}
\label{app:theory}

\subsection{Proof of Proposition~\ref{prop:markov-decomm}}
\label{app:proof:markov-decomm}

\begin{proof}
	Let $\boldsymbol{u}_t=(\boldsymbol{a}_t,\boldsymbol{m}^{s}_t)$ denote the joint environment--communication action.
	Given the current augmented state $\widetilde{s}_t=(s_t,\mathcal{B}_t)$ and $\boldsymbol{u}_t$, the next environment state is sampled from $\mathcal{P}(\cdot\mid s_t,\boldsymbol{a}_t)$.
	The delay variables are sampled from the exogenous Markov delay process $\varphi$, and the next message buffer is determined by the current buffer, the outgoing messages, and the sampled delays:
	\begin{equation}
		\mathcal{B}_{t+1}
		=
		F_{\mathcal{B}}\bigl(\mathcal{B}_t,\boldsymbol{m}^{s}_t,\boldsymbol{d}_t\bigr).
	\end{equation}
	Therefore, the conditional distribution of $\widetilde{s}_{t+1}=(s_{t+1},\mathcal{B}_{t+1})$ depends only on $(\widetilde{s}_t,\boldsymbol{u}_t)$, which establishes the controlled Markov property.

	The augmented POMG uses the same local observation process as DeComm-POMG and exposes to each receiver the messages delivered by $\mathcal{B}_t$.
	Under the same decentralized policy and the same realizations of environment transitions and communication delays, the two formulations consequently generate identical local observation, action, and message histories.
\end{proof}

\subsection{Proof of Uniform Bound on $V^g_{ij,t}$}
\label{app:gain-bound}

\begin{proposition}[Uniform bound on $V^g_{ij,t}$]
	\label{prop:gain-bound}
	Assume the per-timestep reward is bounded as $|r_i(\widetilde{s}_t,\boldsymbol{a}_t)|\le R_{\max}$ for all $t$ and $\gamma\in(0,1)$. Then for any $(i,j,t,\widetilde{s}_t,\boldsymbol{a}_t)$,
	\begin{equation}
		\big|V^g_{ij,t}(\widetilde{s}_t,\boldsymbol{a}_t)\big|
		\;\le\;
		\frac{2R_{\max}}{1-\gamma}.
		\label{eq:gain-bound}
	\end{equation}
\end{proposition}

\begin{proof}
	Fix any augmented state--action pair $(\widetilde{s}_t,\boldsymbol{a}_t)$. By definition, $V^g_{ij,t}(\widetilde{s}_t,\boldsymbol{a}_t)=Q_i^{(+j)}(\widetilde{s}_t,\boldsymbol{a}_t)-Q_i^{(0)}(\widetilde{s}_t,\boldsymbol{a}_t)$. By the triangle inequality,
	\begin{equation}
		\big|V^g_{ij,t}(\widetilde{s}_t,\boldsymbol{a}_t)\big|
		\le
		\big|Q_i^{(+j)}(\widetilde{s}_t,\boldsymbol{a}_t)\big|
		+
		\big|Q_i^{(0)}(\widetilde{s}_t,\boldsymbol{a}_t)\big|.
		\label{eq:gain-bound-triangle}
	\end{equation}
	
	For any joint policy $\boldsymbol{\pi}$, the augmented-state action-value satisfies $Q_i^{\boldsymbol{\pi}}(\widetilde{s}_t,\boldsymbol{a}_t)=\mathbb{E}_{\boldsymbol{\pi}}[\sum_{k\ge 0}\gamma^k r_i(\widetilde{s}_{t+k},\boldsymbol{a}_{t+k}) \mid \widetilde{s}_t,\boldsymbol{a}_t]$. Using $|r_i|\le R_{\max}$ and $|\mathbb{E}[X]|\le \mathbb{E}[|X|]$ gives $\big|Q_i^{\boldsymbol{\pi}}(\widetilde{s}_t,\boldsymbol{a}_t)\big|\le \sum_{k\ge 0}\gamma^k R_{\max}=R_{\max}/(1-\gamma)$. Applying this bound to both $\boldsymbol{\pi}^{(+j)}$ and $\boldsymbol{\pi}^{(0)}$ and substituting into Eq.~\eqref{eq:gain-bound-triangle} yields Eq.~\eqref{eq:gain-bound}.
\end{proof}

\subsection{Proof of Theorem~\ref{thm:value-kl}}
\label{app:proof:value-kl}

\begin{proof}
	Let $\boldsymbol{\pi}^{\mathrm{timely}}$ and $\boldsymbol{\pi}^{\mathrm{delayed}}$ be defined as in Theorem~\ref{thm:value-kl}. We work on the augmented Markov game induced by their shared message-generation protocol, with any internal state of that protocol included in $s_t$ for analysis; its remaining control variable is the joint environment action $\boldsymbol{a}_t$.
	
	Define the action-value function under $\boldsymbol{\pi}^{\mathrm{delayed}}$ as
	\begin{equation}
		\label{eq:proof-delayed-q}
		Q_i^{\mathrm{delayed}}(\widetilde{s}_t,\boldsymbol{a}_t)
		:=
		\mathbb{E}\!\left[\sum_{k\ge 0}\gamma^k r_i(\widetilde{s}_{t+k},\boldsymbol{a}_{t+k})
		\;\middle|\; \widetilde{s}_t,\boldsymbol{a}_t, \boldsymbol{\pi}^{\mathrm{delayed}}\right],
	\end{equation}
	and the corresponding value function as
	\begin{equation}
		\label{eq:proof-delayed-v}
		V_i^{\mathrm{delayed}}(\widetilde{s}_t)
		:=
		\mathbb{E}_{\boldsymbol{a}_t\sim \boldsymbol{\pi}^{\mathrm{delayed}}(\cdot\mid \widetilde{s}_t)}
		\!\left[Q_i^{\mathrm{delayed}}(\widetilde{s}_t,\boldsymbol{a}_t)\right].
	\end{equation}
	
	By a standard telescoping argument on the augmented process, we have
	\begin{align}
		\label{eq:proof-pdl-step1}
		J_i(\boldsymbol{\pi}^{\mathrm{timely}})-J_i(\boldsymbol{\pi}^{\mathrm{delayed}})
		&=
		\mathbb{E}_{\tau\sim \boldsymbol{\pi}^{\mathrm{timely}}}\!\Bigg[
		\sum_{t\ge 0}\gamma^t
		\Big(
		\mathbb{E}_{\boldsymbol{a}_t\sim\mu_t}\!\big[
		r_i(\widetilde{s}_t,\boldsymbol{a}_t)+\gamma V_i^{\mathrm{delayed}}(\widetilde{s}_{t+1})
		\big]
		\nonumber\\
		&\qquad\qquad\qquad
		-
		\mathbb{E}_{\boldsymbol{a}_t\sim\nu_t}\!\big[
		r_i(\widetilde{s}_t,\boldsymbol{a}_t)+\gamma V_i^{\mathrm{delayed}}(\widetilde{s}_{t+1})
		\big]
		\Big)
		\Bigg],
	\end{align}
	where $\mu_t:=\boldsymbol{\pi}^{\mathrm{timely}}(\cdot\mid \widetilde{s}_t)$ and $\nu_t:=\boldsymbol{\pi}^{\mathrm{delayed}}(\cdot\mid \widetilde{s}_t)$.
	
	The reward bound implies $|V_i^{\mathrm{delayed}}(\widetilde{s})|\le R_{\max}/(1-\gamma)$.
	For a fixed $\widetilde{s}_t$, define
	$g_i(\boldsymbol{a}_t):=r_i(\widetilde{s}_t,\boldsymbol{a}_t)+\gamma\,\mathbb{E}[V_i^{\mathrm{delayed}}(\widetilde{s}_{t+1})\mid\widetilde{s}_t,\boldsymbol{a}_t]$.
	Then $\|g_i\|_\infty\le R_{\max}/(1-\gamma)$, and the standard total-variation inequality for bounded functions gives
	\begin{equation}
		\label{eq:proof-pdl-step2}
		\left|\mathbb{E}_{\mu_t}[g_i]-\mathbb{E}_{\nu_t}[g_i]\right|
		\le
		\frac{2R_{\max}}{1-\gamma}\|\mu_t-\nu_t\|_{\mathrm{TV}}.
	\end{equation}
	Substituting this bound yields
	\begin{equation}
		\label{eq:proof-pdl-step3}
		J_i(\boldsymbol{\pi}^{\mathrm{timely}})-J_i(\boldsymbol{\pi}^{\mathrm{delayed}})
		\le
		\frac{2R_{\max}}{1-\gamma}\,
		\mathbb{E}_{\tau\sim \boldsymbol{\pi}^{\mathrm{timely}}}\!\left[
		\sum_{t\ge 0}\gamma^t \|\mu_t-\nu_t\|_{\mathrm{TV}}
		\right].
	\end{equation}
	
	Moreover, the two policies share the same message-generation protocol and the same environment-action policies for all agents other than $i$. Under the factorized decentralized joint policy, they differ only in whether agent $i$ conditions its environment action on $m^{\mathrm{tar}}_{ji,t}$ or $m^{r}_{ij,t}$. Hence,
	\begin{equation}
		\label{eq:proof-pdl-step4}
		\|\mu_t-\nu_t\|_{\mathrm{TV}}
		\le
		\left\|
		\pi_i(\cdot\mid x_{i,t},m^{\mathrm{tar}}_{ji,t})
		-
		\pi_i(\cdot\mid x_{i,t},m^{r}_{ij,t})
		\right\|_{\mathrm{TV}}.
	\end{equation}
	Finally, Pinsker's inequality gives
	\begin{equation}
		\label{eq:proof-pdl-step5}
		\left\|
		\pi_i(\cdot\mid x_{i,t},m^{\mathrm{tar}}_{ji,t})
		-
		\pi_i(\cdot\mid x_{i,t},m^{r}_{ij,t})
		\right\|_{\mathrm{TV}}
		\le
		\sqrt{\tfrac12\,\Delta_{\mathrm{info}}(i,j,t)}.
	\end{equation}
	Substituting Eq.~\eqref{eq:proof-pdl-step5} into Eq.~\eqref{eq:proof-pdl-step3} and taking the positive part of the left-hand side proves Eq.~\eqref{eq:value-kl-bound}.
\end{proof}

\section{Additional Method Details for Section~\ref{sec:method}}
\label{app:method-details}

\subsection{Justification for the Critic-Induced Surrogate}
\label{app:surrogate-justification}

During centralized training, we construct the delay-cost surrogate $V^c_{ij,t}$ from the tempered policy induced by the gain critic in Eq.~\eqref{eq:temp-policy}.
Corollary~\ref{cor:surrogate-bound} relates this surrogate to the actor-policy return-loss bound through explicit actor--probe discrepancy terms.
The offline diagnostic in Appendix~\ref{app:decision-gap} compares method-specific decision distributions under complete timely-reference and delayed-message collections.

\begin{proof}[Proof of Corollary~\ref{cor:surrogate-bound}]
	By the triangle inequality for total variation,
	\begin{align}
		\|\pi_i^{\mathrm{tar}}-\pi_i^{\mathrm{del}}\|_{\mathrm{TV}}
		\le{}&
		\|\pi_i^{\mathrm{tar}}-\widetilde\pi_i^{\mathrm{tar}}\|_{\mathrm{TV}}
		+
		\|\widetilde\pi_i^{\mathrm{tar}}-\widetilde\pi_i^{\mathrm{del}}\|_{\mathrm{TV}}
		+
		\|\widetilde\pi_i^{\mathrm{del}}-\pi_i^{\mathrm{del}}\|_{\mathrm{TV}}
		\nonumber\\
		\le{}&
		\delta^{\mathrm{tar}}_{ij,t}
		+\sqrt{\tfrac12 V^c_{ij,t}}
		+\delta^{\mathrm{del}}_{ij,t},
		\label{eq:proof-surrogate-tv}
	\end{align}
	where the second inequality follows from Pinsker's inequality and Eq.~\eqref{eq:vc}.
	Substituting Eq.~\eqref{eq:proof-surrogate-tv} for the actor-policy TV term in Eq.~\eqref{eq:proof-pdl-step3} and taking the positive part of the return gap proves Eq.~\eqref{eq:surrogate-return-bound}.
\end{proof}

\paragraph{Interpretation of the critic-tempered discrepancy.}
The surrogate measures the change in the gain critic's tempered action distribution between the timely-reference and delayed-message conditions under the same base inputs.
It supplies the delay-cost term in CGDC.

\subsection{Prior-Regularized Attention View for CAMA}
\label{app:cama-prior}

Fix agent $i$ and timestep $t$ with a nonempty available-sender set $\mathcal{Z}_{i,t}$, and define the content score $s_{ij,t}\triangleq q_{i,t}^{\top}k_{ij,t}$. The score stored with the message is $\bar c_{ij,t}=\hat c_{ij,\kappa_{ji}(t)}>0$, where $\kappa_{ji}(t)$ is its send timestep; it need not equal the current prediction $\hat c_{ij,t}$. These stored scores induce the prior
\begin{equation}\label{eq:cama-prior}
	\bar{\alpha}_{ij,t}
	\triangleq
	\frac{\bar c_{ij,t}}{\sum_{l\in\mathcal{Z}_{i,t}} \bar c_{il,t}},
	\qquad j\in\mathcal{Z}_{i,t}.
\end{equation}
Note that $\bar{\alpha}_{ij,t}>0$ for all $j\in\mathcal{Z}_{i,t}$, hence $\mathrm{KL}(\boldsymbol{\alpha}_t\Vert \bar{\boldsymbol{\alpha}}_t)$ is well-defined on $\Delta_{\mathcal{Z}_{i,t}}$.

Consider the following KL-regularized objective over the simplex $\Delta_{\mathcal{Z}_{i,t}}=\{\boldsymbol{\alpha}_t:\alpha_{ij,t}\ge 0,\ \sum_{j\in\mathcal{Z}_{i,t}}\alpha_{ij,t}=1\}$:
\begin{equation}\label{eq:cama-opt}
	\mathcal{L}_{\mathrm{attn}}(\boldsymbol{\alpha}_t)
	=
	\sum_{j\in\mathcal{Z}_{i,t}} \alpha_{ij,t}\, s_{ij,t}
	-
	\frac{1}{\beta}\,
	\mathrm{KL}\big(\boldsymbol{\alpha}_t \,\Vert\, \bar{\boldsymbol{\alpha}}_t\big),
\end{equation}
where $\beta>0$ and $\mathrm{KL}(\boldsymbol{\alpha}\Vert \bar{\boldsymbol{\alpha}})=\sum_{j\in\mathcal{Z}_{i,t}} \alpha_{ij,t}\log\frac{\alpha_{ij,t}}{\bar{\alpha}_{ij,t}}$.

\begin{proposition}[Optimality of CAMA weights]
	\label{prop:cama-optimal}
	The unique maximizer of $\mathcal{L}_{\mathrm{attn}}(\boldsymbol{\alpha}_t)$ over $\Delta_{\mathcal{Z}_{i,t}}$ is exactly the CAMA weight $\boldsymbol{\alpha}^{\mathrm{cgdc}}_t=(\alpha^{\mathrm{cgdc}}_{ij,t})_{j\in\mathcal{Z}_{i,t}}$ defined in Eq.~\eqref{att_weights_cgdc}.
\end{proposition}

\begin{proof}
	For brevity, drop the indices $(i,t)$ and write $\mathcal{Z}$ for $\mathcal{Z}_{i,t}$, $s_j$ for $s_{ij,t}$, $\alpha_j$ for $\alpha_{ij,t}$, and $\bar{\alpha}_j$ for $\bar{\alpha}_{ij,t}$. Consider the Lagrangian with multiplier $\lambda$ for the constraint $\sum_{j\in\mathcal{Z}}\alpha_j=1$:
	\begin{equation}\label{eq:proof-cama}
		\mathcal{J}(\boldsymbol{\alpha},\lambda)
		=
		\sum_{j\in\mathcal{Z}}\alpha_j s_j
		-
		\frac{1}{\beta}\sum_{j\in\mathcal{Z}}\alpha_j\log\frac{\alpha_j}{\bar{\alpha}_j}
		+
		\lambda\Big(\sum_{j\in\mathcal{Z}}\alpha_j-1\Big).
	\end{equation}
	Taking derivatives and setting them to zero gives, for each $j\in\mathcal{Z}$, $0=\frac{\partial \mathcal{J}}{\partial \alpha_j}=s_j-\frac{1}{\beta}\big(\log\frac{\alpha_j}{\bar{\alpha}_j}+1\big)+\lambda$. Rearranging yields $\alpha_j=C\,\bar{\alpha}_j\,\exp(\beta s_j)$, where $C>0$ is independent of $j$. Enforcing $\sum_{j\in\mathcal{Z}}\alpha_j=1$ gives $\alpha_j=\bar{\alpha}_j\exp(\beta s_j)/\sum_{l\in\mathcal{Z}}\bar{\alpha}_l\exp(\beta s_l)$. Substituting Eq.~\eqref{eq:cama-prior} recovers Eq.~\eqref{att_weights_cgdc}. Uniqueness follows because $-\mathrm{KL}(\cdot\Vert \bar{\boldsymbol{\alpha}})$ is strictly concave on the simplex and the first term is linear.
\end{proof}

\subsection{Training Scheme}
\label{app:training-scheme}

CDCMA comprises seven components: observation encoder $e_i$, DCOS predictor $\phi_i$, OTG predictors $(f_i^{a},f_i^{o})$, CAMA projections $(W_i^{Q},W_i^{K},W_i^{V})$, decentralized actor $\pi_i$, a centralized full critic $Q^{\mathrm{full}}_i$, and an auxiliary gain critic $Q^{\mathrm{gain}}_i$. We share network parameters across agents and condition networks on an agent-identity embedding when needed.

Under centralized training with decentralized execution, we train an off-policy actor--critic with target networks $(\cdot)^{-}$. For each agent $i$, the full critic is trained by minimizing the TD loss
\begin{equation}\label{qloss}
	\mathcal{L}(\theta^{Q^{\mathrm{full}}_i})
	= \mathbb{E}\big[(Q^{\mathrm{full}}_{i}(\boldsymbol{o}, \boldsymbol{a}, \widetilde{m}_{i})-y_{i})^{2}\big],
\end{equation}
where $y_{i} = r_{i} + \gamma\, Q_{i}^{\mathrm{full}-}(\boldsymbol{o}', \boldsymbol{a}', \widetilde{m}_{i}^{\,'})$ and $\boldsymbol{a}'$ is formed by sampling each $a_i' \sim \pi_{i}^{-}(e_{i}',\widetilde{m}_{i}^{\,'})$.

\paragraph{Gain critic for critic-induced surrogates.}
To instantiate the critic-induced tempered policy in Eq.~\eqref{eq:temp-policy} and the gain surrogate in Eq.~\eqref{eq:vg-critic}, we introduce an auxiliary gain critic $Q^{\mathrm{gain}}_i$. For each ordered pair $(i,j)$ with $j\neq i$, we evaluate a present/absent message pair under the same base inputs $(a_i,\boldsymbol{o}_{-j},\boldsymbol{a}_{-ij})$, namely $Q^{\mathrm{gain}}_i(a_i,\boldsymbol{o}_{-j},\boldsymbol{a}_{-ij},m^{r}_{ij})$ and $Q^{\mathrm{gain}}_i(a_i,\boldsymbol{o}_{-j},\boldsymbol{a}_{-ij},m^{\emptyset})$. The two targets share the sampled reward and continuation actions but use next-step critic inputs with and without the message, respectively.

Concretely, the gain critic is trained by minimizing
\begin{equation}\label{qloss-gain}
	\begin{aligned}
	\mathcal{L}(\theta^{Q^{\mathrm{gain}}_i})
	=
	\mathbb{E}\Bigg[
	\sum_{j\in\mathcal{N}\setminus\{i\}}
	\Big(&
	\big(Q^{\mathrm{gain}}_i(a_i,\boldsymbol{o}_{-j},\boldsymbol{a}_{-ij},m^{r}_{ij})-y^{\mathrm{pres}}_{ij}\big)^2
	\\[-2pt]
	&+
	\big(Q^{\mathrm{gain}}_i(a_i,\boldsymbol{o}_{-j},\boldsymbol{a}_{-ij},m^{\emptyset})-y^{\mathrm{abs}}_{ij}\big)^2
	\Big)
	\Bigg].
	\end{aligned}
\end{equation}
where the TD targets are $y^{\mathrm{pres}}_{ij}=r_i+\gamma\,Q^{\mathrm{gain}-}_i(a_i',\boldsymbol{o}'_{-j},\boldsymbol{a}'_{-ij},m^{r\,'}_{ij})$ and $y^{\mathrm{abs}}_{ij}=r_i+\gamma\,Q^{\mathrm{gain}-}_i(a_i',\boldsymbol{o}'_{-j},\boldsymbol{a}'_{-ij},m^{\emptyset})$. Here $m^{r\,'}_{ij}$ denotes the received message from $j$ to $i$ in the next-timestep transition. We sample $\boldsymbol{a}'$ once from the target actors conditioned on the full received-message aggregates, $a_i'\sim \pi_i^{-}(e_i',\widetilde m_i^{\,'})$, and reuse it in both targets.

\paragraph{Interpretation of the gain surrogate.}
Both branches reuse the sampled reward and continuation actions, while differing in their current and bootstrap message inputs.
Their value difference supplies a message-conditioning signal for CGDC, motivated by the policy-level gain in Eq.~\eqref{eq:q-gain}.
The policy-level quantity compares distinct continuation policies, whereas the implemented surrogate compares learned value predictions under shared continuation actions.

\paragraph{Actor--critic updates.}
For each actor update, we use local inputs from replay and resample $a_i\sim\pi_i(\cdot\mid e_i,\widetilde m_i)$ from the current actor rather than reusing its recorded action. We update the actor using the score-function policy gradient with the centralized full critic:
\begin{equation}\label{ploss}
	\nabla_{\theta^{\pi_i}} \mathcal{J}(\theta^{\pi_i})
	= \mathbb{E}\big[\nabla_{\theta^{\pi_i}}\log \pi_{i}(a_i \mid e_{i}, \widetilde{m}_{i})\,
	Q^{\mathrm{full}}_{i}(\boldsymbol{o}, \boldsymbol{a}, \widetilde{m}_{i})\big].
\end{equation}

We backpropagate through the observation encoder jointly with the actor:
\begin{equation}\label{oeloss}
	\nabla_{\theta^{e_i}} \mathcal{J}(\theta^{e_i})
	= \mathbb{E}\Big[
	\nabla_{\theta^{e_i}} e_{i}\!(o_{i})
	\,\nabla_{e_{i}} \log \pi_{i}(a_i \mid e_{i}, \widetilde{m}_{i})
	\,Q^{\mathrm{full}}_{i}(\boldsymbol{o}, \boldsymbol{a}, \widetilde{m}_{i})
	\Big].
\end{equation}

The CAMA projections $(W_i^Q,W_i^K,W_i^V)$ are optimized jointly with the actor by backpropagating its objective through the aggregated message representation $\widetilde m_i$.

\paragraph{OTG objectives.}
OTG is trained with supervised objectives. The observation-dynamics predictor $f_i^{o}$ regresses the one-step observation difference $(o_i' - o_i)$ conditioned on the local embedding $e_i$, the available incoming messages $\{m^{r}_{ij}\}_{j\in \mathcal{Z}_i}$, and the OTG action prediction $\hat{\boldsymbol a}^{\,g}_{i}$:
\begin{equation}\label{otgloss}
	\mathcal{L}(\theta^{\zeta_i})
	= \mathbb{E}\Big[\big((o_{i}'-o_{i})-f^{o}_{i}\big(e_{i}, \{m^{r}_{ij}\}_{j\in \mathcal{Z}_{i}}, \hat{ a}^{\,g}_{i}\big)\big)^{2}\Big],
\end{equation}
where $\mathcal{Z}_i$ denotes the available-sender set in the sampled transition. In addition, the action predictor $f_i^{a}$ predicts teammates' discrete actions from the same inputs and is trained by the cross-entropy loss
\begin{equation}\label{otgloss2}
	\mathcal{L}(\theta^{\psi_i})
	= \mathbb{E}\Big[\sum_{j\in \mathcal{Z}_i}
	\big(- a_j^{\top}\log \hat a_j\big)\Big].
\end{equation}

\paragraph{Pairwise CGDC target construction.}
\label{app:pairwise-targets}
For every ordered pair $(i,j)$ with $j\ne i$, centralized training constructs a delayed-message reception condition while holding the replay observations and joint action fixed.
Within this target calculation, $m^s_{ji,k}$ denotes the sent message and $m^r_{ij,t}$ its received representation under the constructed condition, with $t=k+d_{ji,k}$.
We evaluate both $\widehat V^g_{ij,t}$ and $V^c_{ij,t}$ using this same $m^r_{ij,t}$ and the same base critic inputs; the timely reference is $m^{\mathrm{tar}}_{ji,t}=e_{j,t}$.
The absence of an actual arrival does not set either term to zero.
This construction supplies CGDC labels and does not replace the recorded messages used by DCOS, CAMA, the actor, or the full critic, or change the actual available-sender set $\mathcal Z_{i,t}$.

\paragraph{DCOS objective.}
DCOS predicts these pairwise targets from its local observation and actual received-message context by minimizing
\begin{equation}\label{osloss}
	\mathcal{L}(\theta^{\phi_i})
	= \mathbb{E}\Big[\sum_{j\in\mathcal{N}\setminus\{i\}}
	(c_{ij}-\hat c_{ij})^{2}\Big].
\end{equation}
The targets use the minibatch-wise gain--cost weight defined below.

\paragraph{Weight scaling and numerical boundary cases.}
\label{app:weight-scaling}
For a training minibatch $\mathcal B$, let $\mathcal I_{\mathcal B}=\{(i,j,t)\in\mathcal B\mid j\in\mathcal Z_{i,t}\}$ contain the tuples with an available received message. When $\mathcal I_{\mathcal B}\ne\varnothing$, define
\begin{equation}
 \bar g_+:=\frac{1}{|\mathcal I_{\mathcal B}|}\sum_{(i,j,t)\in\mathcal I_{\mathcal B}}[\widehat V^g_{ij,t}]_+,
 \qquad
 \bar v_c:=\frac{1}{|\mathcal I_{\mathcal B}|}\sum_{(i,j,t)\in\mathcal I_{\mathcal B}}V^c_{ij,t}.
\end{equation}
With numerical threshold $\epsilon_c=10^{-8}$, set
\begin{equation}\label{eq:stable-cgdc-weight}
 \lambda_0:=
 \begin{cases}
 0,&\mathcal I_{\mathcal B}=\varnothing,\\
 0,&\mathcal I_{\mathcal B}\ne\varnothing\ \text{and}\ \bar v_c\le\epsilon_c,\\
 \operatorname{stopgrad}(\bar g_+/\bar v_c),&\text{otherwise},
 \end{cases}
 \qquad \lambda=\kappa\lambda_0,\quad\kappa\ge0.
\end{equation}
The default multiplier is $\kappa=1$. The positive part is applied to each gain only for scale estimation; the target retains the signed gain, $c_{ij,t}=\widehat V^g_{ij,t}-\lambda V^c_{ij,t}$. The arrival-based subset above is used only for scale estimation; CGDC labels are constructed for all non-self pairs. The low-cost branch avoids amplifying negligible KL values, and $\operatorname{stopgrad}$ detaches the weight from differentiation. DCOS learns these targets as the critics and minibatch scales evolve. At zero delivery delay, $H=0$; the same weight rule applies to the computed surrogate cost, which is zero only when the timely and received conditions induce identical probe distributions.

\subsection{Time Complexity Analysis}
\label{app:time-complexity}

We distinguish decentralized-execution complexity from CGDC target construction during centralized training. Action-space size, per-agent representation dimensions, and network widths are held fixed; dependence on the number of input agents and messages is retained explicitly.

Let $\mathcal{G}_{i,t}=\{j\neq i\mid\hat c_{ij,t}>0\}$ be the senders requested by receiver $i$, $\mathcal{Z}_{i,t}$ the senders whose messages are available to $i$, and $\mathcal{S}_t=\bigcup_i\mathcal{G}_{i,t}$ the requested-sender set. Define $b_q:=\max_i|\mathcal{G}_{i,t}|$, $b_r:=\max_i|\mathcal{Z}_{i,t}|$, and $n_s:=|\mathcal{S}_t|\le N$.

\paragraph{Decentralized-execution complexity.}
DCOS outputs $N(N-1)$ scores per timestep. Let $C_{\mathrm{D}}(N,b_r)$ denote its per-receiver forward cost and $C_{\mathrm{O}}(b_r)$ the cost of one OTG prediction step, including incoming-message processing and teammate-action prediction. Each requested sender constructs one message for all its receivers, giving $\mathcal{O}(n_s[1+H C_{\mathrm{O}}(b_r)])$ construction cost. Link delivery and CAMA aggregation cost $\mathcal{O}(N(b_q+b_r))$. The total per-timestep cost is therefore
\[
\mathcal{O}\!\left(N C_{\mathrm{D}}(N,b_r)+n_s[1+H C_{\mathrm{O}}(b_r)]+N(b_q+b_r)\right).
\]
For linear-cost scoring and message processing, $C_{\mathrm{D}}=\mathcal{O}(N+b_r)$ and $C_{\mathrm{O}}=\mathcal{O}(1+b_r)$, yielding $\mathcal{O}(N^2+n_sH(1+b_r))$. This becomes $\mathcal{O}(N^2+NH)$ when $b_r$ is bounded independently of $N$, and $\mathcal{O}(N^2(1+H))$ for dense arrivals.

\paragraph{Additional training-time complexity.}
The gain and delay-cost surrogates are not evaluated during decentralized execution. Given the constructed pairwise message conditions, evaluating the CGDC targets for $B$ sampled transitions requires $\mathcal{O}(BN^2)$ gain-critic evaluations at fixed action-space size. If one evaluation costs $C_Q(N)$, this critic-evaluation cost is $\mathcal{O}(BN^2 C_Q(N))$, excluding message-condition construction. The dependence on $N$ accounts for the other agents' observations and actions in the critic input; for example, direct concatenation with a fixed-width first layer gives $C_Q(N)=\mathcal{O}(N)$. Target construction is performed once after each training episode. For an episode minibatch, $B$ counts all nonpadding transitions in the sampled episodes.

\paragraph{Dependence on team size.}
Under the linear-processing assumptions above and fixed $H$, execution has quadratic worst-case dependence on $N$. Sender-level message reuse avoids repeating OTG rollouts for each requested link, while sparse arrivals reduce the message-processing cost within each rollout.

\section{Additional Experimental Details and Results}
\label{app:exp}

\subsection{Implementation Details}
\label{app:exp-details}

\paragraph{Environment details.}
For Cooperative Navigation and Predator Prey in MPE, we remove teammate vision to strengthen partial observability. Each episode lasts 60 timesteps, and each agent chooses from the discrete action set \{up, down, left, right, stop\}.

\textbf{Cooperative Navigation.} Eight agents and eight landmarks are randomly placed. The agents aim to occupy as many landmarks as possible while minimizing collisions. Each agent observes only the positions of its three nearest landmarks and has no access to teammate information, requiring agents to infer teammates' intentions when choosing landmarks.

\textbf{Predator Prey.} Eight predators cooperate to capture four faster-moving prey. Each predator observes the relative positions and velocities of its three nearest prey but has no access to teammate information. The prey follow a fixed strategy of moving away from the nearest predator. Predators have size 0.04 and acceleration 5.0, while prey have size 0.05 and acceleration 7.0. The prey's speed advantage makes individual capture difficult and motivates coordinated pursuit.

For SMAC, we use SC2.4.10 with difficulty 7, and each episode lasts 50 timesteps.

\paragraph{Evaluation protocol.}
Each configuration is trained for 200,000 episodes with four independent random seeds and evaluated every 500 episodes over 32 evaluation episodes.
Performance tables use seed-level evaluation means from the final checkpoint; zero-shot tests reuse the \texttt{easy}-trained checkpoint without further updates.
For SMAC performance in Tables~\ref{results_table} and~\ref{tab:mmm2_main}, we report the mean and sample standard deviation ($n-1$ denominator, $n=4$) of the unsmoothed final-checkpoint win rates, rounded to two decimal places in percent.
Learning curves first apply a seven-point moving average to each seed's evaluation sequence; their means and shaded standard deviations are then computed across the smoothed seed curves.

\paragraph{CDCMA optimization schedule.}
For CDCMA and its ablations, training starts without a replay-buffer warm-up.
When fewer than 32 episodes are available, we sample with replacement to form a 32-episode minibatch.
We perform one gradient step after each training episode and update the target networks every 200 gradient steps.

\paragraph{Delayed-message handling protocol.}
All methods share the same delay distributions and message-delivery rules. Agents may use only messages that arrive before action selection. When multiple messages from the same sender arrive simultaneously at a receiver, we retain only the most recent one and discard older arrivals from that sender at the same timestep.
In CDCMA, a receiver's request is modeled as an instantaneous control signal that is not subject to the delayed content channel. Once a request is issued, the sender constructs the outgoing message at the same timestep. The request-time CGDC score is stored in a separate scalar tensor aligned with the corresponding message; the message and score undergo the same sampled delivery delay and are used together by CAMA upon arrival.
Each timestep contains one message-construction round. DCOS and OTG use the communication inputs after previously sent messages have arrived and before new messages are sent. The channel delivers newly generated zero-delay messages before CAMA aggregation and action selection; these arrivals do not trigger another construction round. A current zero-delay message replaces any older arrival from the same sender together with its score. For actual-message diagnostics, timely references are stored for the final available-sender set used for action selection, with empty references for silent pairs.

\paragraph{Baseline implementations.}
We retain the algorithm-specific hyperparameters of the official implementations of DACOM, \href{https://github.com/ZhuohuiZhang/TGCNet}{TGCNet}, \href{https://github.com/ZangZehua/T2MAC}{T2MAC}, \href{https://github.com/DiXue98/SMS}{SMS}, G2ANet, and ATOC, without additional tuning.
All methods follow the common task settings, training budget, evaluation protocol, and delayed-message interface described above.
Hyperparameters for CDCMA and its ablations are listed in Table~\ref{app:hyperparams}.

\paragraph{CoDe reproduction.}
We reproduce CoDe~\citep{song2025code} with a QMIX backbone and Adam at a learning rate of $10^{-3}$.
The agent GRU and mixer have hidden dimensions of 64 and 32, respectively; the two-layer hypernetworks have a hidden dimension of 64.
Its timeliness alignment uses the framework-provided message age $\Delta t$, which is not an input to CDCMA.
The delay interface changes delivery timing while preserving CoDe's message cache and timeliness alignment.

\paragraph{DACOM adaptation.}
We omit DACOM's TimeNet module for within-step waiting and evaluate the remaining architecture under the common cross-timestep delivery protocol.
All reported DACOM results use this variant.

\paragraph{Communication payload.}
CDCMA uses 64-dimensional messages, with one additional scalar request-time CGDC score stored separately for each selected message.
OTG changes message construction but not dimensionality.
The official-code baselines retain their native message representations.

\paragraph{CDCMA-specific defaults.}
Unless otherwise stated, we use $H=d_{\max}=5$ under delayed communication and $H=0$ in delay-free evaluations.
We use $\eta=1$ for the critic-tempered policy and the standard scaled dot-product factor $\beta=1/\sqrt{d_k}$ for CAMA attention, where $d_k$ is the key dimension. Neither value is tuned separately.

\begin{table}[!t]
	\centering
	\small
	\caption{Hyperparameters of CDCMA and its ablations.}
	\label{app:hyperparams}
	\setlength{\tabcolsep}{4pt}
	\renewcommand{\arraystretch}{1.05}
	\begin{tabular*}{\linewidth}{@{\extracolsep{\fill}}lcc@{}}
		\hline
		Hyperparameter & MPE & SMAC \\
		\hline
		Discount factor & 0.96 & 0.99 \\
		Actor learning rate & $10^{-3}$ & $10^{-4}$ \\
		Critic learning rate & $10^{-2}$ & $10^{-3}$ \\
		Actor hidden dimension & 64 & 64 \\
		Critic hidden dimension & 128 & 128 \\
		Message dimension & 64 & 64 \\
			Batch size (episodes) & 32 & 32 \\
			Replay-buffer capacity (episodes) & $5\times 10^3$ & $5\times 10^3$ \\
		Rollout workers & 1 & 1 \\
		Activation function & ReLU & ReLU \\
		Optimizer & Adam & Adam \\
		\hline
	\end{tabular*}
	
	\vspace{2mm}
	\begin{minipage}{\linewidth}
		\footnotesize
		The ablations are CDCMA-DCOS, CDCMA-CAMA, CDCMA-OTG, and CDCMA-PlainAttn.
		MPE includes Cooperative Navigation and Predator Prey.
	\end{minipage}
\end{table}

\subsection{Delay Model}
\label{app:exp-delay-model}

\begin{figure}[!t]
	\centering
	\subfloat[Delay parameters $(\mu_{\ell},\sigma_{\ell})$.]{%
		\begin{minipage}[t]{0.45\linewidth}\centering\footnotesize
			\begin{tabular}{@{}lcc@{}}
				\toprule
				\textbf{Difficulty $\ell$} & $\mu_{\ell}$ & $\sigma_{\ell}$ \\
				\midrule
				\texttt{easy}         & 1.00 & 1.00 \\
				\texttt{medium}       & 2.00 & 0.80 \\
				\texttt{hard}         & 3.00 & 0.70 \\
				\texttt{super\_hard}  & 4.00 & 0.70 \\
				\bottomrule
			\end{tabular}
		\end{minipage}
	}
	\subfloat[Illustration of the four delay distributions.\label{fig:delay_dist}]{
		\begin{minipage}{.2\linewidth}\centering
			\includegraphics[width=\linewidth]{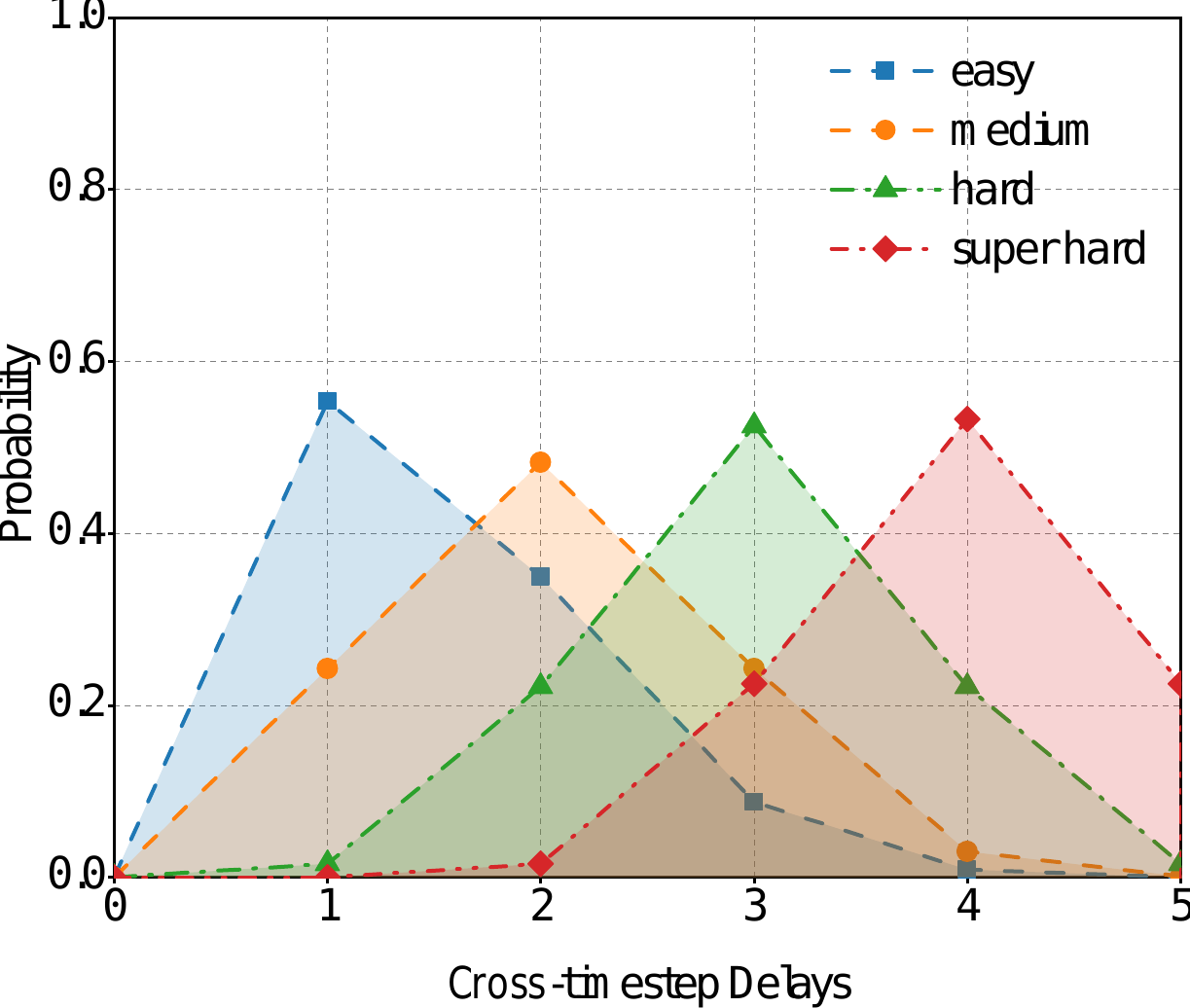}
		\end{minipage}
	}
	\caption{Delay model: parameter table and corresponding distribution shapes.}
	\label{fig:delay_table_and_plot}
\end{figure}

For each ordered pair $(j,i)$ and each send timestep $k$, the environment samples an integer delay $d_{ji,k}\in\{1,2,\dots,d_{\max}\}$ with $d_{\max}=5$ in the delayed settings. The delay-free setting corresponds to $d_{ji,k}=0$ for all ordered pairs. We use a difficulty-indexed probability mass function $p_{\ell}(d)$ on $\{1,\dots,d_{\max}\}$ with $\ell\in\{\texttt{easy}, \texttt{medium}, \texttt{hard}, \texttt{super\_hard}\}$. Our default specification is a truncated discrete normal:
\begin{equation}\label{eq:truncnorm_delay}
	p_{\ell}(d)=
	\frac{
		\Phi\!\big(\tfrac{d+0.5-\mu_{\ell}}{\sigma_{\ell}}\big)\;-\;
		\Phi\!\big(\tfrac{d-0.5-\mu_{\ell}}{\sigma_{\ell}}\big)
	}{
		\sum_{u=1}^{d_{\max}}
		\!\Big[
		\Phi\!\big(\tfrac{u+0.5-\mu_{\ell}}{\sigma_{\ell}}\big)\;-\;
		\Phi\!\big(\tfrac{u-0.5-\mu_{\ell}}{\sigma_{\ell}}\big)
		\Big]
	},
\end{equation}
where $\Phi(\cdot)$ is the standard normal cumulative distribution function. The parameters $(\mu_{\ell},\sigma_{\ell})$ are listed in Fig.~\ref{fig:delay_table_and_plot}(a), and the resulting PMF shapes are shown in Fig.~\ref{fig:delay_table_and_plot}(b). As the difficulty increases, $\mu_{\ell}$ shifts probability mass toward longer delays, while $\sigma_{\ell}$ controls the spread.

At the beginning of each timestep $k$, for each ordered pair $(j,i)$, the environment samples an exogenous per-link delay $d_{ji,k}\sim p_{\ell}(\cdot)$ independently across ordered pairs and timesteps. Any message generated at timestep $k$ from $j$ to $i$ is delivered after $d_{ji,k}$ timesteps. Message arrival and the rule for concurrent arrivals follow the DeComm-POMG specification in Sec.~\ref{sec:pro_formu}: the arrival timestep is $t=k+d_{ji,k}$, and if several messages from the same sender arrive simultaneously, only the most recent one is retained.

\subsection{Additional Learning Curves}
\label{app:exp-curves}

\begin{figure*}[tb]
	\centering
	\includegraphics[width=0.55\linewidth]{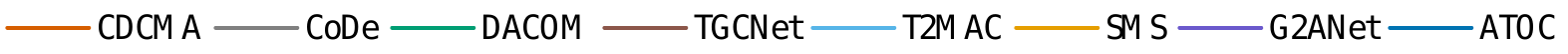}\\
	
	\subfloat[1o2r\_\texttt{easy}]{\includegraphics[width=0.225\linewidth]{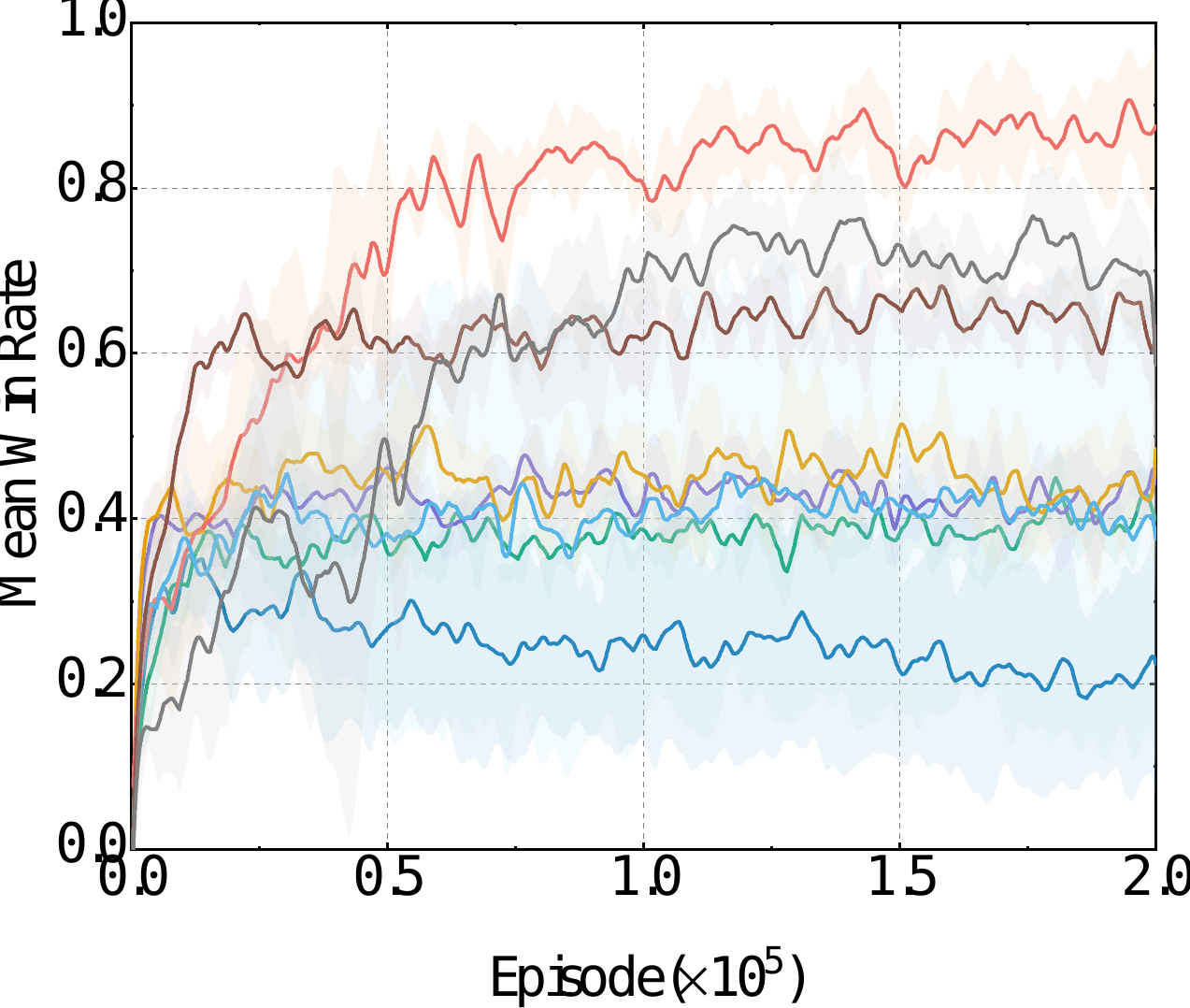}\label{1o2r_easy}}
	\hfill
	\subfloat[1o2r\_\texttt{medium}]{\includegraphics[width=0.225\linewidth]{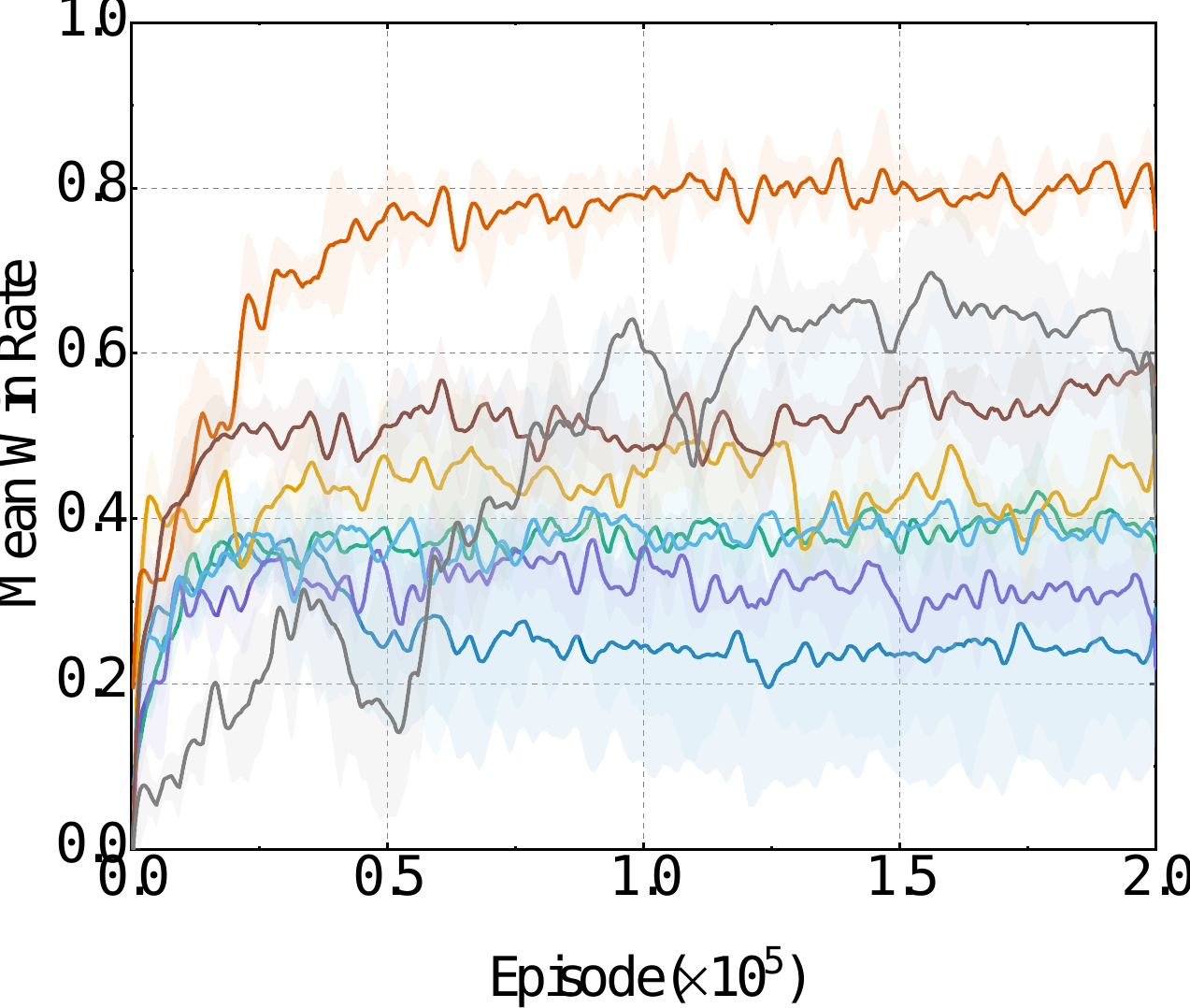}\label{1o2r_medium}}
	\hfill
	\subfloat[1o2r\_\texttt{hard}]{\includegraphics[width=0.225\linewidth]{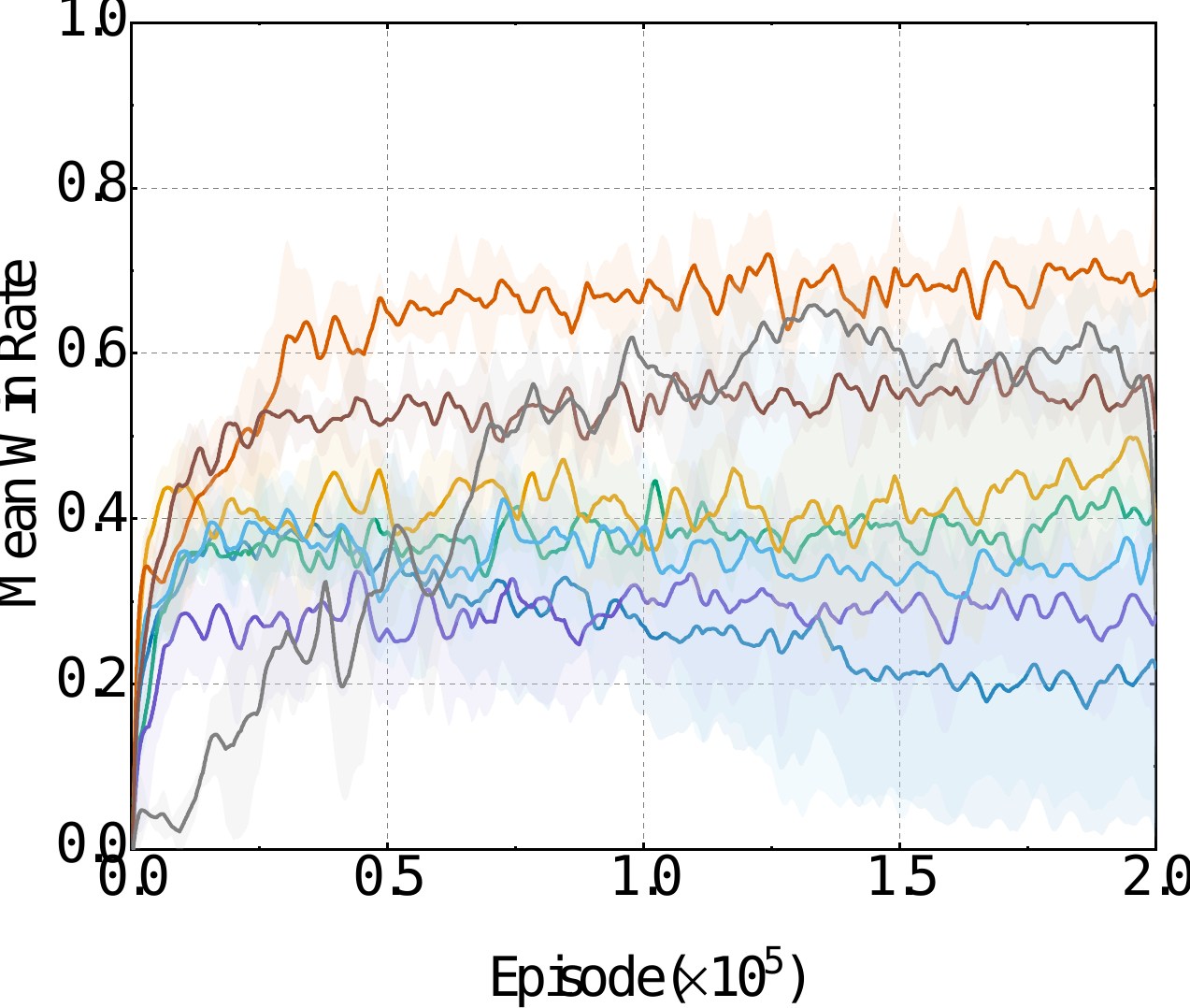}\label{1o2r_hard}}
	\hfill
	\subfloat[1o2r\_\texttt{super\_hard}]{\includegraphics[width=0.225\linewidth]{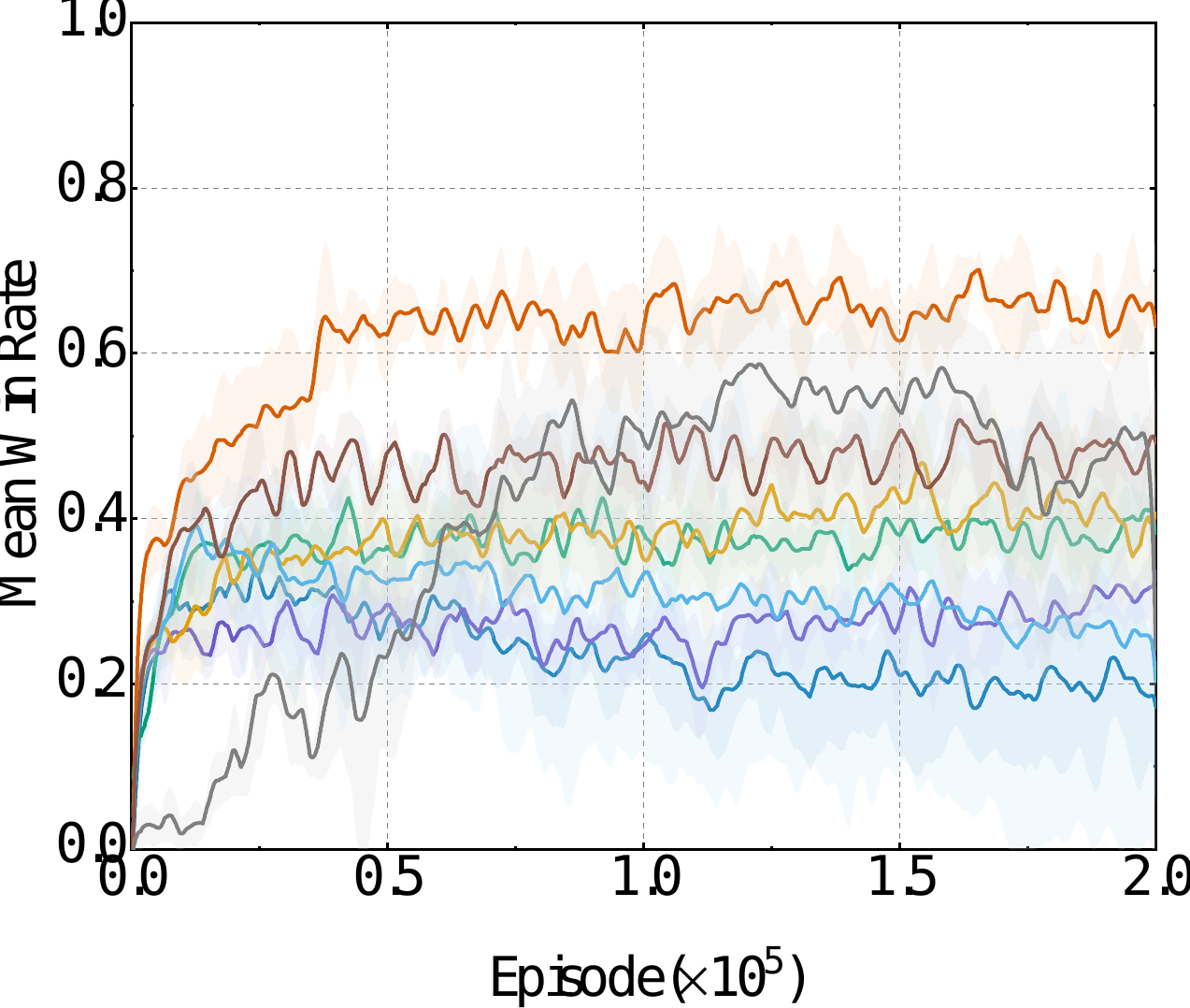}\label{1o2r_super_hard}}
	\\
	
	\subfloat[1o10b\_\texttt{easy}]{\includegraphics[width=0.225\linewidth]{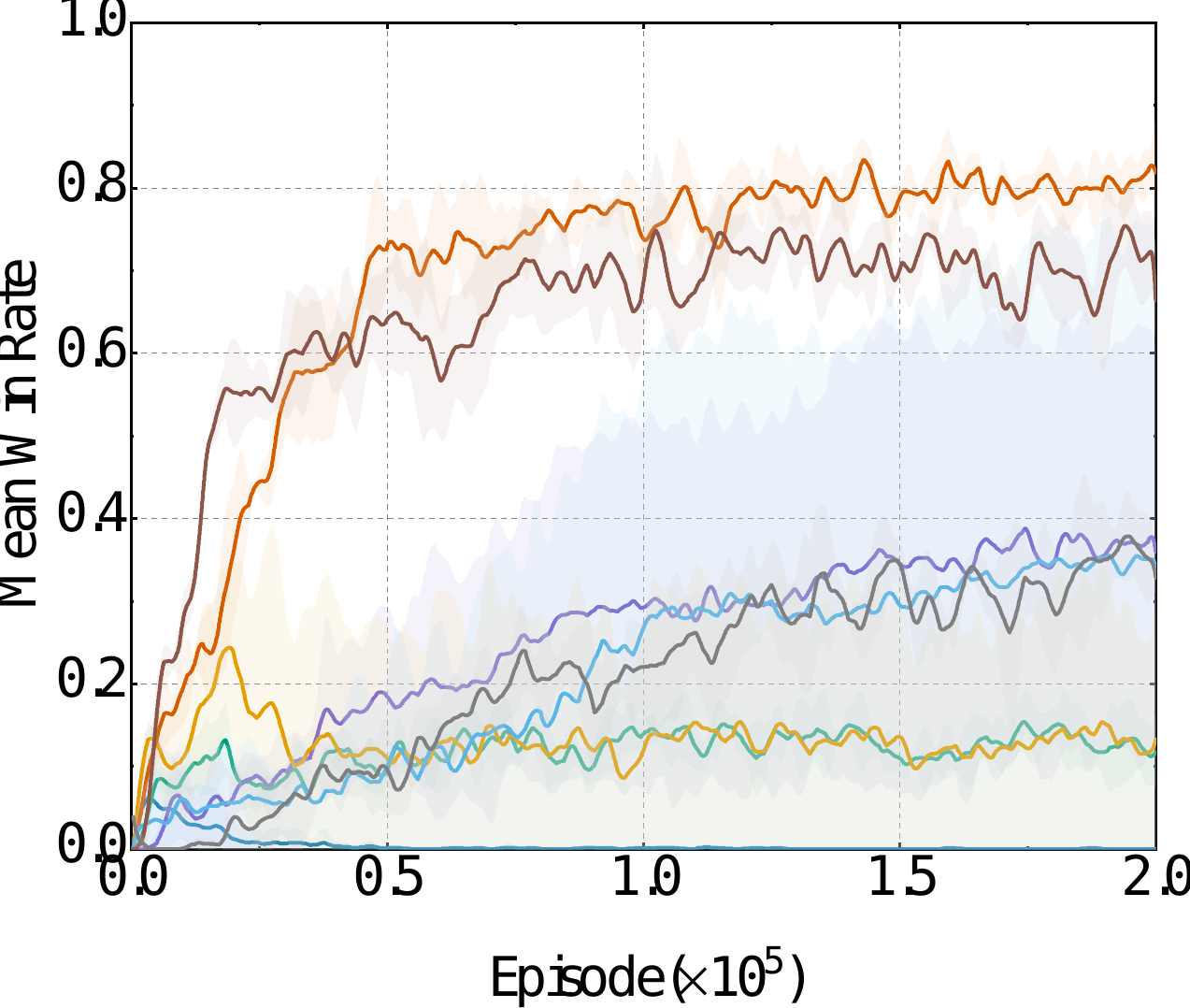}\label{1o10b_easy}}
	\hfill
	\subfloat[1o10b\_\texttt{medium}]{\includegraphics[width=0.225\linewidth]{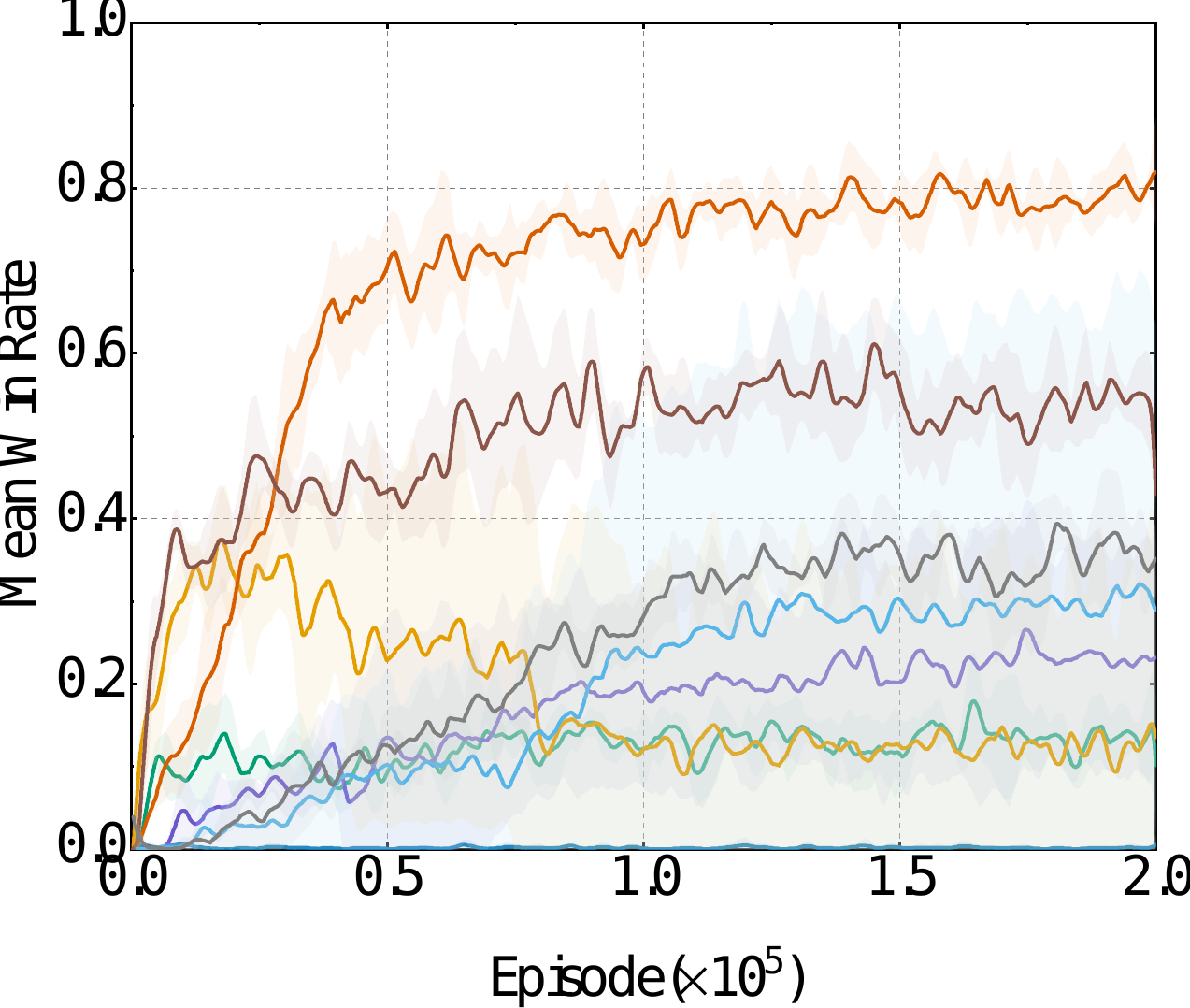}\label{1o10b_medium}}
	\hfill
	\subfloat[1o10b\_\texttt{hard}]{\includegraphics[width=0.225\linewidth]{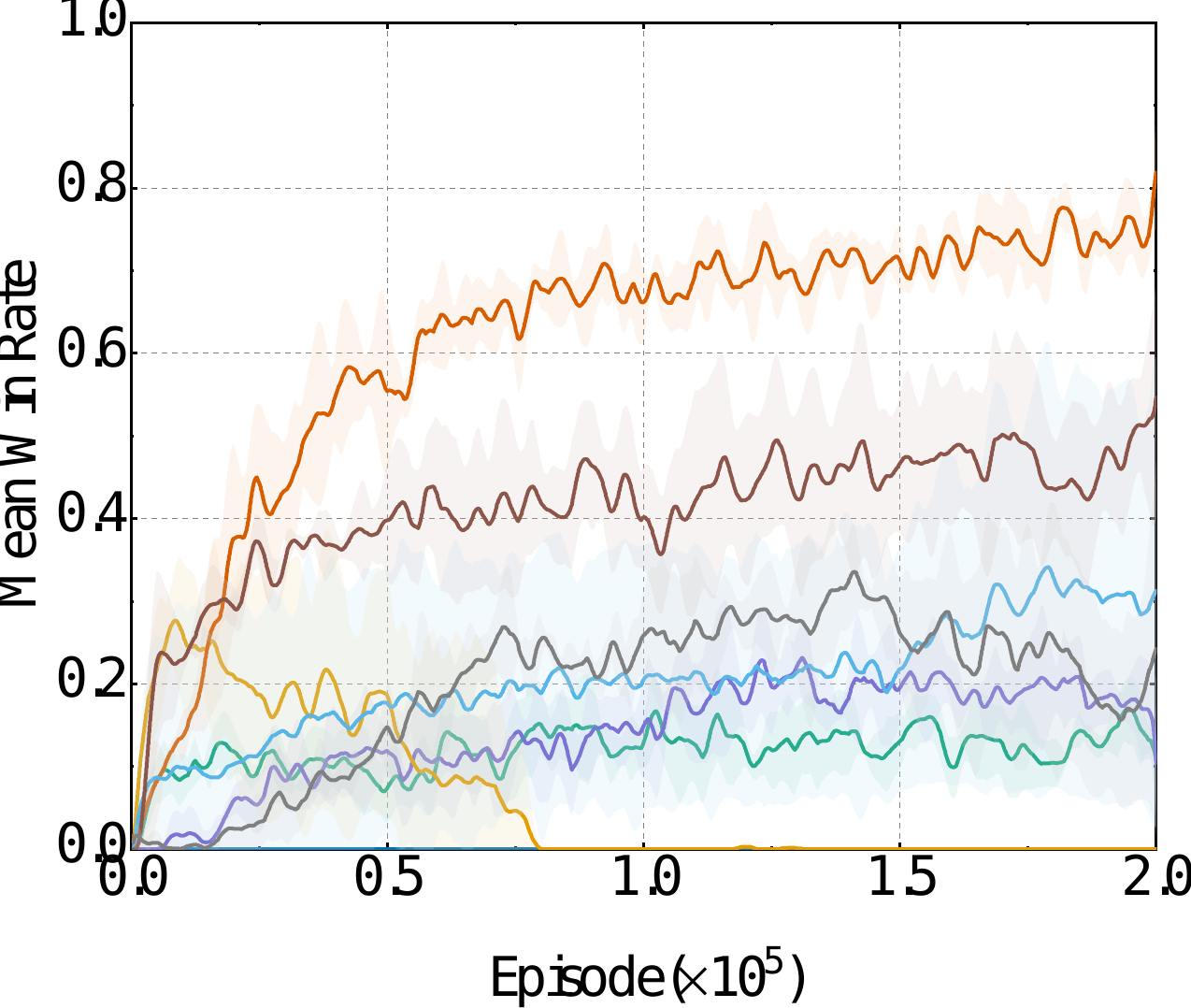}\label{1o10b_hard}}
	\hfill
	\subfloat[1o10b\_\texttt{super\_hard}]{\includegraphics[width=0.225\linewidth]{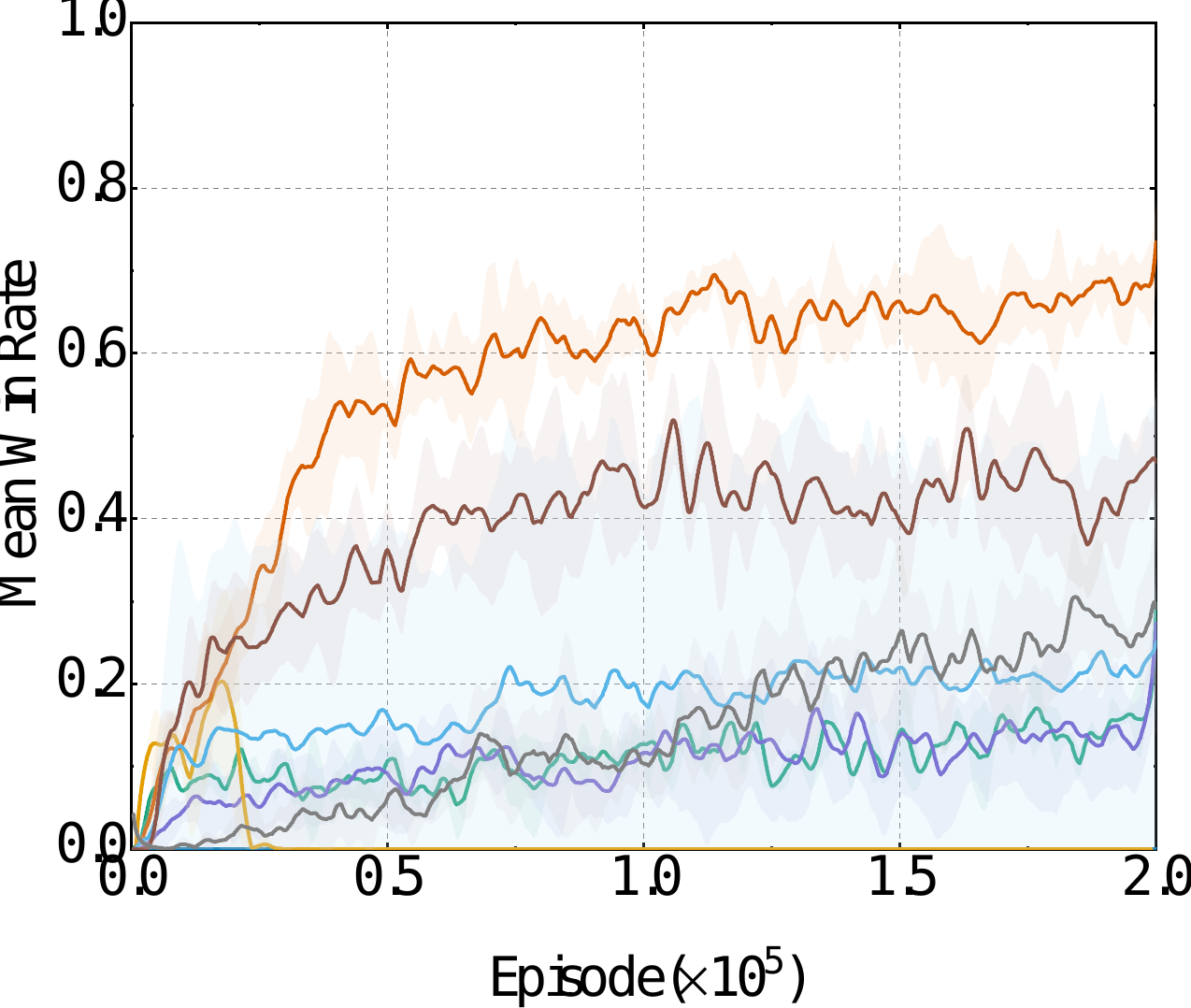}\label{1o10b_super_hard}}
	
	\caption{Mean win rate vs. training episodes across two SMAC tasks and four delay difficulties.}
	\label{smac_result_fig}
\end{figure*}

\begin{figure*}[tb]
	\centering
	\includegraphics[width=0.55\linewidth]{fig/caption.pdf}\\
	
	\subfloat[CN \texttt{easy}]{\includegraphics[width=0.24\linewidth]{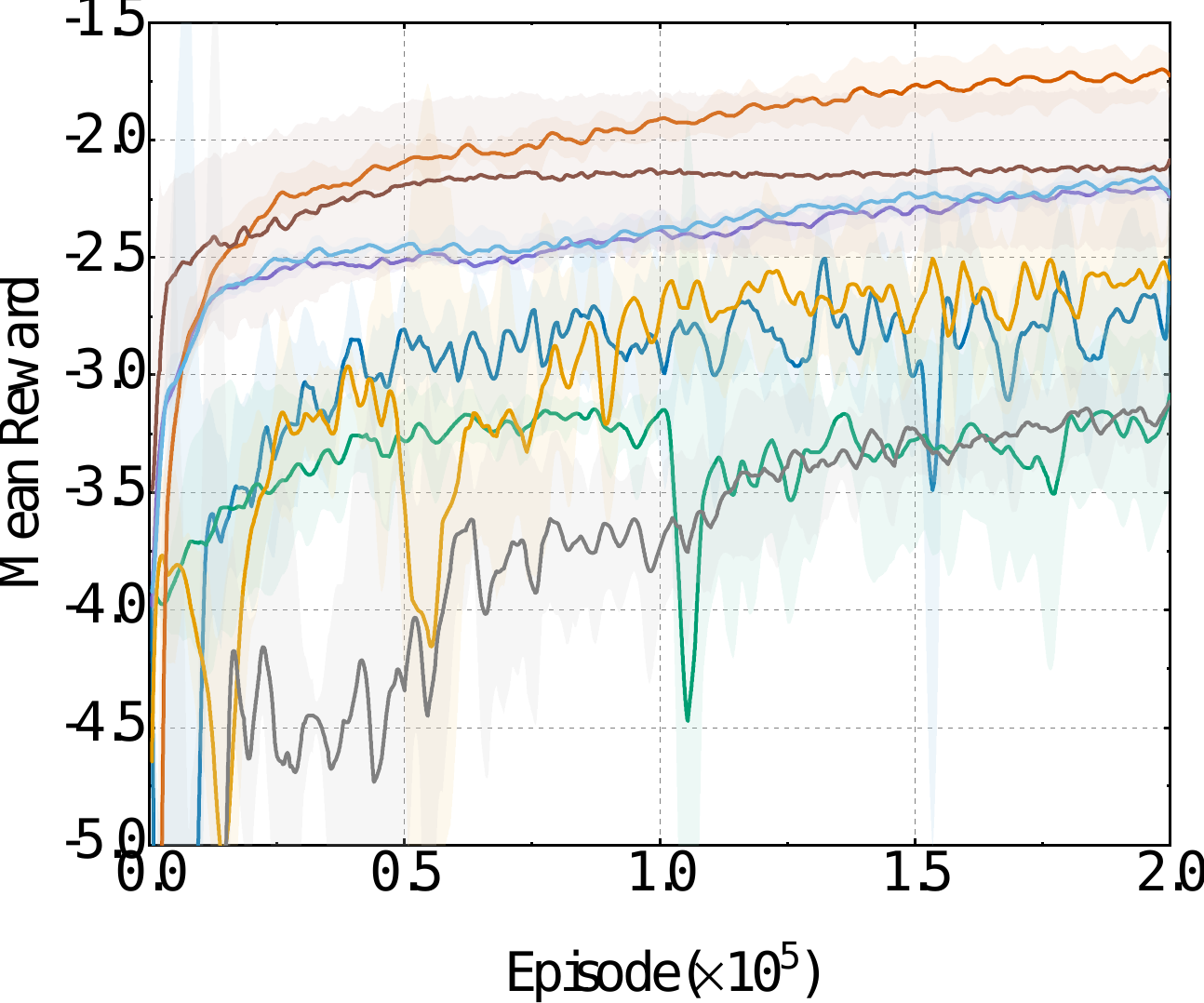}\label{fig:cn_easy}}
	\hfill
	\subfloat[CN \texttt{medium}]{\includegraphics[width=0.24\linewidth]{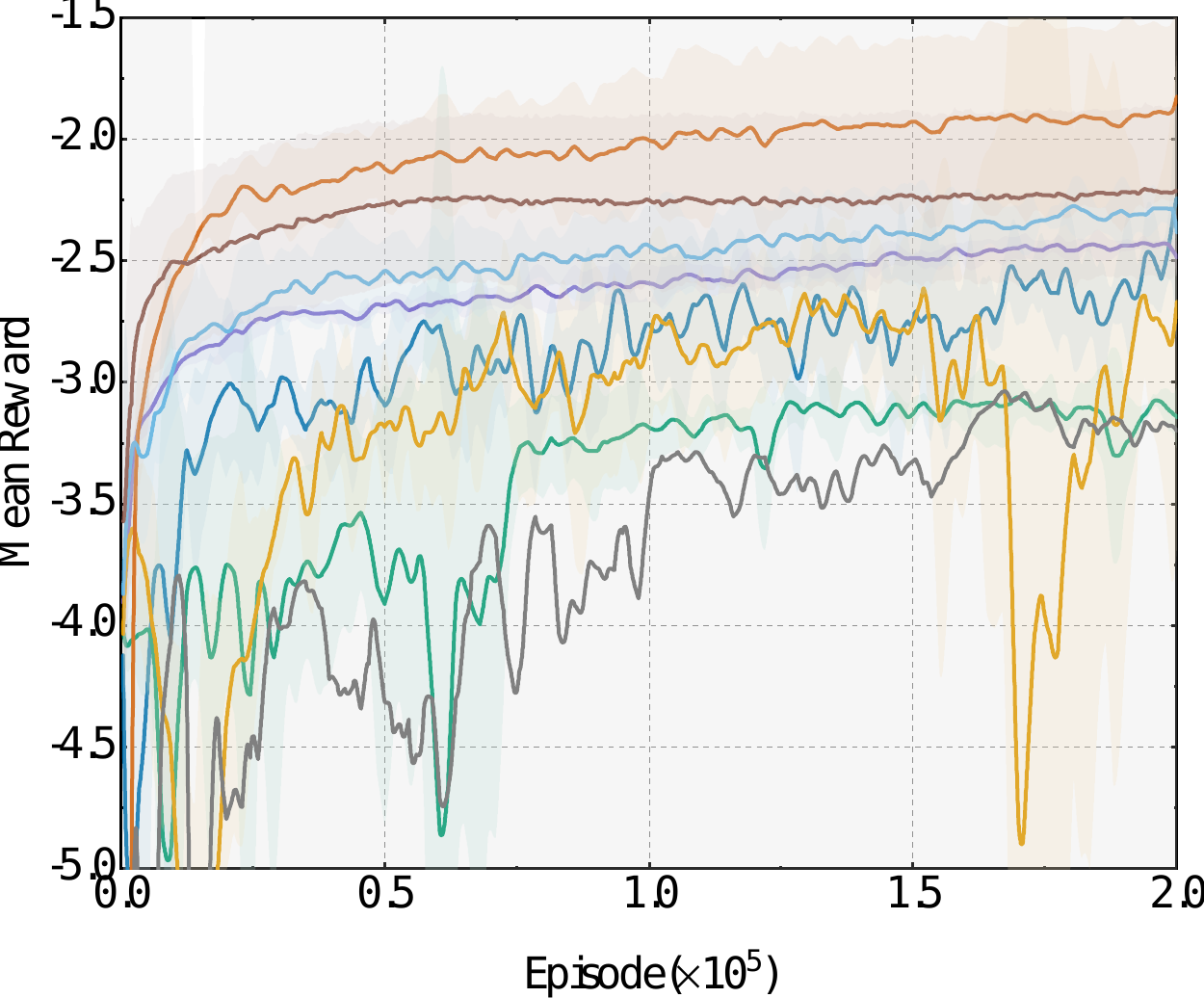}\label{fig:cn_medium}}
	\hfill
	\subfloat[CN \texttt{hard}]{\includegraphics[width=0.24\linewidth]{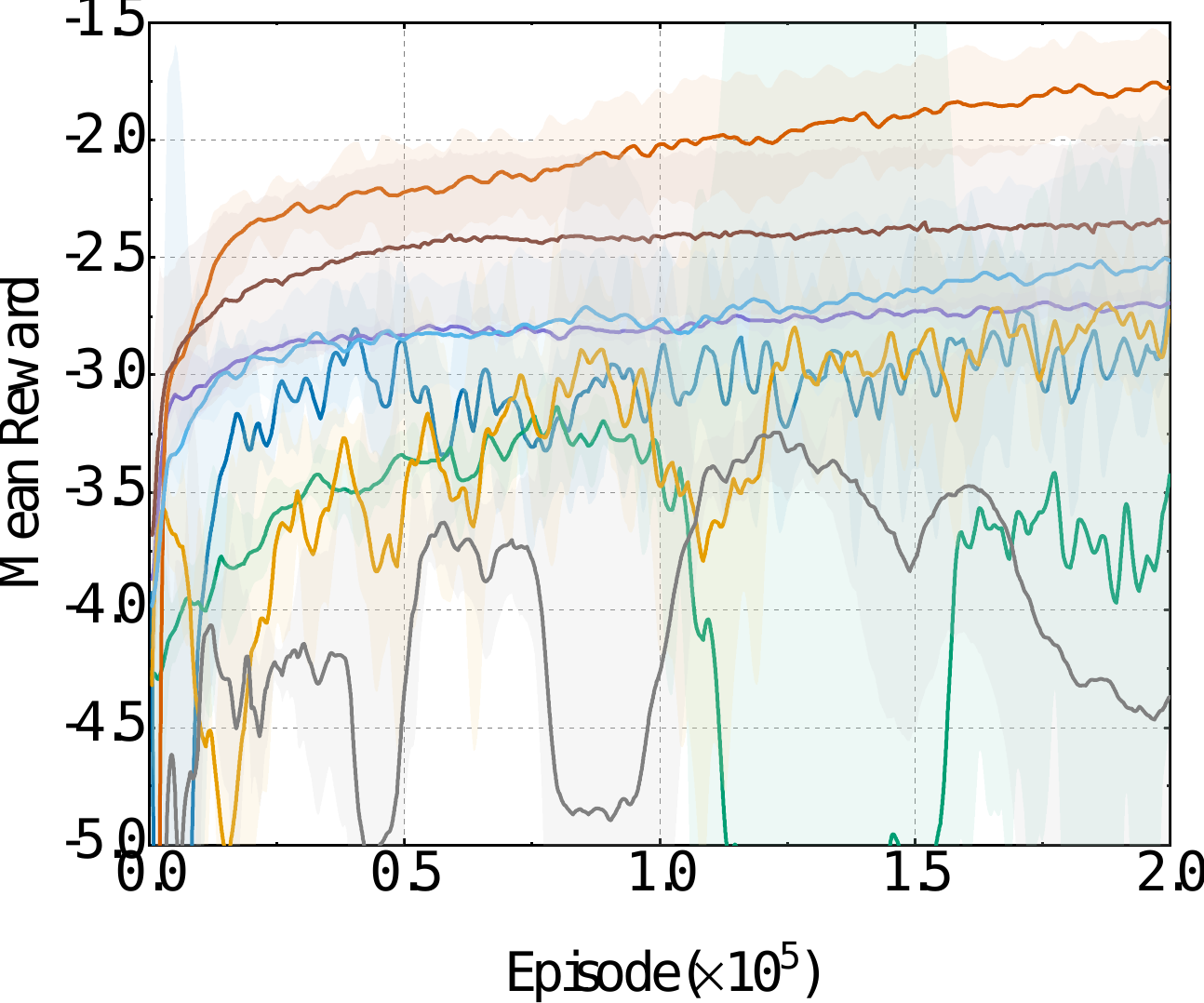}\label{fig:cn_hard}}
	\hfill
	\subfloat[CN \texttt{super\_hard}]{\includegraphics[width=0.24\linewidth]{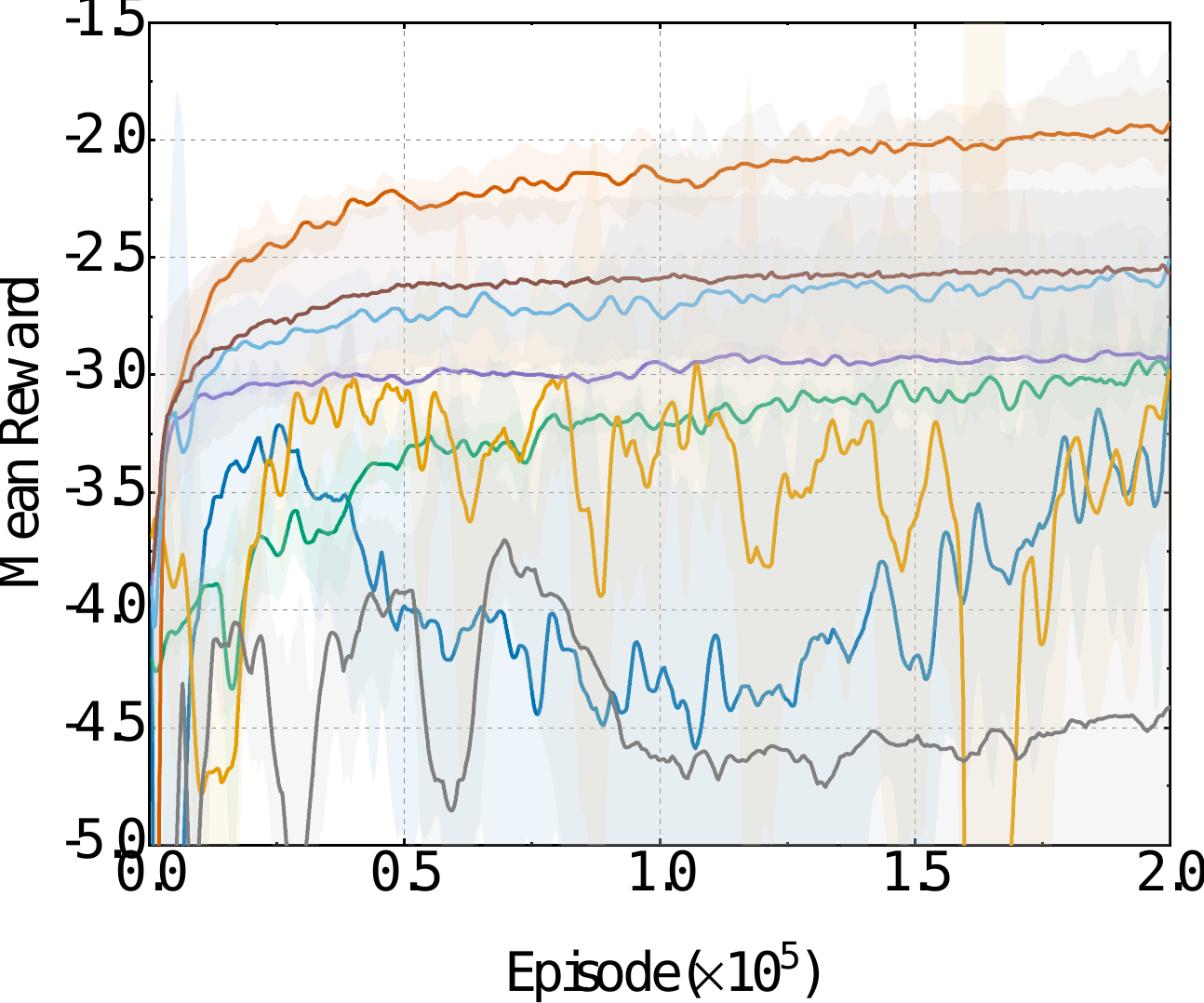}\label{fig:cn_super_hard}}
	\\
	\subfloat[PP \texttt{easy}]{\includegraphics[width=0.24\linewidth]{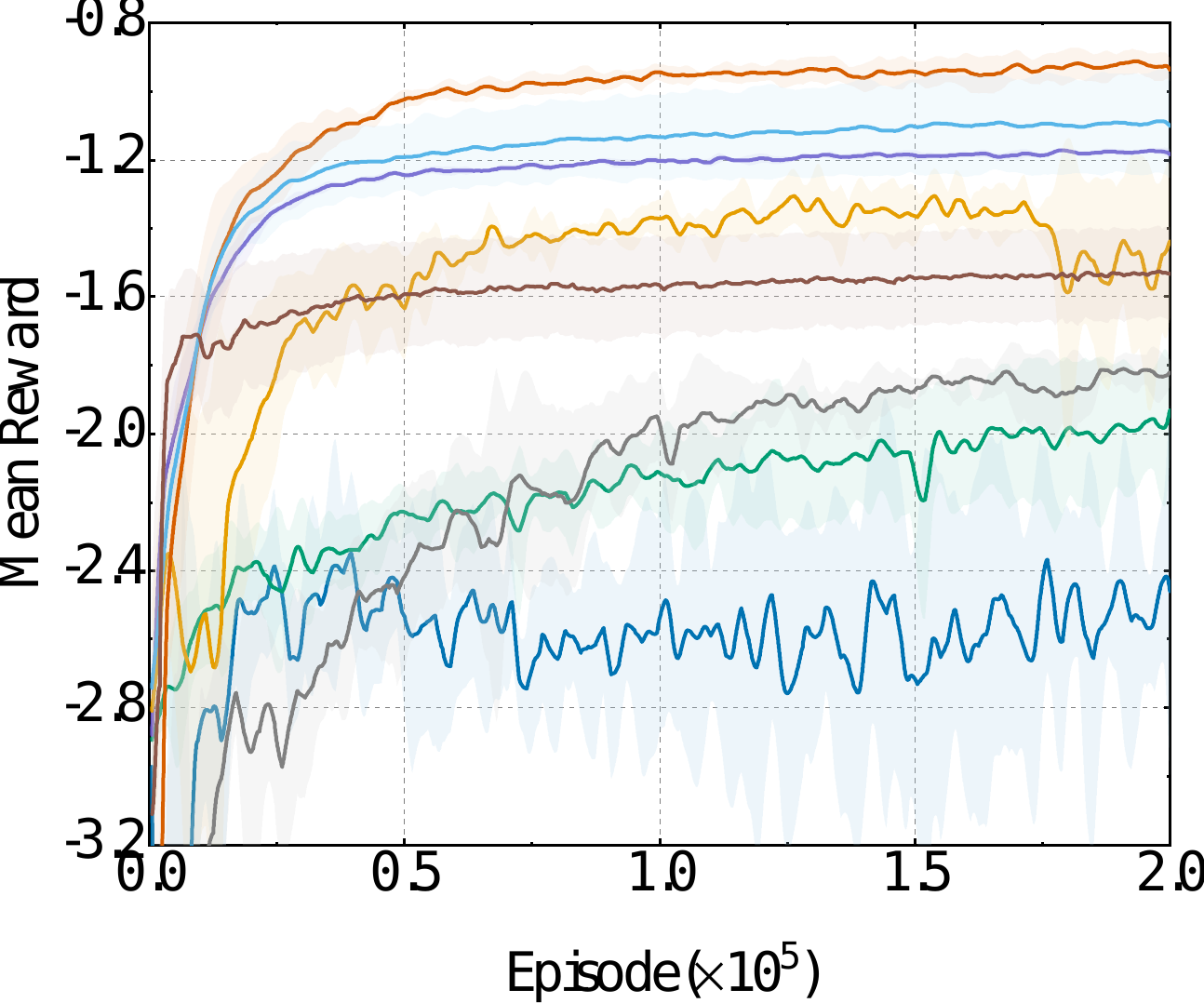}\label{fig:pp_easy}}
	\hfill
	\subfloat[PP \texttt{medium}]{\includegraphics[width=0.24\linewidth]{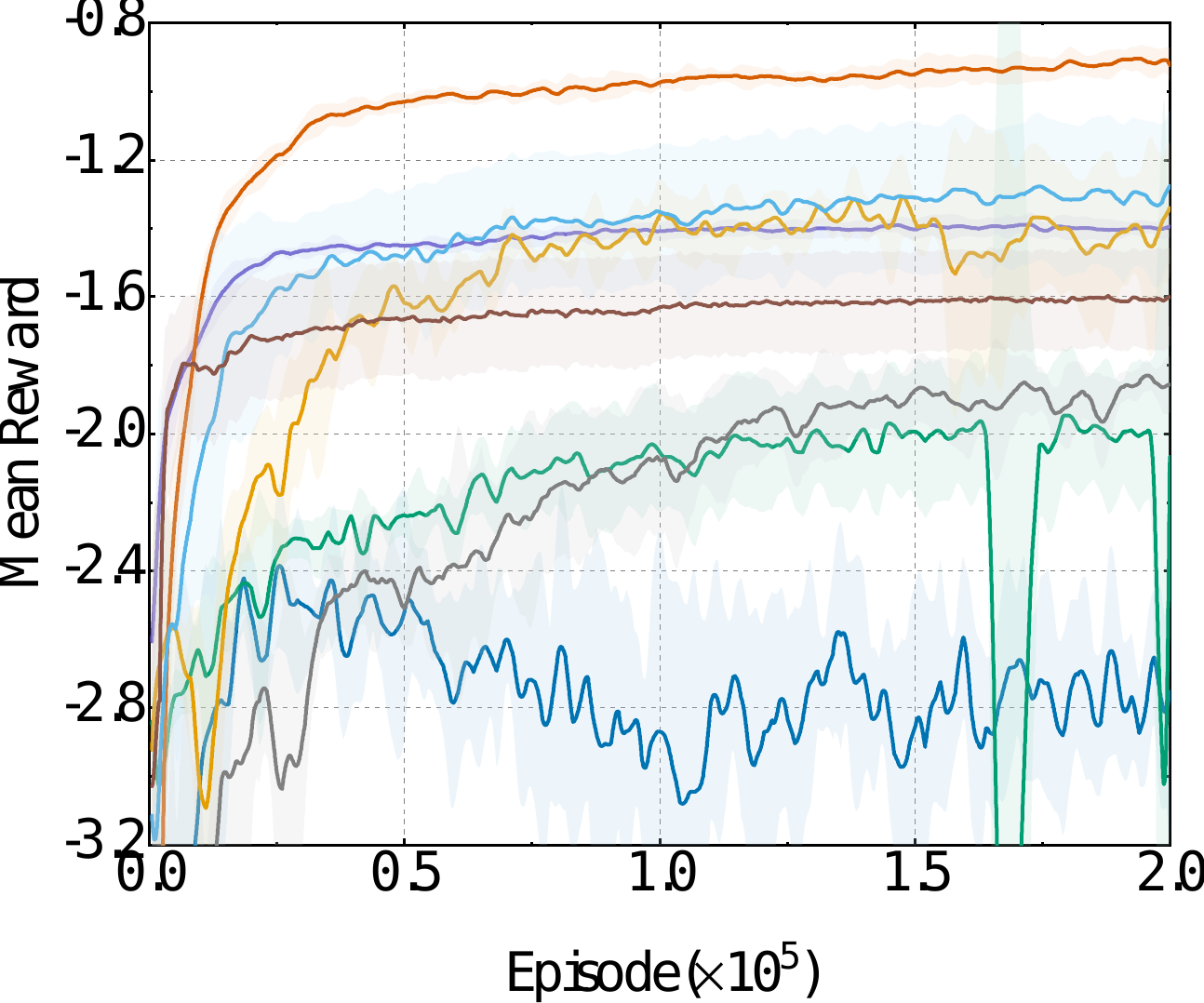}\label{fig:pp_medium}}
	\hfill
	\subfloat[PP \texttt{hard}]{\includegraphics[width=0.24\linewidth]{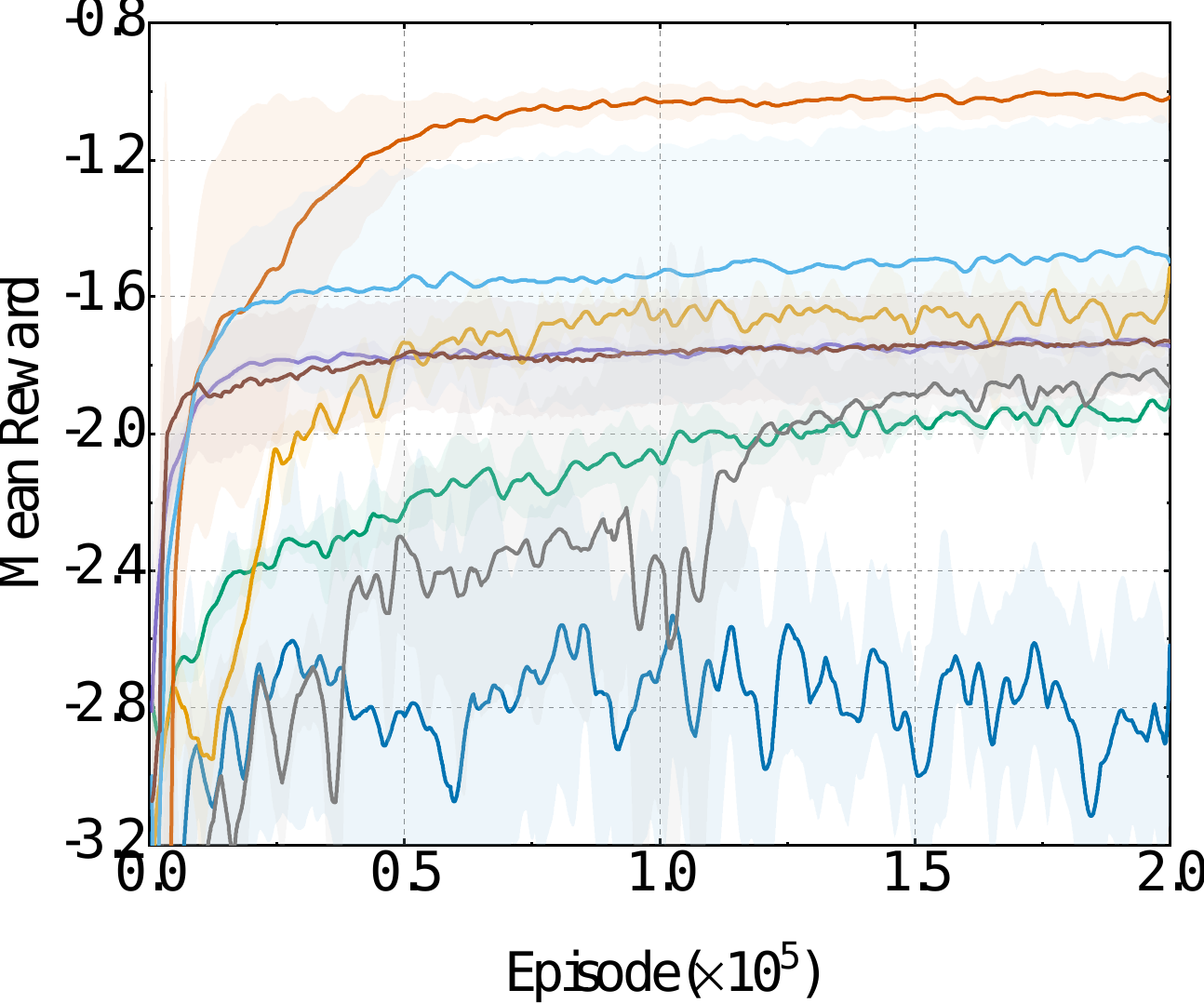}\label{fig:pp_hard}}
	\hfill
	\subfloat[PP \texttt{super\_hard}]{\includegraphics[width=0.24\linewidth]{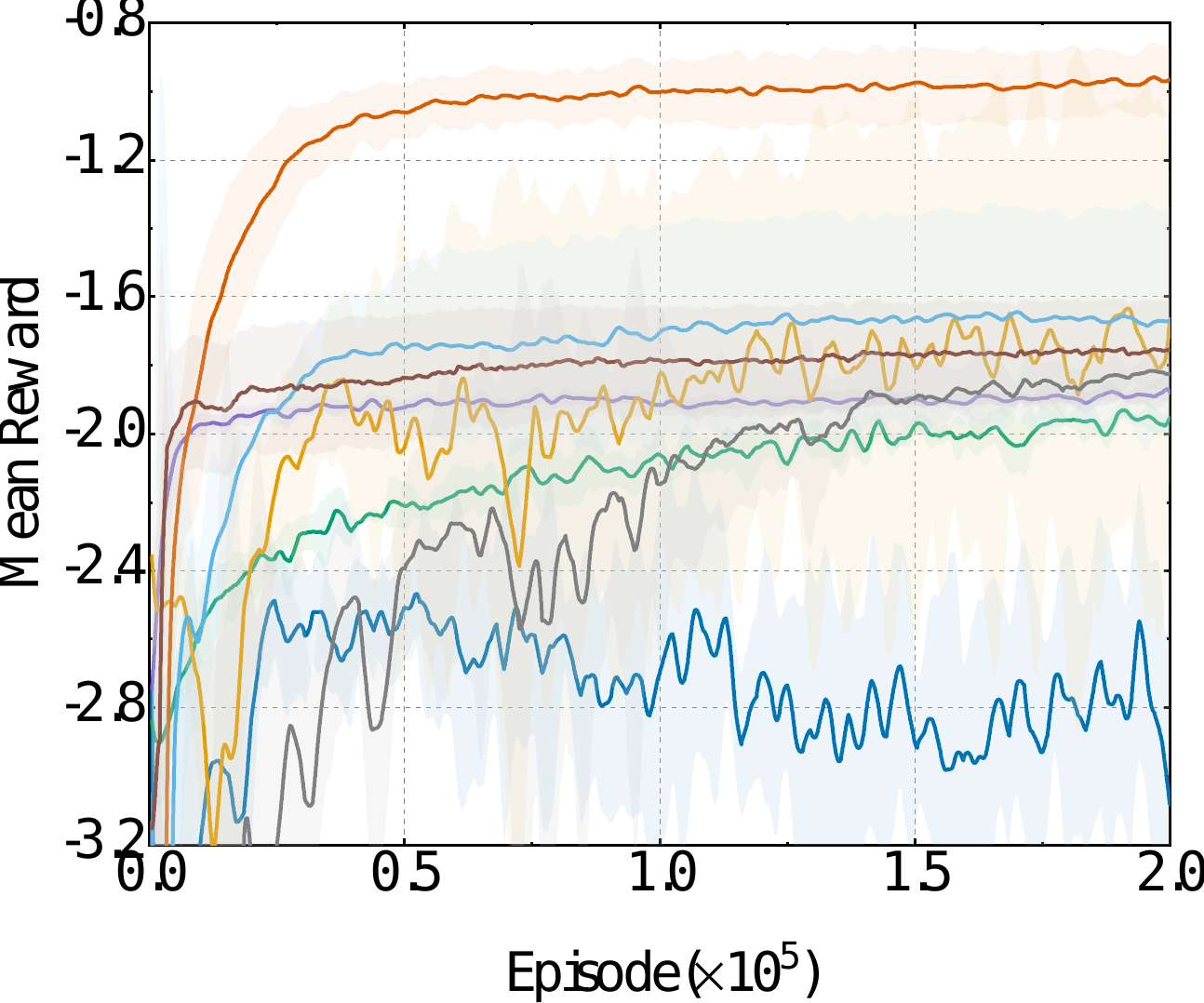}\label{fig:pp_super_hard}}
	\caption{Mean episode reward vs. training episodes on Cooperative Navigation (CN) and Predator Prey (PP) across four delay difficulties.}
	\label{fig:mpe_result_fig}
\end{figure*}

Fig.~\ref{smac_result_fig} and Fig.~\ref{fig:mpe_result_fig} complement Table~\ref{results_table} by showing training curves on SMAC and MPE, respectively. The curves illustrate learning progress across delay settings; smoothing and cross-seed aggregation follow Appendix~\ref{app:exp-details}.

\subsection{Delay-Free Evaluation}
\label{app:delay-free}

Fig.~\ref{fig:1o10b_nodelay} reports the delay-free learning curve on \texttt{1o\_10b\_vs\_1r}.
In the zero-delay limit, CDCMA remains competitive and stable across seeds, indicating that the framework does not rely on delayed settings to remain effective.

\subsection{Additional Validation on MMM2}
\label{app:mmm2}

To broaden the empirical scope beyond the two main SMAC maps reported in the main text, we additionally evaluate CDCMA on MMM2 under the same five settings: \texttt{delay\_free}, \texttt{easy}, \texttt{medium}, \texttt{hard}, and \texttt{super\_hard}. As shown in Table~\ref{tab:mmm2_main}, CDCMA achieves the highest mean win rate across all five settings. This provides an additional evaluation on a third SMAC map.

\begin{figure}[t]
	\begin{minipage}{\linewidth}
		\centering
		\scriptsize
		\captionof{table}{Additional validation on MMM2. Results are mean win rate (std).}
		\label{tab:mmm2_main}
		\resizebox{0.98\linewidth}{!}{%
			\begin{tabular}{lcccccccc}
				\toprule
				Setting & CDCMA & CoDe & DACOM & TGCNet & T2MAC & SMS & G2ANet & ATOC \\
				\midrule
				\texttt{delay\_free}  & \textbf{89.06}(3.13) & 3.13(6.25) & 6.25(4.42) & 84.38(5.71) & 5.47(4.69) & 15.63(14.21) & 29.69(24.54) & 0(0) \\
				\texttt{easy}         & \textbf{89.06}(3.13) & 5.47(6.93) & 6.25(7.22) & 71.88(10.83) & 0.00(0.00) & 7.81(13.62) & 22.66(15.18) & 0(0) \\
				\texttt{medium}       & \textbf{85.16}(4.69) & 4.69(9.38) & 5.47(1.56) & 69.53(13.35) & 0(0) & 4.69(5.41) & 21.88(11.97) & 0(0) \\
				\texttt{hard}         & \textbf{84.38}(4.42) & 10.16(11.80) & 3.91(2.99) & 61.72(18.11) & 0(0) & 1.56(1.80) & 19.53(18.29) & 0(0) \\
				\texttt{super\_hard}  & \textbf{82.03}(5.34) & 14.06(14.09) & 4.69(5.41) & 58.59(9.67) & 0(0) & 0(0) & 0(0) & 0(0) \\
				\bottomrule
			\end{tabular}}
	\end{minipage}

	\vspace{1em}
	\centering
	\includegraphics[width=0.5\linewidth]{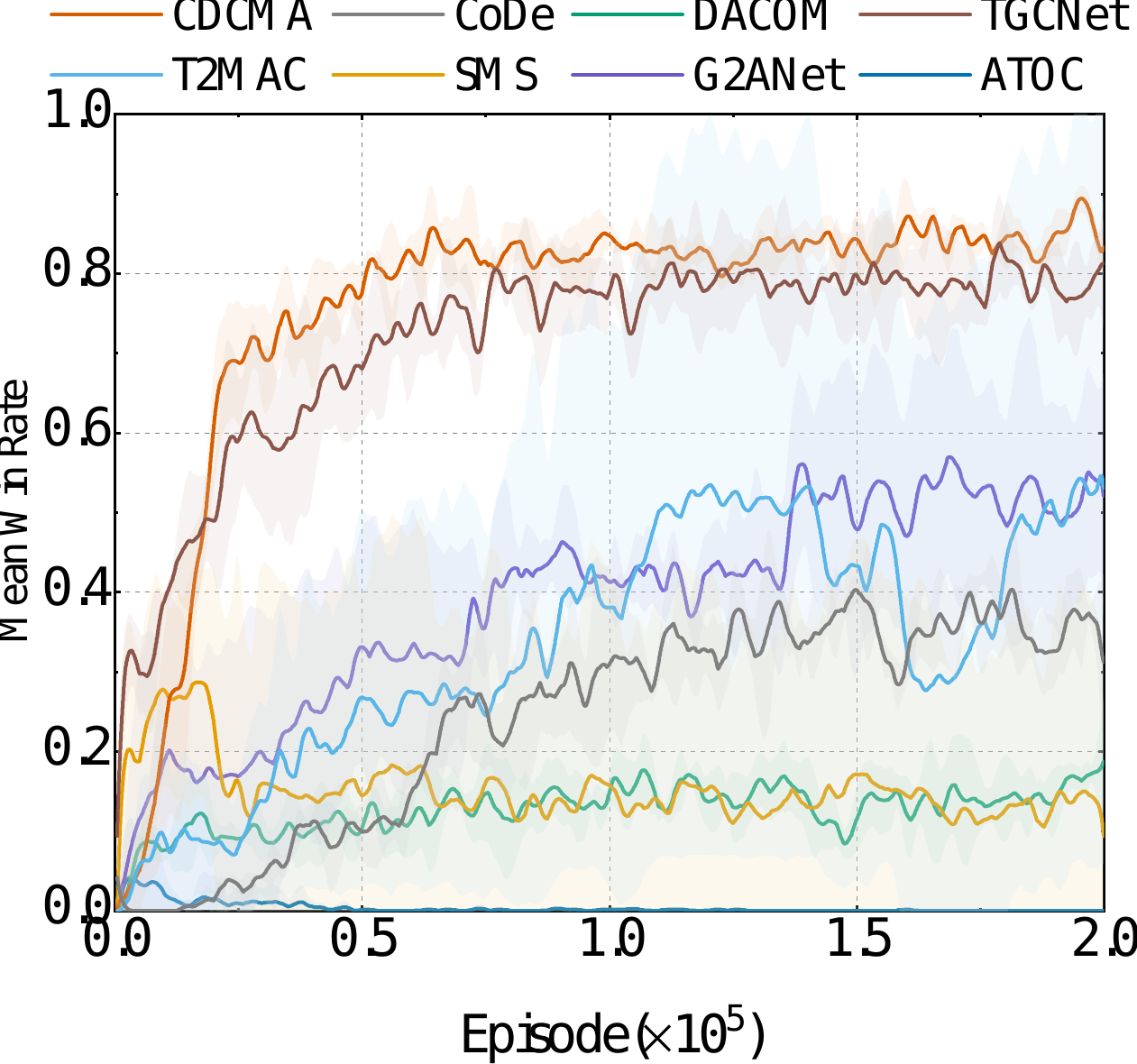}
	\caption{Mean win rate vs. training episodes on \texttt{1o\_10b\_vs\_1r} under the delay-free condition.}
	\label{fig:1o10b_nodelay}
\end{figure}

In the delay-free setting, CDCMA achieves a mean win rate of 89.06\%, compared with 84.38\% for TGCNet, and remains ahead across the four delayed settings.

\subsection{OTG Prediction Diagnostics}
\label{app:exp-otg}

\begin{figure}[t]
	\centering
	\subfloat[True trajectory]{%
		\begin{minipage}{0.30\linewidth}
			\centering
			\includegraphics[width=\linewidth]{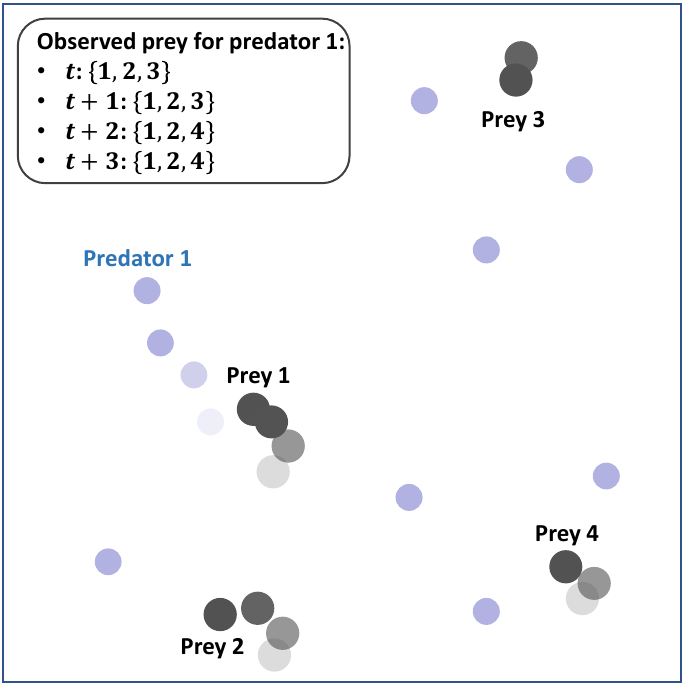}
		\end{minipage}
		\label{fig:obs_true}
	}
	\hfill
	\subfloat[Predicted trajectory]{%
		\begin{minipage}{0.30\linewidth}
			\centering
			\includegraphics[width=\linewidth]{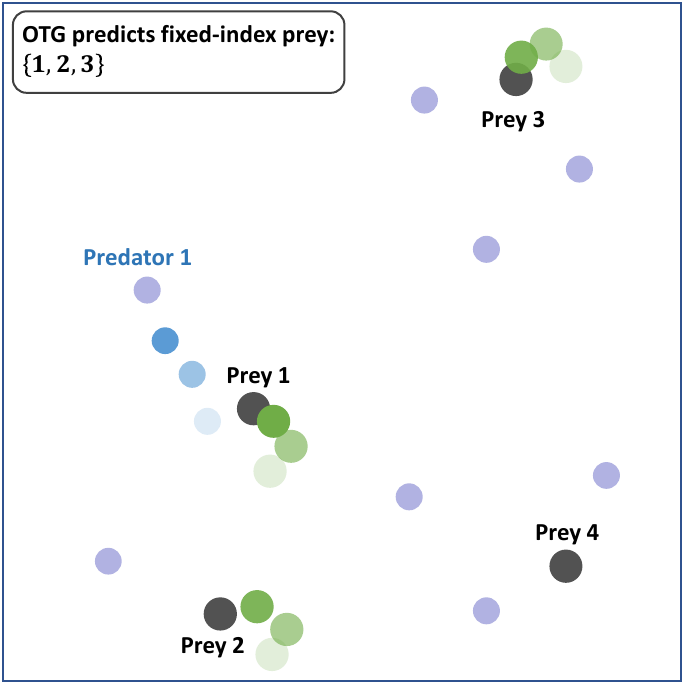}
		\end{minipage}
		\label{fig:obs_pre}
	}
	\hfill
	\subfloat[Prediction loss]{%
		\begin{minipage}{0.33\linewidth}
			\centering
			\includegraphics[width=\linewidth]{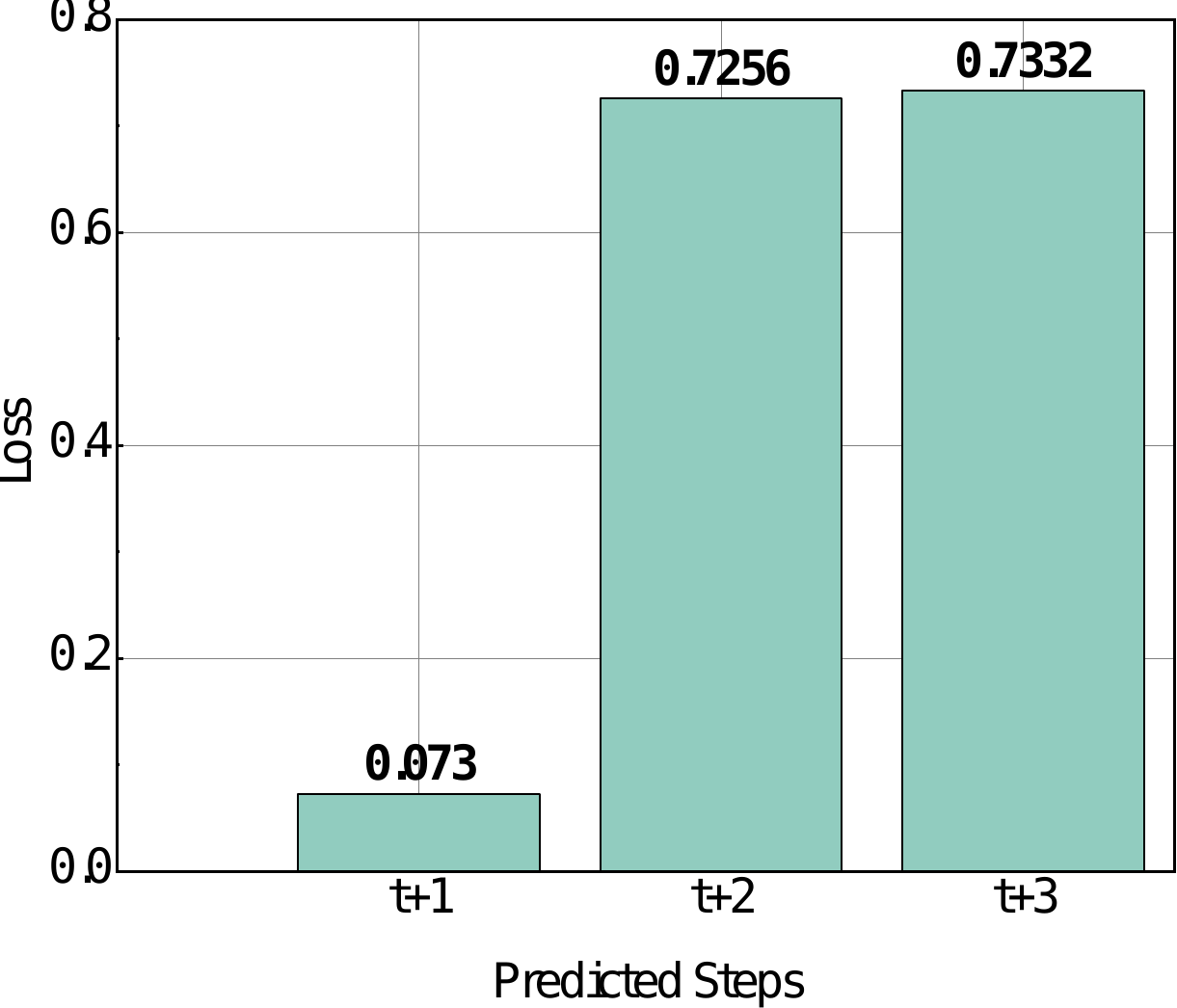}
		\end{minipage}
		\label{fig:obs_loss}
	}
	\caption{Qualitative OTG diagnostics on Predator Prey. Left: ground-truth observation trajectories. Middle: OTG-predicted trajectories. Right: per-timestep prediction loss. Lighter circles indicate farther future steps.}
	\label{fig:otg_diag}
\end{figure}

We provide qualitative diagnostics for OTG on Predator Prey under the \texttt{super\_hard} delay setting. In this task, each predator observes its ego position and velocity together with the relative positions and velocities of its three nearest prey. OTG predicts a future observation trajectory $\hat{\tau}_{i,t}=(\hat{o}_{i,t+1},\dots,\hat{o}_{i,t+H})$, which is encoded into outgoing messages to reduce the information gap at consumption time.

Fig.~\ref{fig:otg_diag} shows that OTG can track short-horizon observation evolution reasonably well, which is consistent with the empirical gains obtained by using a moderate prediction horizon. At the same time, we observe occasional prediction-loss spikes. These spikes are typically associated with nearest-neighbor identity switching in the observation encoding: the prey slots are defined by nearest-neighbor ordering at each timestep, so the slot-to-entity correspondence may change between successive steps. As a result, the supervised target for a fixed coordinate slot may correspond to a different physical prey in the next step, which can produce a large apparent error even when the predicted physical motion is reasonable.

This permutation-induced mismatch becomes more likely as the prediction horizon grows, which helps explain the degradation observed when $H>d_{\max}$ in the ablation results. Overall, the diagnostics support the interpretation that OTG is beneficial for mitigating temporal misalignment under delays, while overly long horizons amplify both model error and supervision mismatch.

\subsection{Additional CGDC Diagnostics}
\label{app:cgdc-diagnostics}

We provide additional diagnostics for CDCMA on PP and \texttt{1o\_10b\_vs\_1r} under the bounded stochastic delay settings used in the main experiments.
We first specify the decision-distribution metric reported in Sec.~\ref{sec:exp-robust}, then examine gain and score consistency.

\paragraph{Decision-distribution information gap.}
\label{app:decision-gap}
At receiver timestep $t$, define the delayed and timely message collections over the same available-sender set as
\[
\mathbf m^r_{i,t}
=
(m^r_{ij,t})_{j\in\mathcal Z_{i,t}},
\qquad
\mathbf m^{\mathrm{tar}}_{i,t}
=
(e_{j,t})_{j\in\mathcal Z_{i,t}}.
\]
We evaluate the same decision network under both collections, holding the receiver's local inputs, history state, and all other inputs fixed.
Only the message contents change; aggregation is recomputed using the method's existing aggregation module.
For CDCMA, the stored request-time CGDC scores remain fixed.

Both message conditions use the same current available-action set $\mathcal A_{i,t}^{\mathrm{avail}}$.
For policy-based methods, $p_i(\cdot\mid x_{i,t},\mathbf m)$ is the actor's output distribution over this set. For the Q-based T2MAC, let $Q_i^{\mathrm{act}}(x_{i,t},\mathbf m,a)$ denote the learned action values used for decision making, and define the diagnostic distribution as
\[
p_i(a\mid x_{i,t},\mathbf m)
=\frac{\exp Q_i^{\mathrm{act}}(x_{i,t},\mathbf m,a)}
{\sum_{a'\in\mathcal A_{i,t}^{\mathrm{avail}}}\exp Q_i^{\mathrm{act}}(x_{i,t},\mathbf m,a')},
\qquad a\in\mathcal A_{i,t}^{\mathrm{avail}}.
\]
The KL divergence is computed over this same available-action set.
T2MAC's diagnostic therefore depends on its native action-value scale; CDCMA's diagnostic uses its actor rather than either centralized critic. We define
\begin{equation}
\label{eq:decision-collection-gap}
\Delta_{\mathrm{dec}}(i,t)
=
D_{\mathrm{KL}}\!\left(
p_i(\cdot\mid x_{i,t},\mathbf m^{\mathrm{tar}}_{i,t})
\,\middle\|\,
p_i(\cdot\mid x_{i,t},\mathbf m^r_{i,t})
\right).
\end{equation}
For $N$ agents and $T$ evaluated timesteps, we report
\begin{equation}
\label{eq:decision-gap-mean}
\overline{\Delta}_{\mathrm{dec}}
=
\frac{1}{NT}
\sum_{t=0}^{T-1}\sum_{i=1}^{N}
\Delta_{\mathrm{dec}}(i,t).
\end{equation}
When $\mathcal Z_{i,t}=\varnothing$, both message collections are empty, yielding zero KL that remains included in the mean.
The overall mean reflects both message availability and the sensitivity of the decision distribution to message replacement.
The diagnostic leaves each method's action-selection rule unchanged and uses delayed inputs for environment interaction.

Equation~\eqref{eq:info-gap} isolates a single sender's effect on the actor policy, whereas Eq.~\eqref{eq:decision-collection-gap} replaces the complete available-message collection and uses the method-specific distribution defined above.
CGDC training uses the critic-tempered cost $V^c_{ij,t}$ in Eq.~\eqref{eq:vc}; Fig.~\ref{information_fig} reports the decision-distribution metric defined above.

As shown in Fig.~\ref{information_fig}, CDCMA has smaller mean decision-distribution gaps than the compared methods across all evaluated delay regimes.
From \texttt{easy} to \texttt{super\_hard}, its gap increases from $0.02$ to $0.09$ on PP and from $0.03$ to $0.08$ on \texttt{1o\_10b\_vs\_1r}.
Removing OTG increases the gap on both tasks.
Together with the module ablations in Fig.~\ref{fig:abla_modules}, these results are consistent with OTG's role in reducing timely-vs-delayed changes in the receiver's action distribution.

\paragraph{Gain and score consistency.}
We examine whether the gain surrogate and current CGDC scores agree with the full critic's assessment of received messages.
For each evaluated transition, we consider only receiver--sender pairs $(i,j,t)$ with $j\in\mathcal Z_{i,t}$ and compare the original aggregation with a message-masked counterfactual.
The message-present condition keeps all received messages, including that from sender $j$, and yields $\tilde m_i^{\mathrm{all}}$, while the message-masked condition replaces $m^r_{ij,t}$ with $m^\emptyset$ and recomputes aggregation, yielding $\tilde m_i^{(-j)}$.
With the local observation, joint action, and other received messages fixed, we define
\begin{equation}
	\Delta Q_{ij,t}
	=
	Q_i^{\mathrm{full}}(\boldsymbol{o}_t,\boldsymbol{a}_t,\tilde m_i^{\mathrm{all}})
	-
	Q_i^{\mathrm{full}}(\boldsymbol{o}_t,\boldsymbol{a}_t,\tilde m_i^{(-j)}).
	\label{eq:cgdc-utility-deltaq}
\end{equation}
We report raw $\Delta Q$ without normalization; positive values indicate a higher full-critic prediction when the message is retained.
This diagnostic uses a leave-one-out comparison with other messages retained, while Eq.~\eqref{eq:q-gain} compares sender-only access with no incoming messages.

For each seed and delay setting, we compare mean $\Delta Q$ between the bottom and top 30\% of pairs ranked by $\widehat V^g_{ij,t}$, and between pairs with current DCOS scores $\hat c_{ij,t}\le0$ and $\hat c_{ij,t}>0$.
Sign agreement is the fraction of pairs for which $\hat c_{ij,t}$ and raw $\Delta Q_{ij,t}$ have matching signs.
We average these statistics equally across the four seeds and four delay settings.

\begin{table}[t]
	\centering
	\scriptsize
	\caption{CGDC consistency diagnostics. $\Delta Q$ denotes the raw full-critic value difference between retaining and masking a received message.}
	\label{tab:cgdc-utility-diagnostics}
	\resizebox{0.98\linewidth}{!}{
		\begin{tabular}{lccccc}
			\toprule
			Task 
			& Low-gain $\Delta Q$ 
			& High-gain $\Delta Q$ 
			& $\Delta Q$ ($\hat c\le0$)
			& $\Delta Q$ ($\hat c>0$)
			& Sign agreement (\%) \\
			\midrule
			PP & -0.019 & 0.088 & -0.013 & 0.076 & 79.1 \\
			\texttt{1o\_10b\_vs\_1r} & -0.012 & 0.121 & -0.018 & 0.109 & 82.1 \\
			\bottomrule
	\end{tabular}}
\end{table}

Table~\ref{tab:cgdc-utility-diagnostics} shows higher raw $\Delta Q$ for high-gain than low-gain pairs and for positive-score than nonpositive-score pairs on both tasks.
Sign agreement reaches 79.1\% on PP and 82.1\% on \texttt{1o\_10b\_vs\_1r}.
These results indicate consistency between the gain surrogate, current CGDC scores, and the full critic's message assessments.

\subsection{Replay Visualizations for Qualitative Analysis}
\label{app:replay}

We provide replay snapshots in Fig.~\ref{smac_policy_fig} to qualitatively examine how CDCMA facilitates coordination under cross-timestep delays on two SMAC scenarios. The top row corresponds to \texttt{1o\_2r\_vs\_4r}, whereas the bottom row corresponds to \texttt{1o\_10b\_vs\_1r}. For each scenario, we juxtapose representative successful and failed rollouts.

\paragraph{\texttt{1o\_2r\_vs\_4r}.}
Roaches depend on the Overseer for target-related information and use it to synchronize navigation and engagement. In the successful rollout, the two Roaches exhibit consistent movement toward the target region early in the episode, thereby avoiding route divergence around the corner, and later enter effective firing range in a temporally aligned manner. In the failure rollout, an early deviation from the intended route induces split navigation and delays mutual convergence, so the team remains spatially dispersed and fails to complete effective engagement within the episode horizon.

\paragraph{\texttt{1o\_10b\_vs\_1r}.}
Banelings communicate with the Overseer to acquire target-location cues and to coordinate regrouping under terrain constraints. In the successful rollout, agents first separate into external and internal groups because of the cliff structure, then the internal group exits the cliffs and rejoins the main movement toward the target region, restoring team cohesion before engagement. In the failure case, the internal group remains trapped or inactive inside the cliffs while the target stays separated, so the remaining agents fail to re-establish coordinated movement before termination.

Overall, these replay snapshots suggest that successful outcomes are associated with earlier target-aligned movement and more timely regrouping, whereas failures are characterized by persistent route divergence or delayed coordination under stale communication.

\begin{figure*}[tb]
	\centering
	\subfloat[1o2r (success), $t{=}2$]{\includegraphics[width=0.24\linewidth]{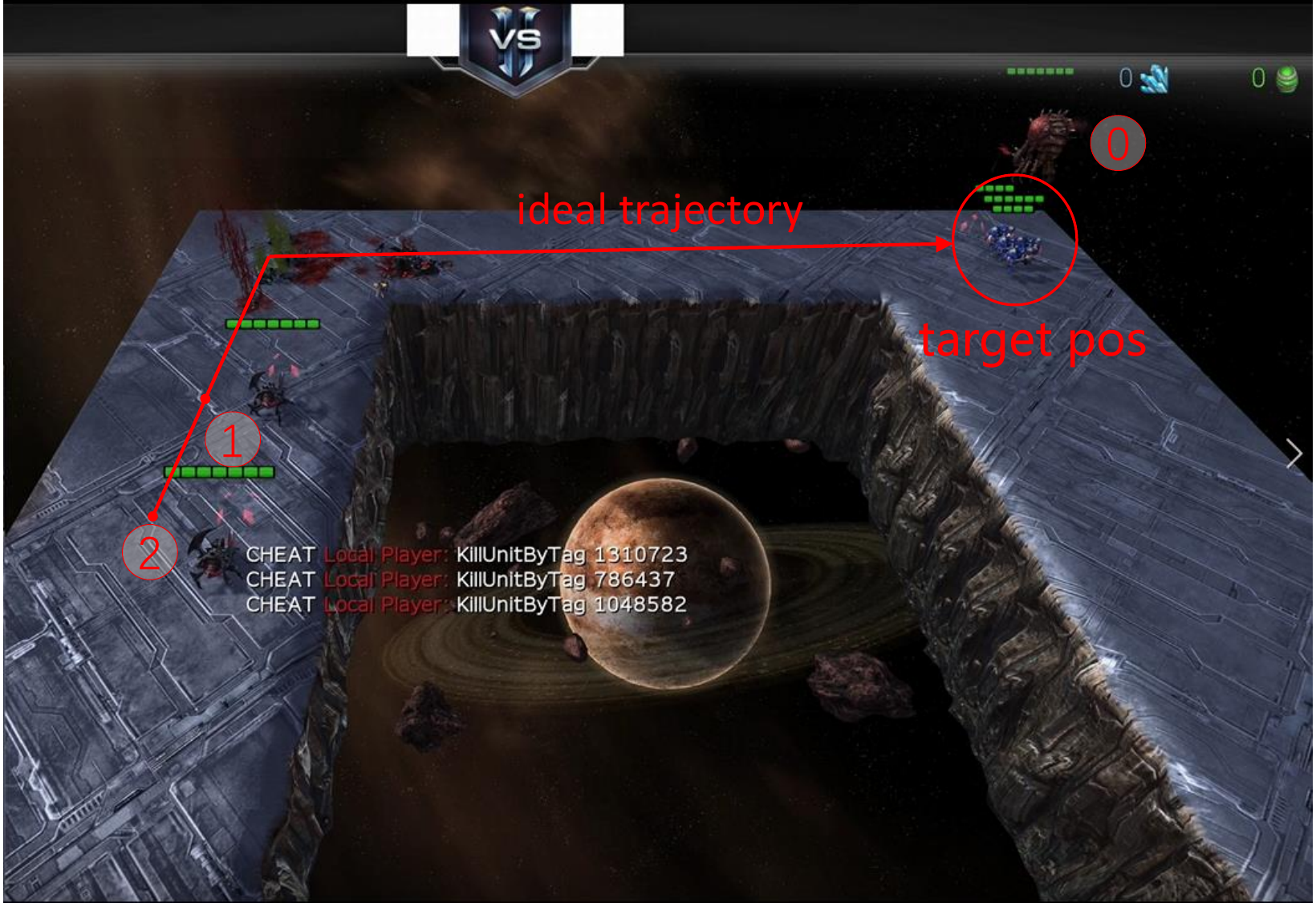}\label{fig:1o2r_success_t2}}
	\hfill
	\subfloat[1o2r (success), $t{=}26$]{\includegraphics[width=0.24\linewidth]{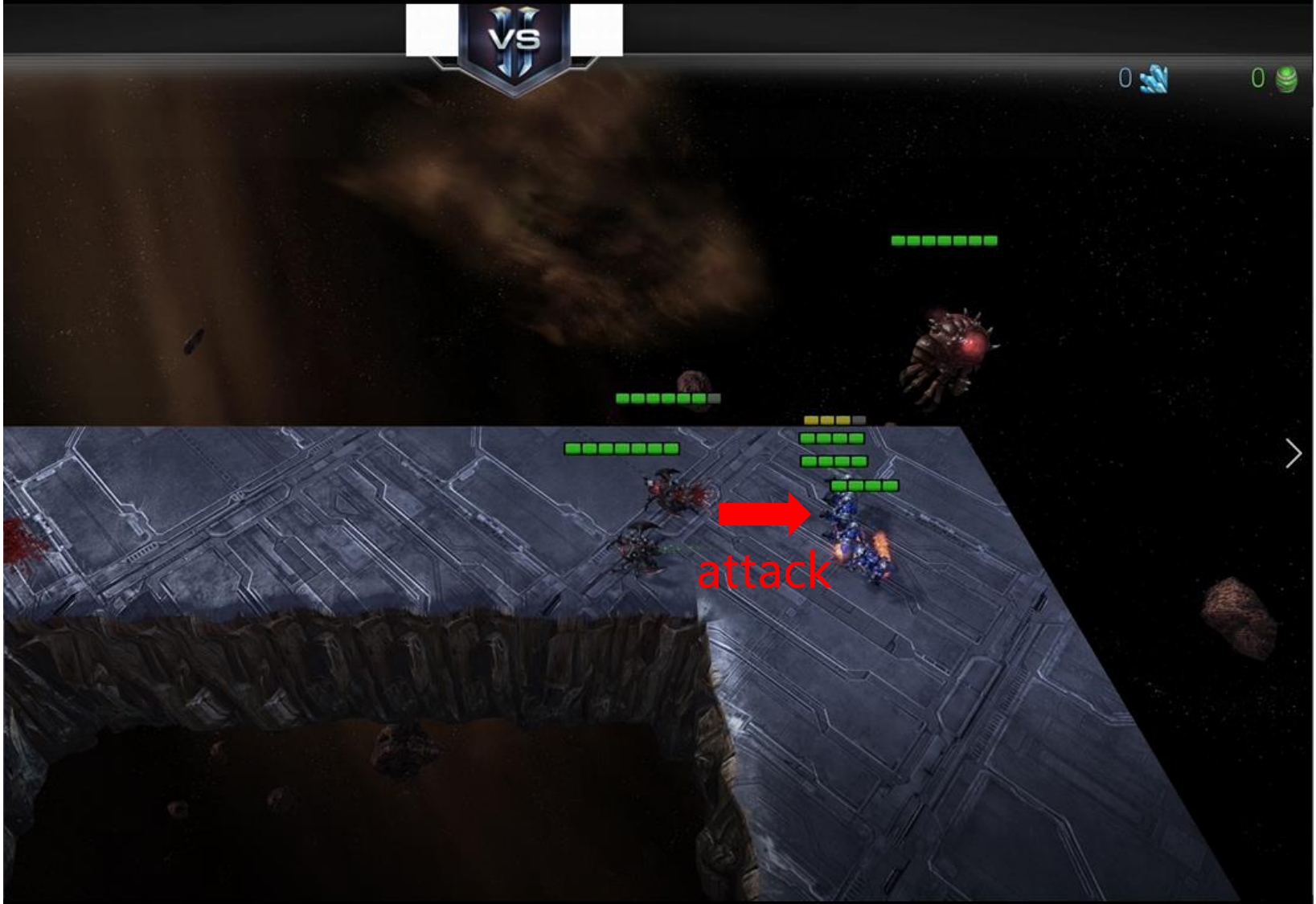}\label{fig:1o2r_success_t26}}
	\hfill
	\subfloat[1o2r (failure), $t{=}10$]{\includegraphics[width=0.24\linewidth]{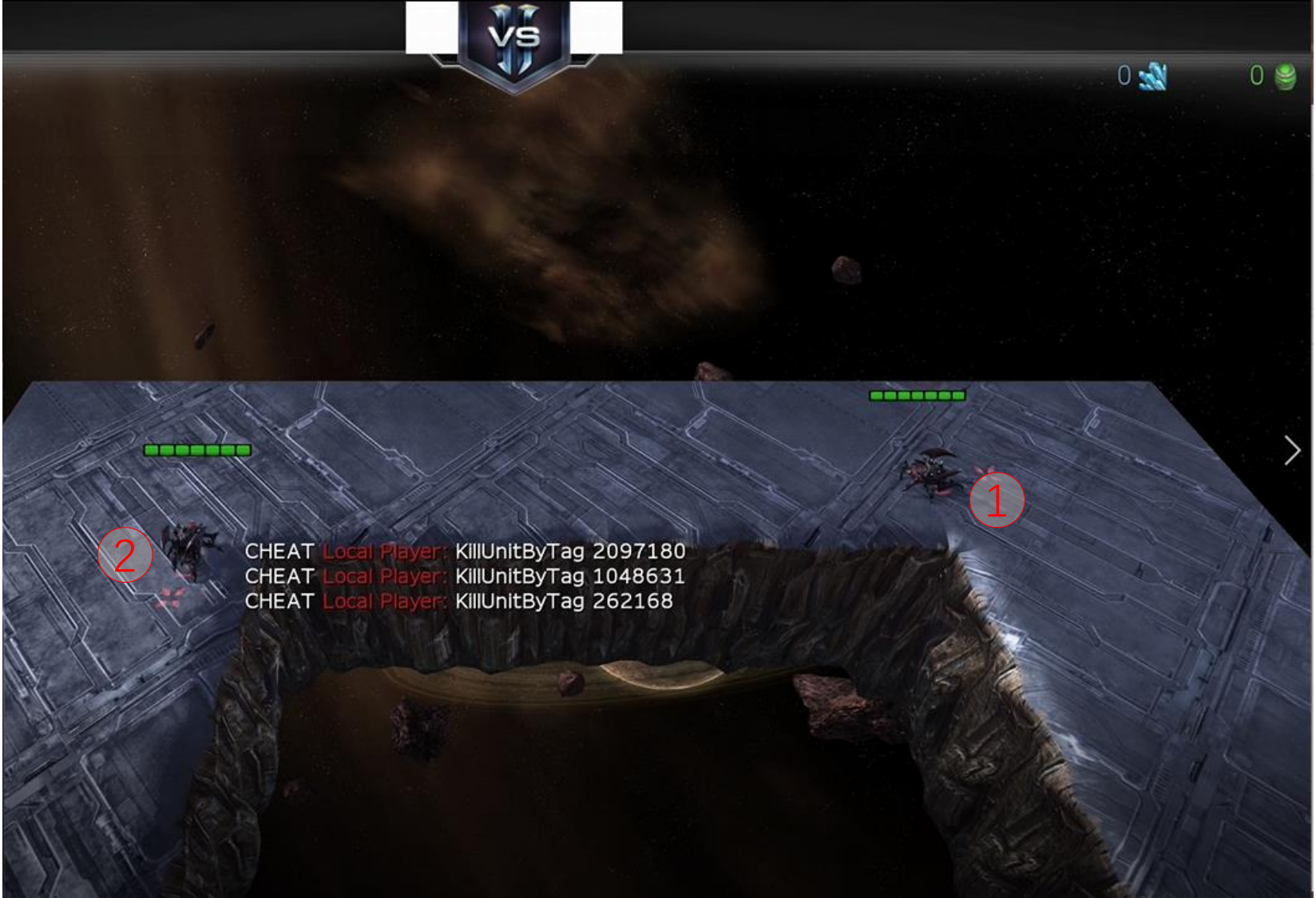}\label{fig:1o2r_failure_t10}}
	\hfill
	\subfloat[1o2r (failure), $t{=}40$]{\includegraphics[width=0.24\linewidth]{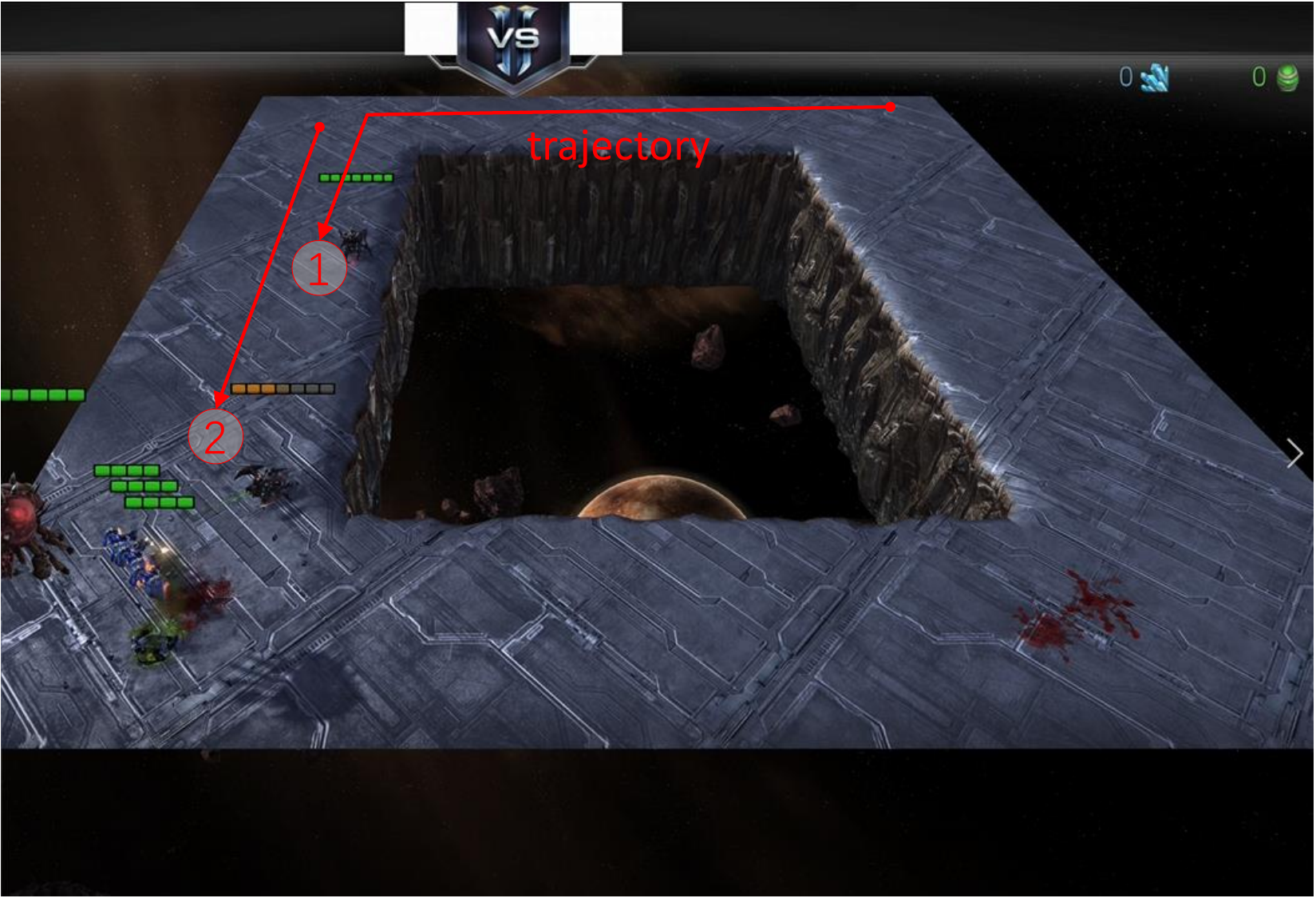}\label{fig:1o2r_failure_t40}}
	
	\subfloat[1o10b (success), $t{=}8$]{\includegraphics[width=0.24\linewidth]{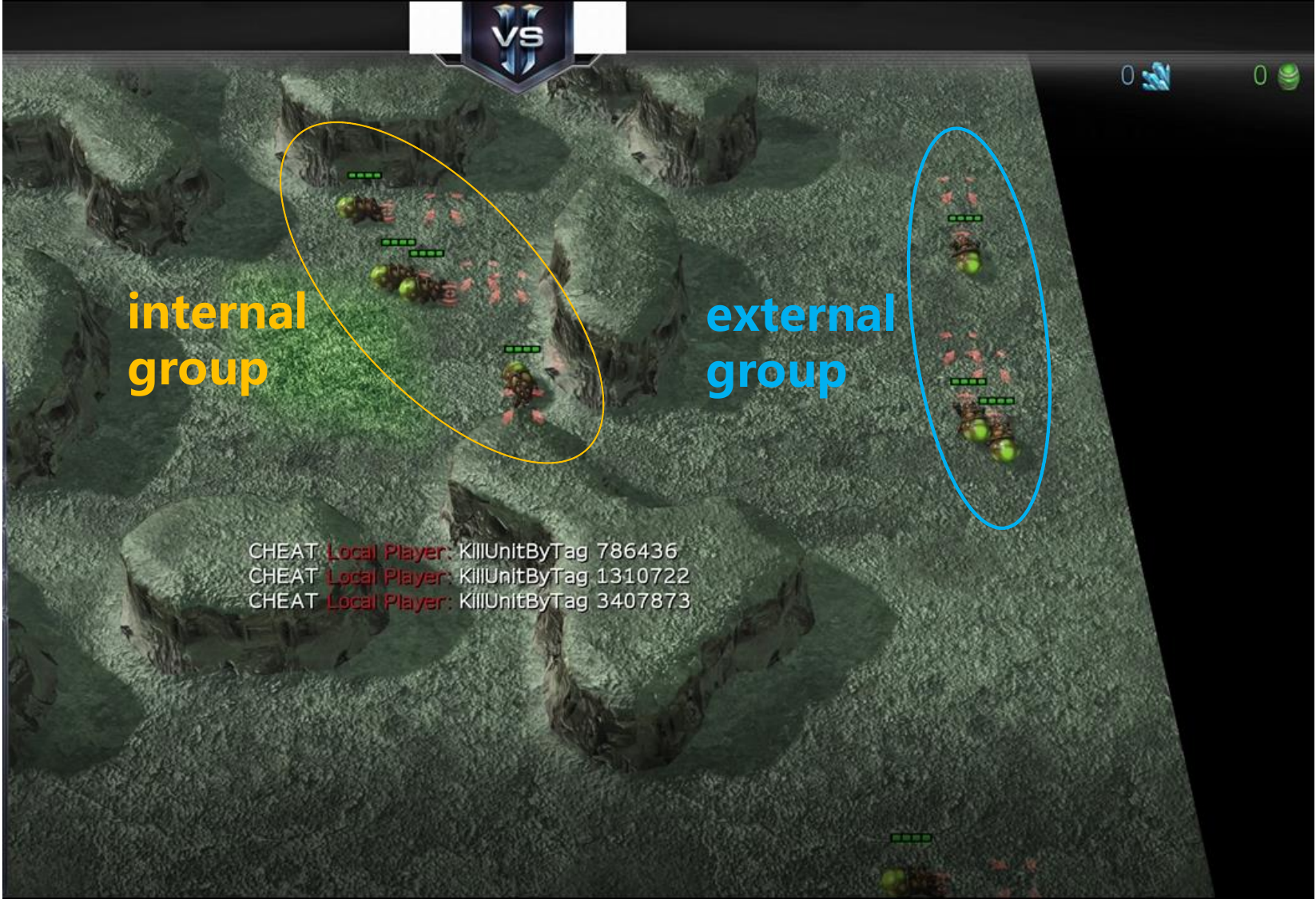}\label{fig:1o10b_success_t8}}
	\hfill
	\subfloat[1o10b (success), $t{=}18$]{\includegraphics[width=0.24\linewidth]{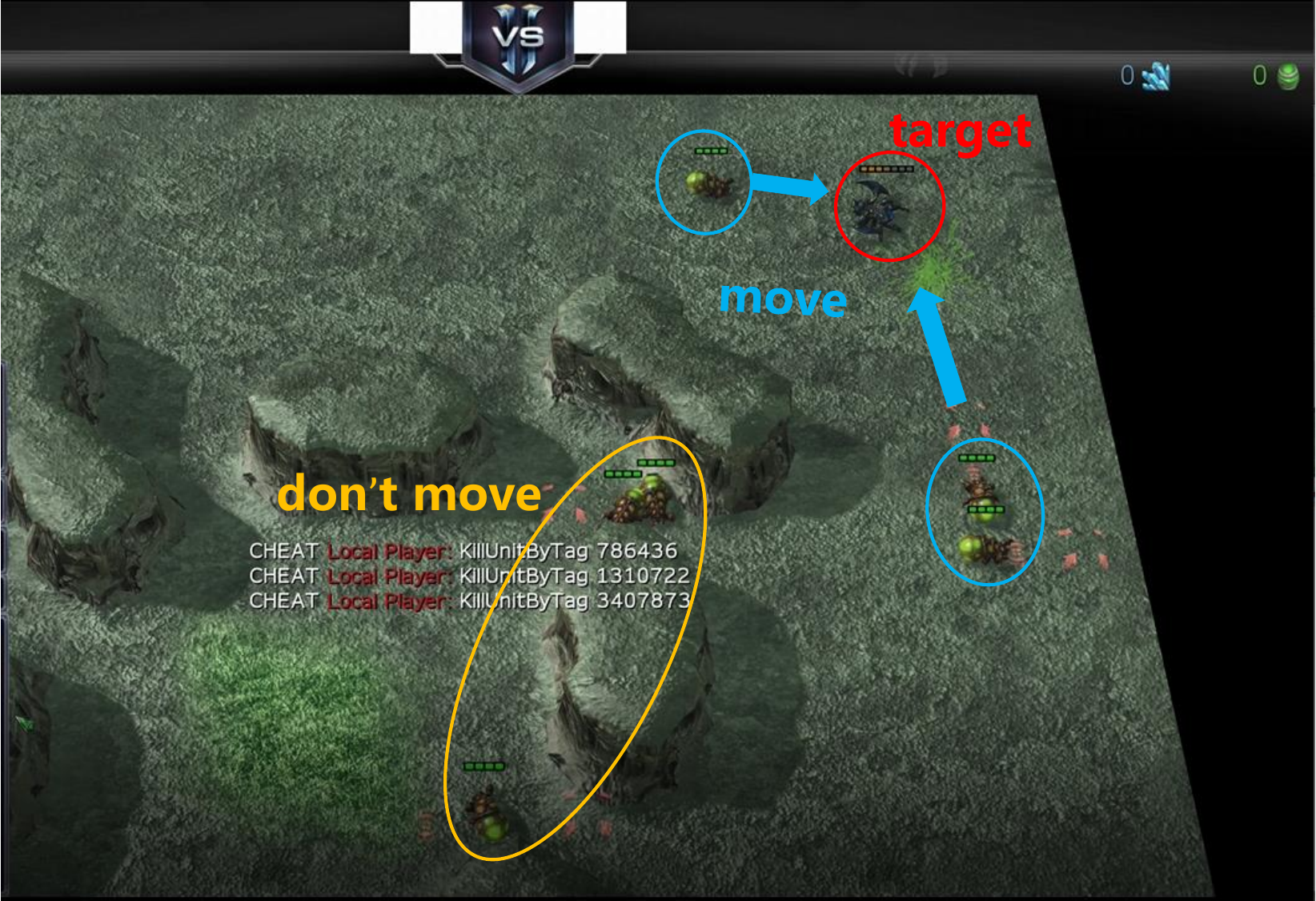}\label{fig:1o10b_success_t18}}
	\hfill
	\subfloat[1o10b (success), $t{=}32$]{\includegraphics[width=0.24\linewidth]{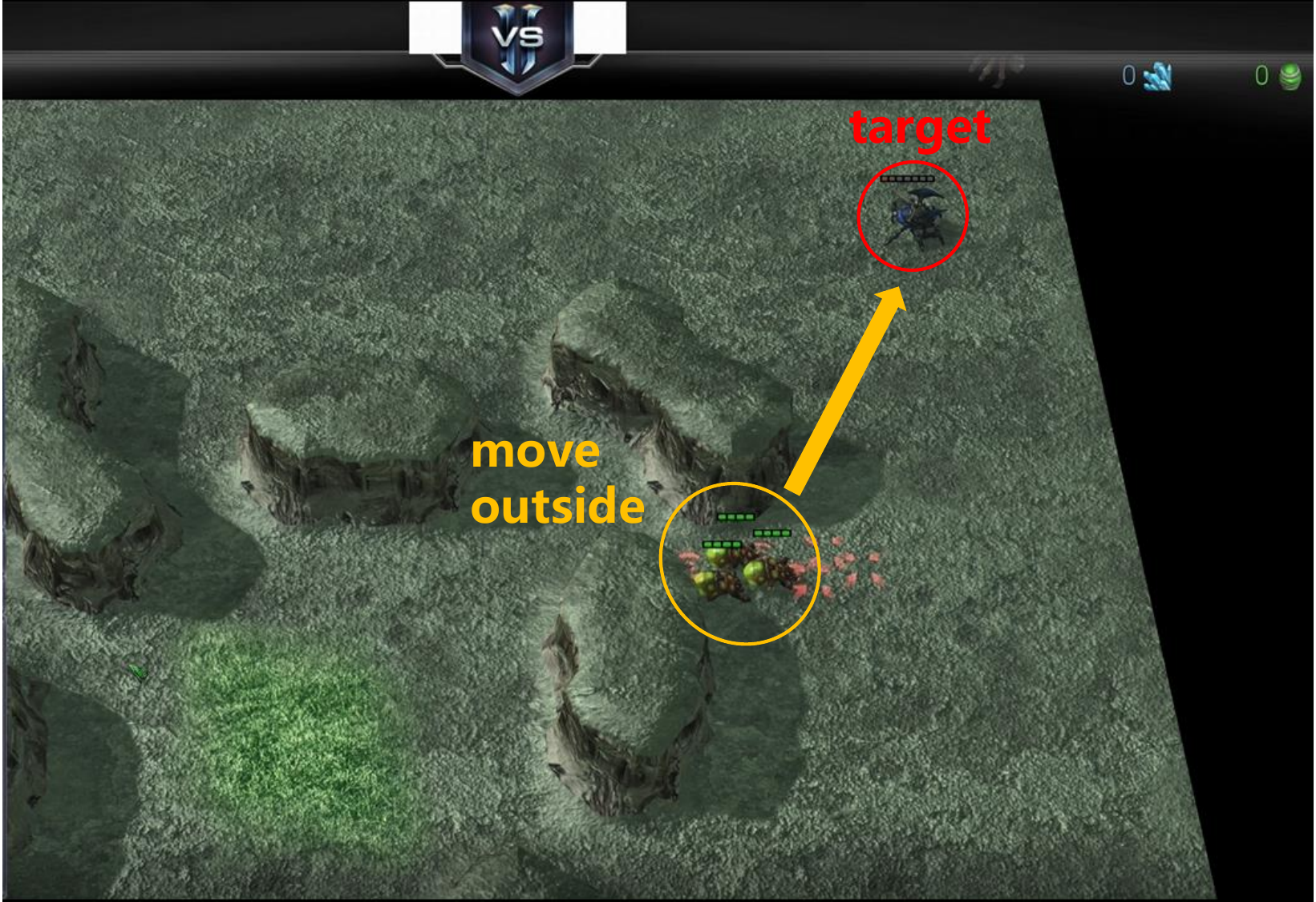}\label{fig:1o10b_success_t32}}
	\hfill
	\subfloat[1o10b (failure)]{\includegraphics[width=0.24\linewidth]{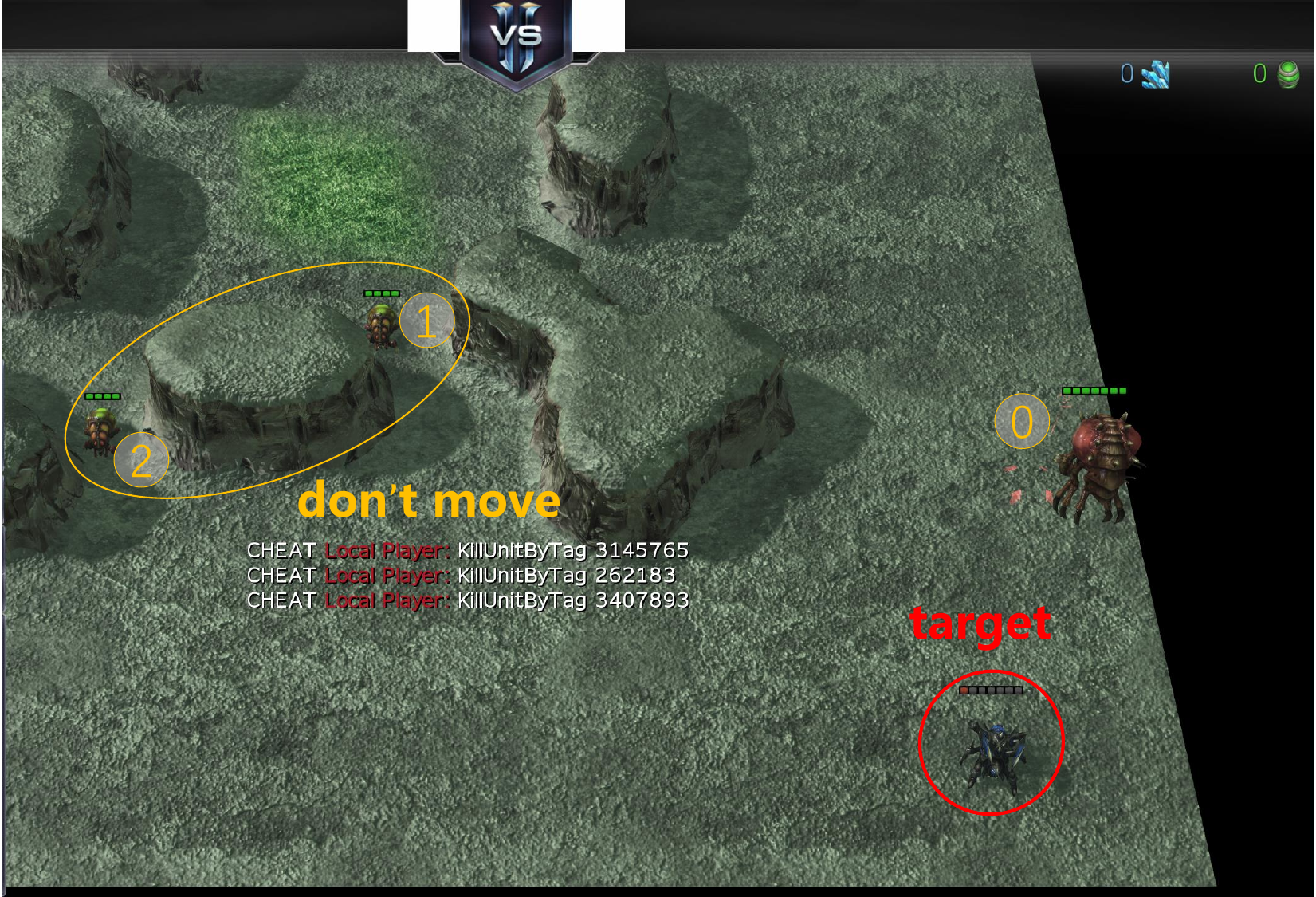}\label{fig:1o10b_failure}}
	\caption{Replay snapshots of CDCMA policies on \texttt{1o\_2r\_vs\_4r} and \texttt{1o\_10b\_vs\_1r} under cross-timestep delays.}
	\label{smac_policy_fig}
\end{figure*}

\clearpage
\section{CDCMA Algorithm}
\label{app:algo-sec}

\begin{algorithm}[H]
\small
\caption{CDCMA}
	\label{app:algo}
	\begin{algorithmic}[1]
		\STATE Initialize parameter-shared networks: observation encoder $e$, actor $\pi$, full critic $Q^{\mathrm{full}}$, gain critic $Q^{\mathrm{gain}}$, DCOS predictor $\phi$, OTG predictors $(f^{a}, f^{o})$, and CAMA projections $(W^{Q}, W^{K}, W^{V})$.
		\STATE Initialize target networks, replay buffer $\mathcal{R}$, and per-pair delay queues $\{\mathcal{B}_{ji}\}_{j\neq i}$.
		
		\FOR{episode $=1,\dots,M$}
			\STATE Reset environment and queues; initialize an episode record and observe $\{o_{i,0}\}_{i=1}^N$.
		\FOR{$t=0,\dots,T-1$}
		\STATE The environment samples exogenous delays $\{d_{ji,t}\}_{j\neq i}$, unobserved by agents.
		
		\STATE \hspace{-0.6em}\textit{// Deliver arrivals at timestep $t$}
		\FOR{agent $i=1,\dots,N$}
		\STATE Deliver all entries whose arrival timestep equals $t$ from $\{\mathcal{B}_{ji}\}_{j\neq i}$.
		\STATE For each sender $j\neq i$, set $(m^{r}_{ij,t},\bar c_{ij,t})$ to the delivered entry with the largest send timestep if any; otherwise set $m^{r}_{ij,t}\leftarrow m^{\emptyset}$.
		\STATE Let $\mathcal{Z}_{i,t}\leftarrow \{\,j\neq i \mid m^{r}_{ij,t}\neq m^{\emptyset}\,\}$.
		\ENDFOR
		
		\STATE \hspace{-0.6em}\textit{// Partner selection}
		\FOR{agent $i=1,\dots,N$}
		\STATE Observe $o_{i,t}$ and received messages $\{m^{r}_{ij,t}\}_{j\in \mathcal{Z}_{i,t}}$; compute $e_{i,t}\leftarrow e_i(o_{i,t})$.
		\STATE Compute $\{\hat c_{ij,t}\}_{j\neq i}\leftarrow \phi_i\!\big(o_{i,t},\{m^{r}_{ij,t}\}_{j\in \mathcal{Z}_{i,t}};\theta^{\phi_i}\big)$ and set $\mathcal{G}_{i,t}\leftarrow \{\,j\neq i \mid \hat c_{ij,t}>0\,\}$.
		\ENDFOR
		
		\STATE \hspace{-0.6em}\textit{// Construct sender-level messages}
		\STATE Let $\mathcal{S}_t\leftarrow\bigcup_{i=1}^{N}\mathcal{G}_{i,t}$ be the requested-sender set.
		\FOR{sender $j\in\mathcal{S}_t$}
		\STATE Roll out $\hat{\tau}_{j,t}$ with $(f^a,f^o)$ using pre-send inputs ($H=0$ when delay-free).
		\STATE Encode $(\tau_{j,t},\hat{\tau}_{j,t})$ into one outgoing message $m^s_{j,t}$.
		\ENDFOR
		
		\STATE \hspace{-0.6em}\textit{// Deliver or enqueue on requested links}
		\FOR{receiver $i=1,\dots,N$}
		\FOR{each sender $j\in \mathcal{G}_{i,t}$}
		\STATE If $d_{ji,t}=0$, the channel replaces $(m^r_{ij,t},\bar c_{ij,t})$ with $(m^s_{j,t},\hat c_{ij,t})$ and adds $j$ to $\mathcal Z_{i,t}$; otherwise enqueue for $t+d_{ji,t}$.
		\ENDFOR
		\ENDFOR
		
		\STATE \hspace{-0.6em}\textit{// Aggregate received messages and act}
		\FOR{agent $i=1,\dots,N$}
		\STATE If $\mathcal{Z}_{i,t}=\varnothing$, set $\widetilde m_{i,t}\leftarrow\boldsymbol{0}$; otherwise aggregate $\{m^{r}_{ij,t}\}_{j\in \mathcal{Z}_{i,t}}$ with CGDC-modulated attention.
		\STATE Select action $a_{i,t}\sim \pi_i(\cdot \mid e_{i,t},\widetilde m_{i,t})$.
		\ENDFOR
		
		\STATE Execute joint action $\boldsymbol{a}_t$; each agent observes $o_{i,t+1}$ and receives $r_{i,t}$.
		
		\STATE \hspace{-0.6em}\textit{// Store transition and actual-message references}
		\STATE For each pair $(i,j)$ with $j\neq i$, set $m^{\mathrm{tar}}_{ji,t}\leftarrow e_{j,t}$ if $j\in\mathcal{Z}_{i,t}$, and $m^{\mathrm{tar}}_{ji,t}\leftarrow m^\emptyset$ otherwise.
			\STATE Append the transition to the episode record, including observations, actions, rewards, received messages, stored request-time scores, and timely references.
			\ENDFOR
			\STATE Store the completed episode in $\mathcal{R}$.

			\STATE \hspace{-0.6em}\textit{// One training update after the episode}
			\STATE Sample a 32-episode minibatch from $\mathcal{R}$, using replacement if $|\mathcal{R}|<32$, and evaluate losses on valid transitions.
		\STATE For every pair $j\neq i$, construct a delayed-message reception condition as in Appendix~\ref{app:pairwise-targets}, without altering the recorded communication inputs.
		\STATE Evaluate $\widehat V^g_{ij,t}$ and $V^c_{ij,t}$ on this same condition using Eqs.~\eqref{eq:vg-critic} and~\eqref{eq:vc}, with timely reference $e_{j,t}$.
		\STATE Compute $\lambda_0$ by Eq.~\eqref{eq:stable-cgdc-weight} on the minibatch; set $\lambda\leftarrow\kappa\lambda_0$ (default $\kappa=1$).
		\STATE For each pair $(i,j)$ with $j\neq i$, compute the CGDC target $c_{ij,t}\leftarrow \widehat V^g_{ij,t}-\lambda V^c_{ij,t}$.
		\STATE Update the full critic using Eq.~\eqref{qloss}, the gain critic using Eq.~\eqref{qloss-gain}, the actor and encoder using Eqs.~\eqref{ploss} and~\eqref{oeloss}, CAMA jointly with the actor, OTG using Eqs.~\eqref{otgloss} and~\eqref{otgloss2}, and DCOS using Eq.~\eqref{osloss}.
		\STATE If 200 gradient updates have elapsed since the last target update, update the target networks.
		\ENDFOR
	\end{algorithmic}
\end{algorithm}


\end{document}